\documentclass{article}
\usepackage{iclr2027_conference,times}

\usepackage{graphicx} %
\usepackage{amsmath}
\usepackage{amssymb}
\usepackage{mathtools}
\usepackage{amsthm}
\usepackage{enumitem}
\usepackage{thm-restate}
\usepackage{longtable}
\usepackage{wrapfig}
\usepackage{xcolor}
\usepackage{float}
\usepackage{subcaption}
\usepackage{algorithm}
\usepackage{algpseudocode}
\usepackage{tabularx}
\usepackage{booktabs}
\newtheorem{theorem}{Theorem}[section]
\newtheorem{lemma}{Lemma}[section]
\newtheorem{corollary}{Corollary}[section]
\newtheorem{proposition}{Proposition}[section]
\newtheorem{assumption}{Assumption}[section]
\newtheorem{remark}{Remark}[section]

\newtheorem{definition}{Definition}[section]

\newcommand{\argmin}{\operatorname{argmin}}

\newcommand{\R}{\mathbb{R}}
\newcommand{\T}{\top}

\newcommand{\Pb}{\mathbb{P}}
\newcommand{\Coverage}{\operatorname{Coverage}}

\newcommand{\cA}{\mathcal{A}}

\newcommand{\cD}{\mathcal{D}}
\newcommand{\cE}{\mathcal{E}}

\newcommand{\cI}{\mathcal{I}}

\newcommand{\cN}{\mathcal{N}}

\newcommand{\cX}{\mathcal{X}}
\newcommand{\cY}{\mathcal{Y}}

\newcommand{\bX}{\mathbf{X}}

\newcommand{\E}{\mathbb{E}}

\newcommand{\bbx}{\mathbf{x}}

\newcommand{\bbone}{\boldsymbol{1}}

\usepackage{xcolor} %

\usepackage{threeparttable}

\makeatletter
\def\thmt@innercounters{}
\let\paper@label\label
\let\paper@ltx@label\ltx@label
\newcommand{\appendixonlyequationlabel}[1]{%
    \ifthmt@thisistheone
        \notag
    \else
        \let\label\paper@label
        \let\ltx@label\paper@ltx@label
        \label{#1}%
    \fi
}
\makeatother

\usepackage[hidelinks,hypertexnames=false]{hyperref}
\let\appendixtocaddcontentsline\addcontentsline
\renewcommand{\addcontentsline}[3]{}

\floatstyle{ruled}
\restylefloat{algorithm}
\algrenewcommand\algorithmicrequire{\textbf{Input:}}
\algrenewcommand\algorithmicensure{\textbf{Output:}}

\title{Split Conformal Interval Length under Huber Contamination}
\author{Xiao Zhang$^1$ \quad Bowen Zhang$^1$ \quad Lie He$^{1,2}$\thanks{Corresponding author.}  \\
$^{1}$Shanghai University of Finance and Economics \\
$^{2}$MoE Key Laboratory of Interdisciplinary Research of Computation and Economics,\\
  \;  Shanghai University of Finance and Economics\\
\texttt{\{xiaozhang,bowenzhang\}@stu.sufe.edu.cn}\quad\texttt{helie@sufe.edu.cn} 
}
\iclrfinalcopy

\begin{document}

\maketitle
\lhead{Preprint. Under review.}

\begin{abstract}
Existing analyses of interval length in split conformal prediction largely focus on clean data. However, how contamination at the training and calibration stages alters interval length relative to the corresponding clean population oracle remains understudied. In this paper, we consider Huber contamination and derive a finite-sample upper bound that holds uniformly over contaminating distributions, with leading terms scaling as $L(\delta_1+n_1^{-1/2}) +(n_2^{-1/2}+\delta_2)/((1-\delta_2)r_0)$, in addition to a clean-sample estimation term. We further establish a minimax lower bound matching the upper-bound rate up to constant factors on a regular asymmetric location submodel. For calibration contamination, we obtain an exact worst-case characterization of the population interval-length deviation. We also extend the upper-bound results to conformalized quantile regression and adaptive residual scores. Experiments on synthetic and real-world datasets provide empirical support for our theoretical findings on training and calibration contamination.
\end{abstract}

\section{Introduction}
Modern regression methods often provide accurate point predictions. However, point predictions alone do not quantify predictive uncertainty, which can be important for downstream decision making. For example, uncertainty in wind and solar generation is relevant to reserve scheduling, balancing, and dispatch in power systems~\citep{pinson2013wind,hong2016probabilistic}. \emph{Conformal prediction} (CP) provides a general framework for constructing prediction sets with distribution-free, finite-sample marginal coverage guarantees~\citep{vovk2005algorithmic}. Among its variants, split conformal prediction (SCP) fits a predictive model on a training split and calibrates the prediction set using held-out data~\citep{papadopoulos2002inductive}. Given a target miscoverage level $\alpha\in(0,1)$, split conformal prediction constructs a prediction set $C_\alpha$ such that, for a new test observation
$(\bX_{\mathrm{new}},Y_{\mathrm{new}})$,
\[
\Pb\!\left(
Y_{\mathrm{new}}\in C_\alpha(\bX_{\mathrm{new}})
\right)
\ge 1-\alpha,
\]
provided that the calibration observations and the test observation are exchangeable.

In regression, the informativeness of a conformal predictor is commonly measured by prediction-interval length, with shorter intervals preferred at the target coverage level. Existing analyses of conformal efficiency, however, largely focus on clean data~\citep{lei2018distribution,sesia2020comparison}, while a separate line of work has studied the robustness of conformal prediction under various contamination and nonexchangeability models~\citep{barber2023conformal,li2024data,zargarbashi2024robust,yang2026ensemble,clarkson2024split}.

The effect of contamination on prediction-interval length remains less understood. We characterize its deviation from the clean population oracle in terms of the training and calibration contamination levels, uniformly over
 contaminating distributions. In split conformal prediction, the two stages affect interval length through different mechanisms: training contamination perturbs the fitted predictor, while calibration contamination shifts the score quantile.

To address this question, we study split conformal regression under Huber contamination in both the training and calibration samples, with contamination levels $\delta_1,\delta_2$ and sample sizes $n_1,n_2$, while allowing the contaminating distributions to be arbitrary. Our main contributions are as follows:
\begin{itemize}[leftmargin=10pt, itemindent=0pt, labelsep=5pt]
\item \textbf{Uniform upper bounds under joint contamination.}
For absolute-residual scores and $\delta_2<\alpha$, under standard stability and regularity conditions,
we show that the prediction-interval length deviates from its clean oracle at
the rate
\[
O_P\!\left(
\eta_{n_1}(P_0)
+
L\bigl(\delta_1+n_1^{-1/2}\bigr)
+
\frac{n_2^{-1/2}+\delta_2}
     {(1-\delta_2)r_0}
\right),
\]
uniformly over contaminating distributions.  Here $\eta_{n_1}(P_0)$ is the clean sample estimation error, while $L$
and $r_0^{-1}$ capture the training- and calibration-side sensitivities, respectively. We further extend the interval-length analysis to conformalized quantile regression and adaptive residual scores, and derive a corresponding bound on set-level discrepancy. For calibration contamination, we additionally obtain
an exact worst-case characterization of interval-length inflation and contraction over contaminating distributions.

\item \textbf{A minimax lower bound on a location submodel.}
On a regular asymmetric location submodel, we establish a minimax lower bound of order
\[
L_\star\bigl(\delta_1+n_1^{-1/2}\bigr)
+
\frac{n_2^{-1/2}+\delta_2}{r_\star}.
\]
Under suitable local comparability conditions on the same submodel, this lower bound matches the upper rate up to constant factors for absolute-residual split conformal prediction.

\item \textbf{Empirical characterization.}
Across synthetic and real-world regression datasets, we empirically characterize the distinct effects of training and calibration contamination and compare their behavior across absolute-residual, adaptive, and CQR scores. The results illustrate the stage-specific mechanisms predicted by our theory and their different sensitivity to contamination.
\end{itemize}

\section{Related Work}
\paragraph{Conformal prediction and efficiency analysis.}
Conformal prediction constructs prediction sets with finite-sample marginal coverage guarantees for exchangeable observations~\citep{papadopoulos2002inductive,vovk2005algorithmic,vovk2009line}. The efficiency of a conformal predictor characterizes the informativeness of its prediction sets and can be evaluated using different criteria~\citep{vovk2016criteria}. In regression, efficiency is commonly measured by prediction-set length or volume, often relative to a population-level oracle benchmark. Prior work has studied such efficiency guarantees for nonparametric prediction bands~\citep{lei2014distribution}, conformalized quantile regression and distributional conformal prediction~\citep{sesia2020comparison,chernozhukov2021distributional}, and conditional-density-based prediction regions in high dimensions~\citep{izbicki2022cd}. Other work derives finite-sample bounds on prediction-set size~\citep{dhillon2024expected,zecchin2024generalization,yao2026non} or directly optimizes set length or volume~\citep{kiyani2024length,braun2025minimum,lebars2025volume,gao2025volume}. Efficiency improvements have also been developed for data-reuse procedures such as cross-conformal prediction through more effective aggregation across folds~\citep{gasparin2025improving}. Our analysis complements this literature by studying efficiency under training and calibration contamination. When $\delta_1=\delta_2=0$, our decomposition reduces to the standard clean-data setting of \citet{lei2018distribution}, in which only finite-sample training and calibration errors remain. 
Under contamination, additional population-level shifts in the predictor and score quantile may arise and need not vanish asymptotically.

\paragraph{Conformal prediction under contamination.}
Robustness of conformal prediction has been studied under departures from exchangeability and various forms of data contamination. \citet{barber2023conformal} analyze coverage under general nonexchangeability, including Huber contamination~\citep{huber1964robust}, and derive a multiplicative clean-coverage correction of $\alpha'=\alpha(1-\delta_2)$ under $\delta_2$-Huber contamination. In classification, prior work considers training poisoning, test-time perturbations, calibration poisoning, and random label 
noise~\citep{li2024data,zargarbashi2024robust,einbinder2024label}.
\citet{yang2026ensemble} propose EnCP, a partition-and-vote defense against bounded poisoning of training or calibration samples. Their method modifies the conformal procedure through an ensemble construction, whereas we study unmodified split conformal regression under distributional contamination.
The work most closely related to our calibration-contamination setting is \citet{clarkson2024split}, 
who bound the absolute difference between the expected clean and contaminated calibration order statistics by $\delta_2 C(n_2,k_{n_2})W_1(\Pi_{P_0},\Pi_{Q_2})$,
where $W_1$ denotes the 1-Wasserstein distance between the clean and contaminating score distributions. This bound may be less informative in the worst-case setting: $W_1(\Pi_{P_0},\Pi_{Q_2})$ may be unbounded, whereas our high-probability bound is uniform over $Q_2$; moreover, for $k_{n_2}=\lceil(n_2+1)(1-\alpha)\rceil$, $C(n_2,k_{n_2})=\Theta(\sqrt{n_2})$. Our bound instead separates a vanishing $O(n_2^{-1/2})$ sampling term
from an $O(\delta_2/r_0)$ population-contamination term and is uniform over $Q_2$.

\section{Preliminaries and Setups}
\subsection{Preliminaries on Split Conformal Prediction}
We consider regression with feature space $\cX\subseteq\R^d$ and response space $\cY=\R$, and let $P_0$ denote the clean distribution on $\cX\times\cY$. Fix a target miscoverage level $\alpha\in(0,1)$. Let $A$ denote a training source, either a population distribution or a realized training sample, and let $\mu(\cdot;A)$ denote the corresponding population or fitted predictor. For the absolute-residual nonconformity score, define
\[
R_A(\bbx,y):=|y-\mu(\bbx;A)|, \qquad (\bbx,y)\in\cX\times\cY.
\]
For any distribution $B$ on $\cX\times\cY$, let
\[
F_{A,B}(t)
:=
\Pb_{(X,Y)\sim B}\bigl(R_A(X,Y)\le t\bigr),
\qquad
q(A;B)
:=
F_{A,B}^{-1}(1-\alpha),
\]
where $F_{A,B}^{-1}(u):=\inf\{t\in\R:F_{A,B}(t)\ge u\}$, $u\in(0,1]$. When $A$ is a realized training sample, $F_{A,B}$ and $q(A;B)$ are understood conditionally on $A$.

Let $\cD_1$ and $\cD_2$ denote training and calibration samples of
sizes $n_1$ and $n_2$, respectively, with their sampling
distributions specified below. Write
$\cD_2=\{(X_i,Y_i)\}_{i=1}^{n_2}$. For a training source $A$,
define the calibration scores
\[
R_{A,i}:=R_A(X_i,Y_i),
\qquad i=1,\ldots,n_2,
\]
and let $R_{A,(1)}\le\cdots\le R_{A,(n_2)}$ denote their order statistics. Set
$k_{n_2}:=\left\lceil (n_2+1)(1-\alpha)\right\rceil$ and define the split-conformal calibration quantile by
\[
\widehat q(A;\cD_2)
:=
\begin{cases}
R_{A,(k_{n_2})}, & k_{n_2}\le n_2,\\
+\infty, & k_{n_2}=n_2+1.
\end{cases}
\]
The split conformal prediction interval based on $\cD_1$ is
\[
C_\alpha(x;\cD_1,\cD_2)
:=
\bigl[
\mu(x;\cD_1)-\widehat q(\cD_1;\cD_2),
\,
\mu(x;\cD_1)+\widehat q(\cD_1;\cD_2)
\bigr].
\]
If, conditional on $\cD_1$, the calibration observations and a new
test observation $(X_{\mathrm{new}},Y_{\mathrm{new}})$ are
exchangeable, then
\[
\Pb\!\left(
Y_{\mathrm{new}}
\in
C_\alpha(X_{\mathrm{new}};\cD_1,\cD_2)
\,\middle|\,
\cD_1
\right)
\ge 1-\alpha.
\]
For the absolute-residual score, the interval length is
$2\widehat q(\cD_1;\cD_2)$.
The corresponding clean population oracle is
\[
C_\alpha^0(x)
:=
\bigl[
\mu(x;P_0)-q(P_0;P_0),
\,
\mu(x;P_0)+q(P_0;P_0)
\bigr].
\]
Our analysis compares the fitted interval under contamination with this clean oracle.

To control the sampling error of the fitted clean predictor, we use the following condition.
\begin{assumption}[Sampling stability~\citep{lei2018distribution}]
\label{ass:sampling-stability}
For every fixed training distribution $P$, there exist deterministic
sequences $\eta_n(P)=o(1)$ and $\rho_n(P)=o(1)$ such that, for
$\cD_n\sim P^n$ and all sufficiently large $n$,
\[
\Pb_{\cD_n\sim P^n}\left(
\left\|\mu(\cdot;\cD_n)-\mu(\cdot;P)\right\|_{\infty}
\ge\eta_n(P)
\right)
\le \rho_n(P),
\]
where
$\|g\|_\infty:=\sup_{\bbx\in\cX}|g(\bbx)|$.
\end{assumption}
\subsection{Training and Calibration Contamination}
\label{sec:contamination}

We model contamination in the training and calibration stages separately
using Huber's contamination model \citep{huber1964robust}:
\[
P_1 := (1-\delta_1)P_0+\delta_1 Q_1,
\qquad
P_2 := (1-\delta_2)P_0+\delta_2 Q_2,
\]
where $\delta_1,\delta_2\in[0,1)$ are the contamination levels and $Q_1,Q_2$ are arbitrary distributions on $\cX\times\cY$. For the training stage, we construct a clean dataset $\cD_0\sim P_0^{n_1}$ and a contaminated dataset $\cD_1\sim P_1^{n_1}$ on a common probability space through Bernoulli replacement. Specifically, each observation in $\cD_0$ is independently replaced by a draw from $Q_1$ with probability $\delta_1$. Hence, $\cD_0$ and $\cD_1$ differ in at most $K$
observations, where $K\sim\operatorname{Binomial}(n_1,\delta_1)$. The calibration dataset $\cD_2\sim P_2^{n_2}$ is independent of $(\cD_0,\cD_1)$. 

\begin{remark}[Fixed-fraction replacement]
An alternative is to replace exactly \(k=\lfloor n\delta\rfloor\) uniformly selected observations. Unlike the Bernoulli replacement induced by the Huber contamination model, the number of replacements is then fixed rather than random. Our analysis extends to this fixed-fraction scheme: the same contamination-dependent terms are retained, while the finite-sample fluctuation due to the random replacement count is removed. Appendix~\ref{app:bernoulli-fixed-fraction} gives the corresponding bound and an empirical comparison of the two schemes.
\end{remark}

For any fixed $A$, the score CDF under $P_2$ satisfies
\newcommand{\FixedScoreMixtureIdentity}{%
F_{A,P_2}(t)
=(1-\delta_2)F_{A,P_0}(t)+\delta_2F_{A,Q_2}(t)%
}
\begin{equation*}
\FixedScoreMixtureIdentity.
\end{equation*}
The identity compares score distributions while holding the fitted predictor fixed. For our interval-length analyses, we assume $\alpha\in(0,1/2)$ to simplify the two-sided calibration bounds. For the clean oracle $A=P_0$, define
\begin{equation}
q_-(\delta_2):=F_{P_0,P_0}^{-1}
\left(\frac{1-\alpha-\delta_2}{1-\delta_2}\right),
\qquad
q_+(\delta_2):=F_{P_0,P_0}^{-1}
\left(\frac{1-\alpha}{1-\delta_2}\right).
\label{eq:contamination-quantile-endpoints}
\end{equation}
For the calibration-side analysis, we assume $\delta_2\in[0,\alpha)$.
\begin{assumption}[Local score-density regularity]
\label{ass:score-density}
The clean oracle score CDF $F_{P_0,P_0}$ is absolutely continuous on an
open interval $\cI\supset[q_-(\delta_2),q_+(\delta_2)]$, with density
$f_{P_0,P_0}$ satisfying
\[
\inf_{t\in\cI}f_{P_0,P_0}(t)\ge r_0>0.
\]
\end{assumption}

\section{Robustness Analysis Under Contamination}

\subsection{Finite-Sample Robustness under Joint Contamination}
\label{ssec:interval-length-deviation}

To isolate the effects of contamination from finite-sample variation at the training and calibration stages, we use the following decomposition of the quantile error:
\newcommand{\QuantileErrorDecomposition}{%
\left|\widehat q(\cD_1;\cD_2)-q(P_0;P_0)\right|
\le&
\underbrace{
\left|\widehat q(\cD_1;\cD_2)-\widehat q(\cD_0;\cD_2)\right|
}_\text{Lemma~\ref{lem:training-contamination-quantile-error}: Training contamination error}
+
\underbrace{
\left|\widehat q(\cD_0;\cD_2)-\widehat q(P_0;\cD_2)\right|
}_\text{Lemma~\ref{lem:clean-training-sampling-quantile-error}: Training-stage sampling error}
\nonumber\\
&+
\underbrace{
\left|\widehat q(P_0;\cD_2)-q(P_0;P_2)\right|
}_\text{Lemma~\ref{lem:empirical-quantile-error-under-p2}: Calibration-stage sampling error}
+
\underbrace{
\left|q(P_0;P_2)-q(P_0;P_0)\right|
}_\text{Lemma~\ref{lem:calibration-contamination-quantile}: Calibration contamination error}%
}
\begin{align*}
\QuantileErrorDecomposition.
\end{align*}

\subsubsection{Training-Side Error}

The first term in the preceding decomposition captures how training contamination changes the calibration quantile through the fitted predictor. We control this term uniformly over the contamination distribution $Q_1$ by bounding the predictor's sensitivity to multiple training-sample replacements. Specifically, motivated by uniform algorithmic stability \citep{bousquet2002stability,liu2017algorithmic,bassily2020stability}, we use the following multi-replacement version of the one-replacement condition studied by \citet{pournaderi2024training}.

\begin{definition}[Uniform replacement stability]
\label{def:replacement-stability}
A regression learning algorithm $\mu$ is said to be $L$-uniformly replacement stable if, for every pair of training datasets $\cD$ and $\cD'$ of size $n$ that differ in at most $k$ observations,
\begin{equation*}
\left\|
\mu(\cdot;\cD)-\mu(\cdot;\cD')
\right\|_\infty
\le
Lk/n,
\end{equation*}
where $L<\infty$ is independent of $n$, $k$, $\cD$, and $\cD'$.
\end{definition}

Next, we give an explicit stability constant for $\ell_2$-regularized linear quantile regression.
\begin{restatable}[$\ell_2$-regularized linear quantile regression]{example}{RQReplacementStabilityExample}
\label{ex:rq-replacement-stability}
Suppose that $\sup_{\bbx\in\cX}\|\bbx\|_2\le B_x$. For $\cD=\{(\bX_i,Y_i)\}_{i=1}^{n}$, define the $\ell_2$-regularized linear quantile regression estimator by
\begin{equation}
\hat\omega(\cD)
:=
\argmin_{\omega\in\R^d}
\left\{
\frac{1}{n}\sum_{i=1}^{n}
\rho_\tau(Y_i-\omega^\T\bX_i)
+
\frac{\lambda}{2}\|\omega\|_2^2
\right\},
\appendixonlyequationlabel{eq:rq-objective}
\end{equation}
where $\rho_\tau(u)=u\bigl(\tau-\bbone\{u<0\}\bigr)$ is the pinball loss, $\tau\in(0,1)$, and $\lambda>0$. Then the predictor $\mu(\bbx;\cD):=\bbx^\T\hat\omega(\cD)$ is $L$-uniformly replacement stable with $L=2B_x^2\max\{\tau,1-\tau\}/\lambda$.
\end{restatable}

Uniform replacement stability yields the following bound on the calibration-quantile perturbation:
\begin{restatable}[Training-contamination quantile perturbation]
{lemma}{TrainingContaminationQuantileLemma}
\label{lem:training-contamination-quantile-error}
If $\mu$ is $L$-uniformly replacement stable, then, for all sufficiently large $n_2$, the following holds with probability at least $1-n_1^{-2}$, uniformly over $Q_1$:
\begin{equation}
\left|
\widehat q(\cD_1;\cD_2)-\widehat q(\cD_0;\cD_2)
\right|
\le
L\left(
\delta_1+\sqrt{\frac{\log n_1}{n_1}}
\right).
\appendixonlyequationlabel{eq:training-contamination-quantile-error}
\end{equation}
\end{restatable}
Sampling stability bounds the calibration-quantile discrepancy between the fitted clean predictor and its population counterpart.
\begin{restatable}[Training-stage sampling quantile error]{lemma}{CleanTrainingSamplingQuantileLemma}
\label{lem:clean-training-sampling-quantile-error}
Under Assumption~\ref{ass:sampling-stability}, for all sufficiently large $n_1,n_2$, with probability at least
$1-\rho_{n_1}(P_0)$,
\begin{equation}
\left|\widehat q(\cD_0;\cD_2)-\widehat q(P_0;\cD_2)\right|
\le \eta_{n_1}(P_0).
\appendixonlyequationlabel{eq:clean-training-sampling-quantile-error}
\end{equation}
\end{restatable}

\subsubsection{Calibration-Side Error}

The calibration-stage error comprises empirical quantile fluctuation under $P_2$ and the contamination-induced shift of the population quantile.

\begin{restatable}[Finite-sample calibration quantile error]{lemma}{EmpiricalQuantilePtwoLemma}
\label{lem:empirical-quantile-error-under-p2}
Under Assumption~\ref{ass:score-density}, for every $\gamma\in(0,1)$
and all sufficiently large $n_2$, the following holds with probability at least $1-\gamma$:
\begin{equation}
\left|\widehat q(P_0;\cD_2)-q(P_0;P_2)\right|
\le
\frac{1}{(1-\delta_2)r_0}
\left(
\sqrt{\frac{\log(2/\gamma)}{2n_2}}
+
\frac{2}{n_2}
\right).
\appendixonlyequationlabel{eq:empirical-quantile-error-under-p2}
\end{equation}
\end{restatable}

\begin{wrapfigure}{r}{0.47\textwidth}
\centering
\includegraphics[width=\linewidth]{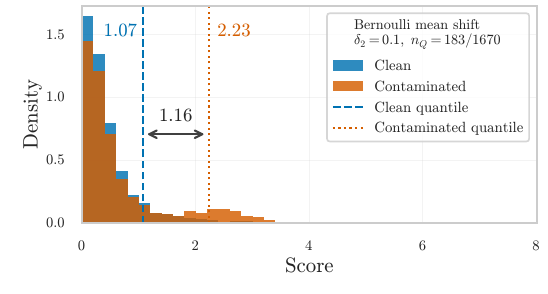}
\caption{Clean and contaminated absolute-score distributions on \texttt{Abalone}.}
\label{fig:abalone_score}
\end{wrapfigure}

Figure~\ref{fig:abalone_score} illustrates how upper-tail calibration contamination shifts the empirical score quantile on the \texttt{Abalone} dataset. The following lemma bounds the corresponding population-level shift.
\begin{restatable}[Quantile bias under calibration contamination]{lemma}{CalibrationContaminationQuantileLemma}
\label{lem:calibration-contamination-quantile}
Under Assumption~\ref{ass:score-density} and $\delta_2<\alpha$, the population
$(1-\alpha)$-quantile of the oracle score $R_{P_0}$ satisfies
\begin{equation}
\left|q(P_0;P_2)-q(P_0;P_0)\right|
\le
\frac{\delta_2(1-\alpha)}
{(1-\delta_2)r_0}.
\appendixonlyequationlabel{eq:population-quantile-error-from-delta2-contamination}
\end{equation}
\end{restatable}
The factor $r_0^{-1}$ captures local quantile sensitivity: the same probability-level perturbation can produce a larger quantile shift when the clean score density near the target quantile is smaller.

\subsubsection{Joint Interval-Length Deviation Bound}
Combining the preceding training- and calibration-stage bounds yields the following finite-sample guarantee under joint contamination.
\begin{restatable}[Interval-length deviation under joint contamination]{theorem}{JointContaminationTheorem}
\label{thm:delta}
Under Assumptions~\ref{ass:sampling-stability} and~\ref{ass:score-density}, suppose also that $\mu$ is $L$-uniformly replacement stable as in Definition~\ref{def:replacement-stability}. Fix $\delta_1\in[0,1)$ and $\delta_2\in[0,\alpha)$. For every $\gamma\in(0,1)$ and all sufficiently large $n_1,n_2$, the following holds with probability at least $1-\rho_{n_1}(P_0)-n_1^{-2}-\gamma$, uniformly over the contamination distributions $Q_1$ and $Q_2$:
\begin{equation}
\begin{aligned}
&2\left|
\widehat q(\cD_1;\cD_2)-q(P_0;P_0)
\right|
\\[-0.25em]
&\quad\le
2\eta_{n_1}(P_0)
+2L\left(\delta_1+\sqrt{\frac{\log n_1}{n_1}}\right)
+\frac{2}{(1-\delta_2)r_0}
\left(
\sqrt{\frac{\log(2/\gamma)}{2n_2}}
+\frac{2}{n_2}
+\delta_2(1-\alpha)
\right).
\end{aligned}
\appendixonlyequationlabel{eq:joint-contamination-finite-sample-bound}
\end{equation}
\end{restatable}
At fixed contamination levels, the sampling terms vanish as
$n_1,n_2\to\infty$, while the contamination-dependent terms remain
in the bound.

\begin{remark}[Extensions to alternative nonconformity scores]
\label{rem:alternative-score-extensions}
The training--calibration decomposition extends to conformalized quantile regression (CQR) and adaptive residual scores under score-specific stability and clean-score regularity conditions. For CQR, the interval-length deviation additionally accounts for perturbations of the fitted lower and upper quantile regressors. For adaptive scores, a uniform additive score perturbation need not hold under arbitrary calibration contamination, so the training-side comparison is handled multiplicatively under a positive lower bound on the fitted scale.
Appendix~\ref{ssec:alternative-score-extensions} gives the corresponding assumptions, finite-sample bounds, and proofs.
\end{remark}

\begin{figure}
    \centering
    \includegraphics[width=\textwidth]{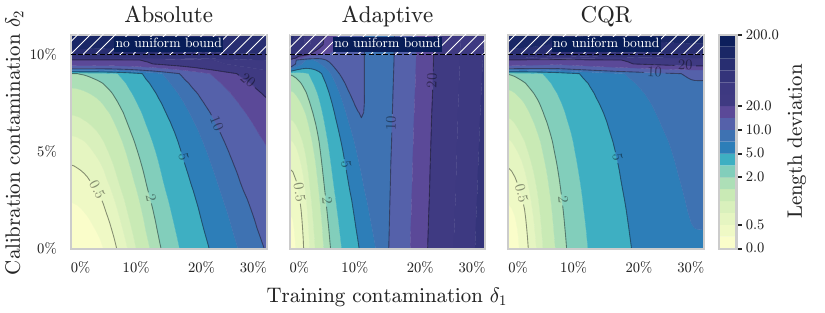}
    \caption{Joint training and calibration contamination for three conformal prediction nonconformity scores: Absolute, Adaptive, and CQR. Both stages use $Q_1=Q_2=\operatorname{Dirac}((-1,-1,-1,-1,+1),-80)$.
}
\label{fig:three_socre}
\end{figure}

\subsubsection{Set-level Discrepancy}
Beyond interval length, the discrepancy between prediction intervals
also depends on their center displacement. We therefore measure the
set-level discrepancy between the fitted interval and the clean oracle
interval by the Lebesgue measure of their symmetric difference.
Define
\[
\begin{aligned}
C_{\mathrm{clean}}(\bbx)
&:=
\left[
\mu(\bbx;P_0)-q(P_0;P_0),
\mu(\bbx;P_0)+q(P_0;P_0)
\right], \quad
C_{\mathrm{cont}}(\bbx)
:=
C_\alpha(\bbx;\cD_1,\cD_2).
\end{aligned}
\]
Let $\operatorname{Leb}$ denote the Lebesgue measure on $\R$ and
$\triangle$ denote the symmetric difference. Comparing interval endpoints
gives
\begin{align*}
&\sup_{\bbx\in\cX}
\operatorname{Leb}\!\left(
C_{\mathrm{cont}}(\bbx)\triangle C_{\mathrm{clean}}(\bbx)
\right)
\le
2\left\|\mu(\cdot;\cD_1)-\mu(\cdot;P_0)\right\|_\infty
+
2\left|
\widehat q(\cD_1;\cD_2)-q(P_0;P_0)
\right|.
\end{align*}
Assumption~\ref{ass:sampling-stability} and
Proposition~\ref{prop:empirical-perturbation-replacement-stability}
give the following bound, uniformly over $Q_1$:
\begin{align*}
\Pb\!\left\{
\left\|\mu(\cdot;\cD_1)-\mu(\cdot;P_0)\right\|_\infty
>
\eta_{n_1}(P_0)
+
L\left(\delta_1+\sqrt{\frac{\log n_1}{n_1}}\right)
\right\}
\le
\rho_{n_1}(P_0)+n_1^{-2}.
\end{align*}

The center bound uses the same training events as Theorem~\ref{thm:delta}. Combining the two bounds therefore yields the following corollary without an additional failure probability.

\begin{corollary}[Set-level discrepancy]
\label{cor:set-level-discrepancy}
Under the assumptions of Theorem~\ref{thm:delta}, for every
$\gamma\in(0,1)$ and all sufficiently large $n_1,n_2$, with probability at
least $1-\rho_{n_1}(P_0)-n_1^{-2}-\gamma$,
\begin{align*}
&\sup_{\bbx\in\cX}
\operatorname{Leb}\!\left(
C_{\mathrm{cont}}(\bbx)\triangle C_{\mathrm{clean}}(\bbx)
\right)
\\
&\quad\le
4\eta_{n_1}(P_0)
+
4L\left(\delta_1+\sqrt{\frac{\log n_1}{n_1}}\right)
+
\frac{2}{(1-\delta_2)r_0}
\left(
\sqrt{\frac{\log(2/\gamma)}{2n_2}}
+
\frac{2}{n_2}
+
\delta_2(1-\alpha)
\right).
\end{align*}
\end{corollary}

\subsection{A Joint Minimax Lower Bound}
\label{ssec:joint-minimax-lower-bound}
To assess the sharpness of the component rates, we consider the location submodel
\[
\bX=0,
\qquad
Y=\theta+\varepsilon,
\qquad
\theta\in\Theta:=[-h_0/2,h_0/2],
\]
where $h_0>0$ and the continuously differentiable noise density $g$ satisfy Assumption~\ref{ass:least-favorable-location}.
Let $q_\star$ be the $(1-\alpha)$-quantile of $|\varepsilon|$, so that the clean oracle interval has length $2q_\star$. The local information scale $L_\star$ controls the distinguishability of nearby location parameters, while $r_\star$ characterizes the score density near the target quantile.
The proof combines two-point indistinguishability arguments \citep{lecam1973convergence} for the training terms with
order-statistic and coupling arguments for the calibration terms.
\begin{theorem}[Informal joint minimax lower rate]
\label{thm:joint-minimax-lower-rate-informal}
For the regular location submodel above, consider standard absolute-residual split conformal with measurable, possibly randomized training predictors $\widehat\theta(\cD_1)\in\Theta$. {} For a fixed submodel and all sufficiently large $n_1,n_2$, the local minimax interval-length deviation from $2q_\star$ is
\[
\Omega\!\left(
L_\star\bigl(n_1^{-1/2}+\delta_1\bigr)
+
\frac{n_2^{-1/2}+\delta_2}{r_\star}
\right).
\]
This lower bound holds with probability at least some constant $c_0>0$. 
\end{theorem}

If $r_\star$ and $r_0$ are comparable up to fixed constant factors and the additional fixed-confidence training conditions in Appendix~\ref{ssec:fixed-confidence-matching} hold, this lower rate matches an attainable upper rate on the same fixed submodel. The comparison establishes local rate sharpness within the absolute-residual split-conformal architecture and the stated algorithm class. A formal version, including the precise assumptions and minimax
quantifiers, is given in Theorem~\ref{thm:joint-minimax-lower-bound}.

\subsection{Calibration-Specific Analysis}
\subsubsection{Exact Worst-Case Inflation, Contraction, and Breakdown}
\label{ssec:worst-case-calibration-contamination}
The preceding analysis provides an upper bound on the population
interval-length deviation induced by calibration contamination.
We now sharpen this analysis by characterizing the exact worst case
over the contamination distribution $Q_2$, and by identifying
contaminating distributions that attain the maximal inflation and contraction.
Define
\begin{equation*}
\mathfrak D_\alpha(\delta_2)
:=
\sup_{Q_2}
2\left|q(P_0;P_2)-q(P_0;P_0)\right|.
\end{equation*}

\begin{restatable}[Exact worst-case calibration deviation]
{proposition}{WorstCaseCalibrationProposition}
\label{prop:worst-case-length-deviation}
Suppose $F_{P_0,P_0}(t)<1$ for every finite $t$ and
$\alpha\in(0,1/2)$. Then, for $\delta_2\in[0,1)$,
\begin{equation}
\mathfrak D_\alpha(\delta_2)
=
\begin{cases}
2\max\left\{
q_+(\delta_2)-q(P_0;P_0),
q(P_0;P_0)-q_-(\delta_2)
\right\},
& 0\le\delta_2<\alpha,\\[0.5em]
+\infty,
& \alpha\le\delta_2<1.
\end{cases}
\appendixonlyequationlabel{eq:worst-case-length-deviation}
\end{equation}
\end{restatable}

For $\delta_2<\alpha$, both the maximal inflation and maximal contraction
are attained by point masses. For any $\bbx_0\in\cX$ and $M\ge 0$, let
\[
Q_M
:=
\operatorname{Dirac}(\bbx_0,\mu(\bbx_0;P_0)+M),
\]
which places all contaminated score mass at $M$. Then $M=0$ attains
maximal contraction, whereas any $M>q_+(\delta_2)$ attains maximal
inflation. For $\delta_2\ge\alpha$, letting $M\to\infty$ yields
arbitrarily large inflation. Hence, the breakdown is driven by interval inflation: contraction remains
bounded, whereas maximal inflation diverges as $\delta_2\uparrow\alpha$.
Thus, $\alpha$ is the exact breakdown threshold.

\subsubsection{Nominal Miscoverage-Level Correction}
\label{ssec:nominal-level-correction}

We now turn from interval-length robustness to coverage under the clean distribution $P_0$. When the calibration sample follows the contaminated distribution $P_2$, we seek a nominal miscoverage level $\alpha'$ that guarantees coverage at least $1-\alpha$ under $P_0$. Let $\alpha\in(0,1)$ be the target miscoverage level under $P_0$, and let
$\alpha'\in(0,1)$ be the nominal miscoverage level used by split conformal prediction with calibration data from $P_2$. Let $\Coverage_P(\alpha')$ denote the marginal coverage of the resulting prediction interval when the test observation follows $P$. \citet{barber2023conformal} give a multiplicative correction of $\alpha'_{\mathrm{tight}} := \alpha(1-\delta_2)$, which guarantees
\[
\Coverage_{P_0}(\alpha'_{\mathrm{tight}})
\ge
1-\alpha.
\]
We further show that this correction is asymptotically sharp.
\begin{restatable}[Asymptotic sharpness of the mixture-based correction]
{proposition}{NominalCorrectionSharpnessProposition}
\label{prop:nominal-correction-sharpness}
For every $\alpha'\in\bigl(\alpha(1-\delta_2),1\bigr)$, there exist a clean
distribution $P_0$, a contamination distribution $Q_2$, and an
absolute-residual split conformal procedure such that, for all sufficiently
large $n_2$,
\[
\Coverage_{P_2}(\alpha')
\ge
1-\alpha',
\qquad
\Coverage_{P_0}(\alpha')
<
1-\alpha.
\]
\end{restatable}

\section{Experiments}

We evaluate interval-length deviation under training and calibration contamination on synthetic and real regression data, relative to the clean population oracle and paired clean split-conformal intervals, respectively.
The synthetic experiments isolate the two contamination stages and examine the behavior as $\delta_2$ approaches $\alpha$, while the real-data experiments assess calibration contamination under structured distributional changes. Experimental setups and supplementary results are provided in the appendices.

\subsection{Synthetic Experiments}

\begin{figure}[!t]
\centering
\begin{subfigure}{0.49\linewidth}
    \centering
    \includegraphics[width=\textwidth]{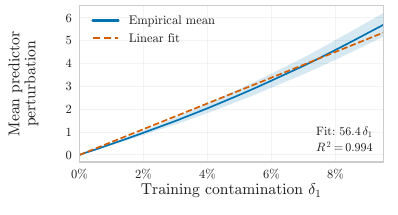}
    \caption{Predictor perturbation.}
    \label{fig:stability_constant}
\end{subfigure}
\begin{subfigure}{0.49\linewidth}
    \centering
    \includegraphics[width=\textwidth]{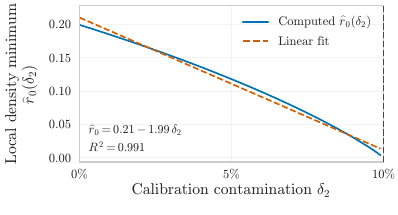}
    \caption{Local density minimum $\widehat r_0$.}
    \label{fig:lower_bound}
\end{subfigure}
\caption{Training and calibration sensitivity. (a) Predictor perturbation along a fixed contamination path. (b) Local score-density minimum $\widehat r_0(\delta_2)$ over the relevant quantile interval.}
\label{fig:Quantities}
\end{figure}

We generate $X=(X_1, \ldots, X_5)$ with $X_j\stackrel{\text{i.i.d.}}{\sim}\operatorname{Unif}([-1,1])$ and $Y=f(X)+\varepsilon$, where
\[
f(x)=\sin(\pi x_1)+x_2^2-\frac{1}{2}x_3
\]
and $\varepsilon\sim N(0,1)$. We fit $\ell_2$-regularized median regression using the 36 fixed features, with $n_1=n_2=5000$, $n_{\mathrm{test}}=10000$, and $\alpha=0.1$. 

Figure~\ref{fig:Quantities} illustrates the two sensitivity mechanisms underlying the interval-length deviation. Figure~\ref{fig:stability_constant} shows that the fitted-predictor perturbation grows approximately linearly with $\delta_1$ under the chosen contamination setting. The fitted slope therefore reflects the observed training-side sensitivity in this setting, rather than estimating the uniform stability constant $L$. Figure~\ref{fig:lower_bound} shows that the computed minimum of the clean score density over the relevant quantile interval, denoted by $\widehat r_0(\delta_2)$, decreases as $\delta_2$ increases, indicating greater calibration sensitivity through the inverse-density factor. Together, these patterns are consistent with the roles of $L$ and $r_0^{-1}$ in Theorem~\ref{thm:delta}.

Figure~\ref{fig:train_calibration} separates the effects of training and calibration contamination. The training-only deviation increases steadily with $\delta_1$, whereas the calibration-only deviation rises sharply as $\delta_2$ approaches $\alpha$, consistent with the breakdown behavior in Proposition~\ref{prop:worst-case-length-deviation}. Figure~\ref{fig:joint_contamination} shows that these effects compound under joint contamination: increasing $\delta_1$ raises the deviation across calibration levels, while the dependence on $\delta_2$ becomes increasingly steep near $\alpha$. These results are consistent with the distinct training- and calibration-side contributions in our theoretical analysis.

\begin{figure}[!t]
\centering
\begin{subfigure}[b]{0.49\linewidth}
    \centering
    \includegraphics[width=\linewidth]{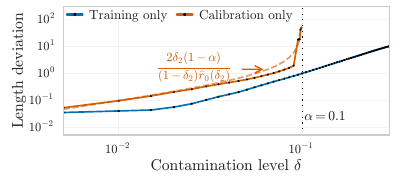}
    \caption{Training-only vs. calibration-only.}
    \label{fig:train_calibration}
\end{subfigure}
\begin{subfigure}[b]{0.49\linewidth}
    \centering
    \includegraphics[width=\linewidth]{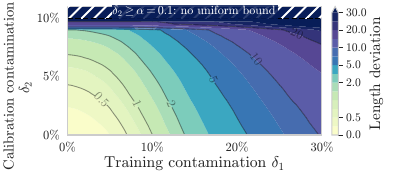}
    \caption{Joint contamination over $(\delta_1,\delta_2)$.}
    \label{fig:joint_contamination}
\end{subfigure}
\caption{Mean interval-length deviation under point-mass contamination: $Q_1=Q_2=\operatorname{Dirac}(x^\star,-80)$, with $x^\star=(-1,-1,-1,-1,+1)$.}
\label{fig:joint_length}
\vspace{-0.5cm}
\end{figure}

\begin{figure}[!t]
\centering
\includegraphics[width=\textwidth]{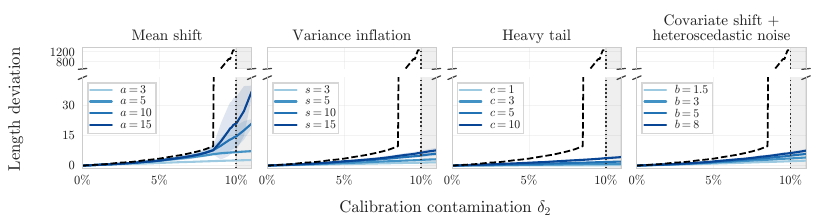}
\caption{Calibration-only interval-length deviation on \texttt{Abalone}
under four structured contamination families using the absolute-residual score.
The dashed curve shows the finite high-score Dirac reference with score $200$ for $\delta_2<\alpha$.}
\label{fig:various_contami}
\end{figure}

\subsection{Real-Data Experiments}

We evaluate calibration contamination on four real-world regression datasets: \texttt{Abalone}, \texttt{cadata}, \texttt{CPU Small}, and \texttt{YearPredictionMSD}~\citep{chang2011libsvm}. At $\alpha=0.1$, we consider four structured contamination families: mean shift, variance inflation, heavy-tailed noise, and covariate shift with heteroscedastic noise.

Figure~\ref{fig:various_contami} presents the \texttt{Abalone} results in the main text; results for the remaining datasets are reported in Appendix~\ref{app:real_data}.Across the four contamination families, the interval-length deviation generally increases with both $\delta_2$ and contamination severity. Several settings exhibit a sharp increase as $\delta_2$ approaches $\alpha$, consistent with the increasing calibration sensitivity predicted by our analysis. The structured contaminations nevertheless remain below the finite-score Dirac curve over most of the pre-breakdown regime, with substantial variation across contamination families.

\section{Conclusion and Future work}
We analyze how training and calibration contamination alter split-conformal interval length relative to the clean population oracle. Under Huber contamination, we establish a finite-sample upper bound uniform over contaminating distributions and a minimax lower bound matching its rate up to constant factors within a specified model class. For calibration contamination, we further characterize the exact worst-case population interval-length deviation. Future work could study contamination effects for conformal procedures that explicitly optimize interval length or prediction-set volume.

\newpage
\subsection*{AI use statement}
We used generative AI tools to assist with language editing, terminology improvement, literature searches, and reference checks. For the mathematical analyses, these tools assisted with developing and checking proofs, reviewing assumptions and derivations, and clarifying notation. For the experiments, they assisted with reviewing the experimental design, implementing and debugging code, running simulations, analyzing results, and preparing figures. The authors take full responsibility for the final manuscript, including its mathematical claims, code, and experimental results.

\subsection*{Ethics statement}
Conformal prediction can help users assess the reliability of regression predictions, particularly in high-stakes applications. Real-world datasets may contain contaminated observations, which can undermine the validity and informativeness of prediction intervals. We study how contamination in the training and calibration samples affects the coverage and interval length of split conformal prediction. Our analysis aims to clarify the reliability, efficiency, and limitations of split conformal prediction under contamination.

The theoretical guarantees rely on the assumptions stated in this paper, and practical use in high-stakes settings requires application-specific validation. Our experiments use synthetic data and publicly available regression benchmarks, and no new data were collected from human participants.
\subsection*{Reproducibility statement}
To support reproducibility, we document the experimental procedures and implementation details.
Appendix~\ref{app:common_protocol} describes the common experimental pipeline, prediction-set construction, evaluation metrics, and aggregation across repetitions.
Appendices~\ref{app:experimental_details} and~\ref{app:real_data} detail the synthetic and real-data experiments, including data generation or preprocessing, data splits, model training, contamination mechanisms, and hyperparameter settings.
Appendix~\ref{app:grids_diagnostics} specifies the parameter grids, diagnostic computations, random seeds, and software versions.
The fixed-fraction replacement experiments are described in Appendix~\ref{app:bernoulli-fixed-fraction}.
We provide the source code in the supplementary material.

\bibliography{ref}
\bibliographystyle{iclr2027_conference}

\newpage
\appendix

\let\addcontentsline\appendixtocaddcontentsline
\begingroup
\renewcommand{\contentsname}{Appendix Contents}
\setcounter{tocdepth}{2}
\tableofcontents
\endgroup
\clearpage

\section{List of Notations} \label{sec:appendix-notation}

    \begin{longtable}{p{4cm} p{9cm}}
    \hline
    \textbf{Symbol} & \textbf{Description} \\
    \hline
    \endfirsthead

    \hline
    \textbf{Symbol} & \textbf{Description} \\
    \hline
    \endhead

    \hline
    \endfoot

    \hline
    \endlastfoot

    \multicolumn{2}{l}{\textbf{Data, distributions, and contamination}} \\
    $(\bX_i,Y_i)$, $\bbx\in\cX\subseteq\R^d$ & Observation pairs and a generic feature value in the $d$-dimensional feature space. \\
    $\cD_0,\cD_1,\cD_2$; $n_1,n_2$ & Coupled clean training sample, contaminated training sample, calibration sample, and their sample sizes. \\
    $P$ & A generic distribution on the observation space; subscripts identify its role. \\
    $P_0$ & Clean distribution. \\
    $Q$, $\delta$, $P_\delta$ & Contamination distribution, contamination level, and Huber mixture $P_\delta=(1-\delta)P_0+\delta Q$. \\
    $P_j,Q_j,\delta_j$, $j\in\{1,2\}$ & Training-side ($j=1$) and calibration-side ($j=2$) contaminated laws, contaminating laws, and contamination levels, with $P_j=(1-\delta_j)P_0+\delta_jQ_j$. \\

    \multicolumn{2}{l}{\textbf{Split conformal construction and clean target}} \\
    $A,B$ & Generic training source $A$ (a population law or realized dataset) and evaluation distribution $B$. \\
    $\mu(\cdot;\cD_n)$, $\mu(\cdot;P)$ & Fitted predictor and its deterministic population target under training law $P$. \\
    $\|g\|_\infty$ & Supremum norm $\sup_{\bbx\in\cX}|g(\bbx)|$. \\
    $\eta_n(P),\rho_n(P)$ & Sampling-stability error and failure-probability sequences in Assumption~\ref{ass:sampling-stability}. \\
    $K$ & Number of contaminated training observations. \\
    $\alpha$ & Target miscoverage level under the clean distribution $P_0$. \\
    $\gamma$ & User-specified failure probability in finite-sample deviation bounds. \\
    $R_A(\bbx,y)$, $R_{A,i}$ & Absolute-residual score and its value on calibration observation $i$. \\
    $F_{A,B}$, $\widehat F_{A,\cD_2}$ & Population and calibration empirical CDFs of the absolute-residual score. \\
    $q(A;B)$, $\widehat q(A;\cD_2)$ & Population $(1-\alpha)$-quantile and conformal empirical quantile. \\
    $k_{n_2}$ & Conformal order-statistic rank $\lceil(n_2+1)(1-\alpha)\rceil$. \\
    $C_\alpha(\cdot;\cD_1,\cD_2)$ & Split-conformal prediction interval; its length is $2\widehat q(\cD_1;\cD_2)$, and the clean oracle length is $2q(P_0;P_0)$. \\
    $q_-(\delta_2),q_+(\delta_2)$, $f_{P_0,P_0}$, $r_0$ & Local quantile endpoints, clean oracle-score density, and its positive lower bound in Assumption~\ref{ass:score-density}. \\
    $\mathfrak D_\alpha^+(\delta_2)$, $\mathfrak D_\alpha^-(\delta_2)$, $\mathfrak D_\alpha(\delta_2)$ & Worst-case interval-length inflation, contraction, and absolute deviation under calibration contamination. \\
    $L$, $k$ & Absolute-score replacement-stability constant and maximum number of replaced training observations. \\
    $B_x$, $\rho_\tau$, $\tau$, $\lambda$, $\omega$, $\hat\omega(\cD_1)$ & Feature-norm bound, pinball loss, quantile level, regularization parameter, coefficient vector, and fitted regularized quantile-regression estimator in Example~\ref{ex:rq-replacement-stability}. \\

    \multicolumn{2}{l}{\textbf{Alternative nonconformity scores}} \\
    $S$ & Score type in $\{\mathrm{Absolute},\mathrm{CQR},\mathrm{Adaptive}\}$. \\
    $\ell(\cdot;A),u(\cdot;A)$ & Lower and upper CQR regressors associated with training source $A$. \\
    $\mu(\cdot;A),\sigma(\cdot;A)$ & Location and scale regressors for the adaptive score; $\mu$ is the same location learner used by the absolute score. \\
    $R^S(A;\bbx,y)$, $R_i^S(A)$ & Score-$S$ nonconformity function and its calibration values. \\
    $F_{P_0,P}^S$, $q_S(P_0;P)$, $\widehat q_S(A;\cD_2)$ & Clean-oracle score CDF, population quantile, and conformal empirical quantile for score $S$. \\
    $\eta_{n_1}^{\mathrm{CQR}}(P_0)$, $\eta_{g,n_1}^{\mathrm{Adaptive}}(P_0)$ & Clean-sampling errors for the CQR regressors and adaptive component $g\in\{\mu,\sigma\}$. \\
    $L_{\mathrm{CQR}}$, $L_g^{\mathrm{Adaptive}}$ & Componentwise replacement-stability constants for the CQR regressors and adaptive component $g\in\{\mu,\sigma\}$. \\
    $\varepsilon_{g,n_1}^{\mathrm{Adaptive}}$ & Clean-anchored adaptive-component perturbation $\eta_{g,n_1}^{\mathrm{Adaptive}}(P_0)+L_g^{\mathrm{Adaptive}}(\delta_1+\sqrt{\log(n_1)/n_1})$.\\
    $[z]_+$ & Positive part $\max\{z,0\}$. \\
    $\varepsilon_{\mathrm{train},n_1}^S$, $\rho_{n_1}^S(P_0)$ & Training contribution to the score-quantile deviation and clean-sampling failure probability for score $S$. \\
    $\varepsilon_{\mathrm{cal},n_2}^S(\gamma)$ & Calibration contamination-and-sampling contribution to the score-quantile deviation. \\
    $\varepsilon_{q,n_1,n_2}^S(\gamma)$ & Total score-quantile deviation bound for score $S$. \\
    $r_0^S$ & Local positive density lower bound for the clean distribution of score $S$. \\
    $\underline\sigma$, $\overline\sigma$, $\kappa_{\sigma,n_1}^{\mathrm{Adaptive}}$, $\beta_{\mu,n_1}^{\mathrm{Adaptive}}$ & Global positive lower scale bound, clean-scale upper bound, relative scale perturbation $\varepsilon_{\sigma,n_1}^{\mathrm{Adaptive}}/\underline\sigma$, and normalized location perturbation $\varepsilon_{\mu,n_1}^{\mathrm{Adaptive}}/\underline\sigma$. \\

    \multicolumn{2}{l}{\textbf{Joint minimax lower bound}} \\
    $\Theta$, $P_\theta$, $\theta$ & Parameter space, clean submodel law, and its location parameter in the minimax lower bound. \\
    $q_\star$, $L_\star$, $h_0$, $s_0$, $r_\star$ & Clean radius, information scale, local radii, and score-density scale of the least-favorable location submodel. \\
    $\widehat\theta$, $\widehat\theta_1$, $\widehat q(u;\cD_2)$ & Estimator, its realized output $\widehat\theta(\cD_1)$, and the calibration quantile for the constant predictor $u\in\Theta$. \\
    $\mathcal E_{n_1}(\Theta)$ & Class of measurable, possibly randomized, $\Theta$-valued estimators in the lower bound. \\
    $\Pb_{\theta,Q_1,Q_2}$ & Joint probability law used in the minimax lower-bound experiments. \\
    $\mathfrak p_{n_1,n_2}^{\mathrm{split}}(t)$ & Minimax tail-probability profile for interval-length deviation. \\
    $a_{n_1}$, $b_{n_2}$ & Training-side and calibration-side minimax difficulty scales. \\
    $x\wedge y$ & Minimum $\min\{x,y\}$. \\

    \multicolumn{2}{l}{\textbf{Set discrepancy and coverage correction}} \\
    $C_{\mathrm{clean}},C_{\mathrm{cont}}$ & Clean oracle and contaminated split-conformal intervals in the set-discrepancy comparison. \\
    $\operatorname{Leb}$, $\triangle$ & Lebesgue measure on $\R$ and symmetric difference. \\
    $\alpha'$, $\alpha'_{\mathrm{tight}}$, $\alpha'_{\mathrm{naive}}$ & Nominal miscoverage level and its mixture-based and naive corrections. \\
    $\Coverage_P(\alpha')$ & Marginal coverage at nominal miscoverage $\alpha'$ under test law $P\in\{P_0,P_2,Q_2\}$. \\
    $O_{\Pb}(\cdot)$ & Big-O in probability. \\
\end{longtable}

\section{Proofs}

\subsection{Proofs of Lemmas for Theorem~\ref{thm:delta}}
\label{ssec:proof_lemmas_thm_delta}

\begin{lemma}[Order-statistic perturbation]
\label{lem:order-statistic-perturbation}
Let $a_1,\ldots,a_m$ and $b_1,\ldots,b_m$ be two collections of real
numbers. If
\[
\max_{1\le i\le m}|a_i-b_i|
\le c,
\]
then, for every $j\in\{1,\ldots,m\}$,
\[
|a_{(j)}-b_{(j)}|
\le c,
\]
where $a_{(j)}$ and $b_{(j)}$ denote the corresponding $j$-th order
statistics.
\end{lemma}

\begin{proof}
At least $j$ indices satisfy $b_i\le b_{(j)}$. For every such index,
$a_i\le b_i+c\le b_{(j)}+c$. Hence at least $j$ elements of the first
collection are no greater than $b_{(j)}+c$, which implies
$a_{(j)}\le b_{(j)}+c$. Interchanging the two collections gives
$b_{(j)}\le a_{(j)}+c$, and the result follows.
\end{proof}

\begin{proposition}[Training-contamination perturbation]
\label{prop:empirical-perturbation-replacement-stability}
If $\mu$ is $L$-uniformly replacement stable as in
Definition~\ref{def:replacement-stability}, then, for $n_1\ge2$, with
probability at least $1-n_1^{-2}$,
\begin{equation}
\left\|
\mu(\cdot;\cD_1)-\mu(\cdot;\cD_0)
\right\|_\infty
\le
L\left(
\delta_1+\sqrt{\frac{\log n_1}{n_1}}
\right),
\label{eq:empirical-perturbation-high-probability}
\end{equation}
uniformly over the contamination distribution $Q_1$.
\end{proposition}

\begin{proof}
For $i=1,\ldots,n_1$, independently draw
\[
Z_i^{(0)}\sim P_0,\qquad
Z_i^{(Q)}\sim Q_1,\qquad
U_i\sim\operatorname{Bernoulli}(\delta_1),
\]
and set
\[
Z_i^{(1)}
:=
\begin{cases}
Z_i^{(0)},&U_i=0,\\
Z_i^{(Q)},&U_i=1.
\end{cases}
\]
Then
\[
\cD_0:=(Z_i^{(0)})_{i=1}^{n_1}\sim P_0^{n_1},
\qquad
\cD_1:=(Z_i^{(1)})_{i=1}^{n_1}\sim P_1^{n_1}.
\]
The coupled datasets $\cD_1$ and $\cD_0$ differ in at most
$K:=\sum_{i=1}^{n_1}U_i$ observations, where
$K\sim\operatorname{Binomial}(n_1,\delta_1)$. The uniform
replacement-stability bound in
Definition~\ref{def:replacement-stability} therefore gives
\begin{equation}
\|\mu(\cdot;\cD_1)-\mu(\cdot;\cD_0)\|_\infty
\le L K/n_1
\label{eq:empirical-perturbation-replacement-stability}
\end{equation}
almost surely, proving \eqref{eq:empirical-perturbation-replacement-stability}.
Hoeffding's inequality yields
\[
\Pb\!\left\{
K/n_1>
\delta_1+\sqrt{\frac{\log n_1}{n_1}}
\right\}
\le \exp(-2\log n_1)=n_1^{-2}.
\]
Combining the two displays proves
\eqref{eq:empirical-perturbation-high-probability}, uniformly over $Q_1$.
\end{proof}

\TrainingContaminationQuantileLemma*
\begin{proof}[Proof of Lemma~\ref{lem:training-contamination-quantile-error}]
On the event in \eqref{eq:empirical-perturbation-high-probability}, the reverse
triangle inequality gives, for every calibration observation,
\[
\left|R_{\cD_1,i}-R_{\cD_0,i}\right|
\le
\left|\mu(\bX_i;\cD_1)-\mu(\bX_i;\cD_0)\right|
\le L\left(\delta_1+\sqrt{\frac{\log n_1}{n_1}}\right).
\]
Applying Lemma~\ref{lem:order-statistic-perturbation} at rank
$k_{n_2}$ proves \eqref{eq:training-contamination-quantile-error}. The event
has probability at least $1-n_1^{-2}$ uniformly over $Q_1$ by
Proposition~\ref{prop:empirical-perturbation-replacement-stability}.
\end{proof}

\CleanTrainingSamplingQuantileLemma*
\begin{proof}[Proof of Lemma~\ref{lem:clean-training-sampling-quantile-error}]
By Assumption~\ref{ass:sampling-stability}, with probability at least
$1-\rho_{n_1}(P_0)$,
\[
\left\|\mu(\cdot;\cD_0)-\mu(\cdot;P_0)\right\|_\infty
\le\eta_{n_1}(P_0).
\]
On this event, for every calibration observation,
\[
\left|R_{\cD_0,i}-R_{P_0,i}\right|
\le
\left|\mu(\bX_i;\cD_0)-\mu(\bX_i;P_0)\right|
\le\eta_{n_1}(P_0).
\]
Lemma~\ref{lem:order-statistic-perturbation} at rank $k_{n_2}$ proves
\eqref{eq:clean-training-sampling-quantile-error}.
\end{proof}

\CalibrationContaminationQuantileLemma*
For reference, the fixed-source mixture identity used in the proof is
\begin{equation}
\FixedScoreMixtureIdentity.
\label{eq:fixed-score-mixture}
\end{equation}
\begin{proof}[Proof of Lemma~\ref{lem:calibration-contamination-quantile}]
Since $0\le F_{P_0,Q_2}(t)\le1$, the mixture identity
\eqref{eq:fixed-score-mixture} gives
\[
(1-\delta_2)F_{P_0,P_0}(t)
\le
F_{P_0,P_2}(t)
\le
(1-\delta_2)F_{P_0,P_0}(t)
+\delta_2.
\]
Taking generalized inverses at level $1-\alpha$ yields
\begin{equation}
F_{P_0,P_0}^{-1}
\left(
\frac{1-\alpha-\delta_2}{1-\delta_2}
\right)
\le q(P_0;P_2)
\le
F_{P_0,P_0}^{-1}
\left(
\frac{1-\alpha}{1-\delta_2}
\right).
\label{eq:calibration-quantile-sandwich}
\end{equation}
The density lower bound makes $F_{P_0,P_0}^{-1}$
$1/r_0$-Lipschitz on this quantile interval. The probability-level
distances from $1-\alpha$ to the lower and upper endpoints in
\eqref{eq:calibration-quantile-sandwich} are
\[
\frac{\alpha\delta_2}{1-\delta_2}
\qquad\text{and}\qquad
\frac{(1-\alpha)\delta_2}
{1-\delta_2},
\]
respectively. Since $\alpha<1/2$, both are bounded by $(1-\alpha)\delta_2/(1-\delta_2)$, which proves
\eqref{eq:population-quantile-error-from-delta2-contamination}.
\end{proof}

\EmpiricalQuantilePtwoLemma*
\begin{proof}[Proof of Lemma~\ref{lem:empirical-quantile-error-under-p2}]
For any $s<t$ in $\cI$, the mixture identity, monotonicity of
$F_{P_0,Q_2}$, and the density lower bound imply
\begin{equation}
F_{P_0,P_2}(t)-F_{P_0,P_2}(s)
\ge
(1-\delta_2)r_0(t-s).
\label{eq:specific-calibration-cdf-growth}
\end{equation}
Consequently, the generalized inverse of $F_{P_0,P_2}$ is
$1/((1-\delta_2)r_0)$-Lipschitz whenever the relevant
quantiles lie in $\cI$.

Let
\[
\widehat F_{P_0,\cD_2}(t)
:=
\frac{1}{n_2}\sum_{i=1}^{n_2}
\bbone\{R_{P_0}(\bX_i,Y_i)\le t\}
\]
be the empirical CDF of the oracle calibration scores, and set
\[
\varepsilon_{n_2}(\gamma)
:=
\sqrt{\frac{\log(2/\gamma)}{2n_2}},
\qquad
u_{n_2}:=\frac{k_{n_2}}{n_2}.
\]
The Dvoretzky--Kiefer--Wolfowitz inequality gives, with probability at
least $1-\gamma$,
\[
\sup_t
\left|\widehat F_{P_0,\cD_2}(t)-F_{P_0,P_2}(t)\right|
\le
\varepsilon_{n_2}(\gamma).
\]
On this event,
\[
F_{P_0,P_2}^{-1}
\left(u_{n_2}-\varepsilon_{n_2}(\gamma)\right)
\le
\widehat q(P_0;\cD_2)
\le
F_{P_0,P_2}^{-1}
\left(u_{n_2}+\varepsilon_{n_2}(\gamma)\right).
\]
Since $k_{n_2}=\lceil(n_2+1)(1-\alpha)\rceil$ and
$k_{n_2}\le n_2$ for all sufficiently large $n_2$,
\[
0\le u_{n_2}-(1-\alpha)\le\frac{2}{n_2}.
\]
Because $\cI$ is open and contains the compact interval
$[q_-(\delta_2),q_+(\delta_2)]$, choose $d_{\cI}>0$ such that
\[
[q_-(\delta_2)-d_{\cI},q_+(\delta_2)+d_{\cI}]\subset\cI.
\]
The sandwich in Lemma~\ref{lem:calibration-contamination-quantile} and
\eqref{eq:specific-calibration-cdf-growth} show that the inverse images of
$u_{n_2}\pm\varepsilon_{n_2}(\gamma)$ remain in this enlarged interval once
\[
\varepsilon_{n_2}(\gamma)+2/n_2
<(1-\delta_2)r_0d_{\cI}.
\]
Thus the relevant quantiles lie in $\cI$ for all sufficiently large $n_2$,
and the inverse-Lipschitz bound following from
\eqref{eq:specific-calibration-cdf-growth} proves
\eqref{eq:empirical-quantile-error-under-p2}.

\end{proof}

\subsection{Proof of Theorem~\ref{thm:delta} }\label{ssec:proof_thm_delta}
\JointContaminationTheorem*
For reference, the quantile-error decomposition used in the proof is
\begin{align}
\QuantileErrorDecomposition.
\label{eq:quantile-error-decomposition}
\end{align}
\begin{proof}

Fix $\gamma\in(0,1)$. On the intersection of the events in
Lemmas~\ref{lem:training-contamination-quantile-error},
\ref{lem:clean-training-sampling-quantile-error}, and
\ref{lem:empirical-quantile-error-under-p2}, combining
\eqref{eq:quantile-error-decomposition} with
Lemma~\ref{lem:calibration-contamination-quantile} gives
\begin{align}
&\left|\widehat q(\cD_1;\cD_2)-q(P_0;P_0)\right|
\nonumber\\
&\quad\le
\eta_{n_1}(P_0)
+L\left(\delta_1+\sqrt{\frac{\log n_1}{n_1}}\right)
+
\frac{1}{(1-\delta_2)r_0}
\left(
\sqrt{\frac{\log(2/\gamma)}{2n_2}}
+
\frac{2}{n_2}
+
\delta_2(1-\alpha)
\right).
\label{eq:proof-quantile-direct-bound}
\end{align}
The intersection has probability at least
$1-\rho_{n_1}(P_0)-n_1^{-2}-\gamma$. Multiplying
\eqref{eq:proof-quantile-direct-bound} by two yields, on the same event,
\begin{align*}
&2\left|
\widehat q(\cD_1;\cD_2)-q(P_0;P_0)
\right|
\nonumber\\
&\quad\le
2\eta_{n_1}(P_0)
+2L\left(\delta_1+\sqrt{\frac{\log n_1}{n_1}}\right)
+
\frac{2}{(1-\delta_2)r_0}
\left(
\sqrt{\frac{\log(2/\gamma)}{2n_2}}
+\frac{2}{n_2}
\right)
\nonumber\\
&\qquad+
\frac{2\delta_2(1-\alpha)}
{(1-\delta_2)r_0}.
\end{align*}
This is \eqref{eq:joint-contamination-finite-sample-bound}.

\end{proof}

\paragraph{Asymptotic consequence.}
For completeness, the finite-sample decomposition also yields
\begin{align}
2\left|
\widehat q(\cD_1;\cD_2)-q(P_0;P_0)
\right|
=O_{\Pb}\Bigg(
&\eta_{n_1}(P_0)
+L\left(\delta_1+n_1^{-1/2}\right)
\!+\!\frac{1}{(1\!-\!\delta_2)r_0\sqrt{n_2}}
\!+\!
\frac{\delta_2(1\!-\!\alpha)}
{(1\!-\!\delta_2)r_0}
\Bigg).
\label{eq:joint-contamination-op-bound}
\end{align}
Indeed, use the almost-sure bound
\eqref{eq:empirical-perturbation-replacement-stability} directly instead of
its $1-n_1^{-2}$ envelope. Since
$K/n_1=\delta_1+O_{\Pb}(n_1^{-1/2})$ and
$\rho_{n_1}(P_0)=o(1)$, the two training terms in
\eqref{eq:quantile-error-decomposition} are jointly
$O_{\Pb}(\eta_{n_1}(P_0)+L(\delta_1+n_1^{-1/2}))$. The DKW inequality gives
an $O_{\Pb}(n_2^{-1/2}/((1-\delta_2)r_0))$ calibration-sampling term, and
Lemma~\ref{lem:calibration-contamination-quantile} supplies the remaining
deterministic term. Multiplying by two proves
\eqref{eq:joint-contamination-op-bound}.

\subsection{Proof of Proposition~\ref{prop:worst-case-length-deviation}}
\label{ssec:proof_worst_case_length_deviation}
\WorstCaseCalibrationProposition*
\begin{proof}
Write $F_0:=F_{P_0,P_0}$ and $q_0:=F_0^{-1}(1-\alpha)$, and define
\[
\Delta_{\delta_2}(Q_2)
:=
q\!\left(P_0;(1-\delta_2)P_0+\delta_2Q_2\right)-q_0.
\]
For the proof, introduce the one-sided worst-case deviations
\[
\mathfrak D_\alpha^+(\delta_2)
:=2\sup_{Q_2}\Delta_{\delta_2}(Q_2),
\qquad
\mathfrak D_\alpha^-(\delta_2)
:=2\sup_{Q_2}\{-\Delta_{\delta_2}(Q_2)\}.
\]
Then $\mathfrak D_\alpha(\delta_2)
=2\sup_{Q_2}|\Delta_{\delta_2}(Q_2)|$. Fix any
$\bbx_0\in\cX$ and $M\ge 0$, and let $Q_M$ be the point mass at
$(\bbx_0,\mu(\bbx_0;P_0)+M)$, whose oracle score equals $M$ almost surely.

First suppose that $0\le\delta_2<\alpha$. For every $Q_2$, the quantile
sandwich \eqref{eq:calibration-quantile-sandwich} gives
\[
q_-(\delta_2)
\le q(P_0;P_2)\le
q_+(\delta_2).
\]
Consequently,
\[
\mathfrak D_\alpha^+(\delta_2)
\le 2\bigl(q_+(\delta_2)-q_0\bigr),
\qquad
\mathfrak D_\alpha^-(\delta_2)
\le 2\bigl(q_0-q_-(\delta_2)\bigr).
\]
For any $M>q_+(\delta_2)$, the contamination-score CDF under $Q_M$ is
zero below $M$. Hence, below $M$, the mixture CDF equals
$(1-\delta_2)F_0$, and its $(1-\alpha)$-quantile is
$q_+(\delta_2)$. Thus, $Q_M$ attains the upper bound for
$\mathfrak D_\alpha^+(\delta_2)$. Under $Q_0$, all contaminated scores
equal zero, so for $t\ge0$ the mixture CDF is
\[
\delta_2+(1-\delta_2)F_0(t).
\]
Since $1-\alpha>\delta_2$, its $(1-\alpha)$-quantile is
$q_-(\delta_2)$, and $Q_0$ attains the upper bound for
$\mathfrak D_\alpha^-(\delta_2)$. Therefore,
\begin{equation}
\mathfrak D_\alpha^+(\delta_2)
=2\bigl(q_+(\delta_2)-q_0\bigr),
\qquad
\mathfrak D_\alpha^-(\delta_2)
=2\bigl(q_0-q_-(\delta_2)\bigr).
\label{eq:worst-case-one-sided-deviations}
\end{equation}
Finally,
\[
\sup_{Q_2}|a(Q_2)|
=\max\left\{\sup_{Q_2}a(Q_2),\sup_{Q_2}[-a(Q_2)]\right\},
\]
applied to the contaminated-minus-clean quantile difference, gives
\eqref{eq:worst-case-length-deviation}.

Now suppose that $\delta_2\ge\alpha$ and take $M$ arbitrarily large. For
every $t<M$,
\[
F_{P_0,P_2}(t)
=(1-\delta_2)F_0(t)
<1-\alpha.
\]
When $\delta_2=\alpha$, strictness follows from the unbounded-upper-support
assumption; when $\delta_2>\alpha$, it follows from
$1-\delta_2<1-\alpha$. For all sufficiently large $M$,
$F_{P_0,P_2}(M)\ge1-\alpha$, so the $(1-\alpha)$-quantile equals $M$.
Letting $M\to\infty$ proves that both
$\mathfrak D_\alpha^+(\delta_2)$ and $\mathfrak D_\alpha(\delta_2)$ are
infinite.
\end{proof}

\subsection{Derivation of the Mixture-Based Correction}
\label{ssec:proof_mixture_nominal_level_correction}
\begin{corollary}[Mixture-based nominal-level correction]
\label{cor:nominal-level-correction}
Define $\alpha'_{\mathrm{tight}}:=\alpha(1-\delta_2)$. If
\[
\Coverage_{P_2}(\alpha'_{\mathrm{tight}})
\ge 1-\alpha'_{\mathrm{tight}},
\]
then
\[
\Coverage_{P_0}(\alpha'_{\mathrm{tight}})
\ge 1-\alpha.
\]
\end{corollary}
\begin{proof}
Conditional on $\cD_1$ and $\cD_2$, the mixture identity for the new test
observation holds for the coverage event. Averaging over the two datasets
gives
\[
\Coverage_{P_2}(\alpha'_{\mathrm{tight}})
=
(1-\delta_2)\Coverage_{P_0}(\alpha'_{\mathrm{tight}})
+\delta_2\Coverage_{Q_2}(\alpha'_{\mathrm{tight}}).
\]
Since $\Coverage_{Q_2}(\alpha'_{\mathrm{tight}})\le1$, the assumed
$P_2$-coverage guarantee implies
\[
\Coverage_{P_0}(\alpha'_{\mathrm{tight}})
\ge
\frac{1-\alpha'_{\mathrm{tight}}-\delta_2}{1-\delta_2}
=
1-\alpha.
\]
\end{proof}

\subsection{Proof of Proposition~\ref{prop:nominal-correction-sharpness}}
\label{ssec:proof_nominal_correction_sharpness}
\NominalCorrectionSharpnessProposition*
\begin{proof}
Fix $\bar\alpha\in(\alpha(1-\delta_2),1)$. Let the feature space
contain only the origin, so that $\bX=0$ almost surely, and define the
clean distribution by
\[
Y\mid \bX=0
\sim
\cN(0,1)
\qquad\text{under }P_0.
\]
Thus $\E_{P_0}[Y\mid\bX=0]=0$, and the clean noise is independent of $\bX$, has a
symmetric density, and is nonincreasing on $[0,\infty)$. Let the
contamination distribution satisfy
\[
Q_2(\bX=0,Y=0)=1.
\]

Fit a no-intercept linear regression estimator
$\mu(\bbx;\cD_1)=\hat\beta^{\top}\bbx$ on the training sample. Since every
feature vector is zero, the fitted predictor satisfies
$\mu(0;\cD_1)=0$ for every training sample. Hence the calibration and test
scores are
\[
R=|Y-\mu(\bX;\cD_1)|=|Y|.
\]
Let $G_0$ and $G_2$ denote the CDFs of $R$ under $P_0$ and
$P_2=(1-\delta_2)P_0+\delta_2Q_2$, respectively. For every $t\ge0$,
\begin{equation}
G_2(t)
=
\delta_2+(1-\delta_2)G_0(t).
\label{eq:sharpness-score-mixture}
\end{equation}

Let
\[
q_{\bar\alpha}
:=
G_2^{-1}(1-\bar\alpha).
\]
Because $G_0$ is continuous and strictly increasing on $[0,\infty)$,
\eqref{eq:sharpness-score-mixture} gives
\begin{equation*}
G_0(q_{\bar\alpha})
=
\max\left\{
0,
1-\frac{\bar\alpha}{1-\delta_2}
\right\}
<
1-\alpha,
\end{equation*}
where the strict inequality follows from
$\bar\alpha>\alpha(1-\delta_2)$.

Now let the calibration size be $n_2=m$, set
\[
k_m
:=
\left\lceil(m+1)(1-\bar\alpha)\right\rceil,
\]
and let $\hat q_m$ be the $k_m$-th order statistic of the calibration
scores. For every $\varepsilon>0$, the definition of
$q_{\bar\alpha}$ and the strict increase of $G_0$ give
\[
G_2(q_{\bar\alpha}-\varepsilon)
<
1-\bar\alpha
<
G_2(q_{\bar\alpha}+\varepsilon).
\]
The same inequalities hold when $q_{\bar\alpha}=0$, despite the atom of
$G_2$ at zero. Applying the law of large numbers to the proportions of
calibration scores below the two endpoints, and using
$k_m/m\to1-\bar\alpha$, yields
\begin{equation}
\hat q_m
\longrightarrow
q_{\bar\alpha}
\quad\text{in probability}.
\label{eq:sharpness-empirical-quantile-limit}
\end{equation}

The resulting interval is
$C_{\bar\alpha}=[-\hat q_m,\hat q_m]$. Since the calibration sample is
i.i.d. from $P_2$, the usual finite-sample split conformal guarantee
gives, for all sufficiently large $m$,
\begin{equation}
\Coverage_{P_2}(\bar\alpha)
\ge
1-\bar\alpha.
\label{eq:sharpness-p2-coverage}
\end{equation}
Moreover, $Y=0$ almost surely under $Q_2$, and hence
$\Coverage_{Q_2}(\bar\alpha)=1$ for every $m$.
On the other hand, independence of the calibration and clean test
samples gives
\[
\Coverage_{P_0}(\bar\alpha)
=
\E\left[G_0(\hat q_m)\right].
\]
By \eqref{eq:sharpness-empirical-quantile-limit} and continuity of
$G_0$, we have $G_0(\hat q_m)\to G_0(q_{\bar\alpha})$ in probability.
Because this sequence is bounded, its expectations converge as well.
Therefore,
\begin{equation*}
\Coverage_{P_0}(\bar\alpha)
\longrightarrow
G_0(q_{\bar\alpha})
<
1-\alpha.
\end{equation*}
Therefore, the clean coverage is below $1-\alpha$ for all sufficiently
large $m$, while \eqref{eq:sharpness-p2-coverage} continues to hold.
This proves Proposition~\ref{prop:nominal-correction-sharpness}.

\end{proof}

\subsection{Extensions to CQR and Adaptive Residual Scores}
\label{ssec:alternative-score-extensions}

This subsection develops the alternative-score extensions outlined in
Remark~\ref{rem:alternative-score-extensions}. We introduce the
score-specific conditions and establish finite-sample bounds for score
quantiles and interval lengths, uniformly over the training and calibration
contamination distributions. Absolute residuals are retained as a baseline
to make the connection with Theorem~\ref{thm:delta} explicit.

\subsubsection{Score Definitions and Stability Conditions}

For $A\in\{\cD_1,\cD_0,P_0\}$, define
\begin{align*}
R^{\mathrm{Absolute}}(A;\bbx,y)
&:=|y-\mu(\bbx;A)|,\\
R^{\mathrm{CQR}}(A;\bbx,y)
&:=\max\{\ell(\bbx;A)-y,\,y-u(\bbx;A)\},\\
R^{\mathrm{Adaptive}}(A;\bbx,y)
&:=\frac{|y-\mu(\bbx;A)|}{\sigma(\bbx;A)}.
\end{align*}
Here $\ell(\cdot;A)$ and $u(\cdot;A)$ are the lower and upper CQR
regressors, and $\sigma(\cdot;A)>0$ is the adaptive scale regressor. The
choices $A=\cD_1$ and $A=\cD_0$ denote fits on the coupled contaminated and
clean samples, whereas $A=P_0$ denotes the clean population target.

For
$S\in\{\mathrm{Absolute},\mathrm{CQR},\mathrm{Adaptive}\}$ and
$P\in\{P_0,P_2,Q_2\}$, define
\[
F_{P_0,P}^S(t)
:=
\Pb_{(\bX,Y)\sim P}\{R^S(P_0;\bX,Y)\le t\},
\qquad
q_S(P_0;P):=(F_{P_0,P}^S)^{-1}(1-\alpha).
\]
For $(\bX_i,Y_i)\in\cD_2$, write
$R_i^S(A):=R^S(A;\bX_i,Y_i)$. For all sufficiently large $n_2$ such
that $k_{n_2}=\lceil(n_2+1)(1-\alpha)\rceil\le n_2$, let
\[
\widehat q_S(A;\cD_2):=R_{(k_{n_2})}^S(A).
\]

The corresponding CQR and adaptive prediction intervals are
\begin{align*}
C_\alpha^{\mathrm{CQR}}(\bbx;A,\cD_2)
&:=\left[
\ell(\bbx;A)-\widehat q_{\mathrm{CQR}}(A;\cD_2),\,
u(\bbx;A)+\widehat q_{\mathrm{CQR}}(A;\cD_2)
\right],\\
C_\alpha^{\mathrm{Adaptive}}(\bbx;A,\cD_2)
&:=\left[
\mu(\bbx;A)-\sigma(\bbx;A)\widehat q_{\mathrm{Adaptive}}(A;\cD_2),\,
\right.\\[-0.4em]
&\hspace{3em}\left.
\mu(\bbx;A)+\sigma(\bbx;A)\widehat q_{\mathrm{Adaptive}}(A;\cD_2)
\right].
\end{align*}
An interval whose lower endpoint exceeds its upper endpoint is interpreted
as the empty set. Thus the CQR interval has length
$[u(\bbx;A)-\ell(\bbx;A)+2\widehat q_{\mathrm{CQR}}(A;\cD_2)]_+$,
where $[z]_+:=\max\{z,0\}$, whereas the adaptive interval has length
$2\sigma(\bbx;A)\widehat q_{\mathrm{Adaptive}}(A;\cD_2)$.
The fitted intervals use $A=\cD_1$. Their clean oracle counterparts use
$A=P_0$ and replace each empirical quantile $\widehat q_S(A;\cD_2)$
by $q_S(P_0;P_0)$.

For the absolute score, this score-indexed notation agrees with the
notation introduced in the preliminaries:
\[
R^{\mathrm{Absolute}}(A;\bbx,y)=R_A(\bbx,y),\qquad
F_{P_0,P}^{\mathrm{Absolute}}=F_{P_0,P},
\]
\[
q_{\mathrm{Absolute}}(P_0;P)=q(P_0;P),\qquad
\widehat q_{\mathrm{Absolute}}(A;\cD_2)=\widehat q(A;\cD_2).
\]
The adaptive score intentionally uses the same location learner
$\mu(\cdot;A)$ as the absolute score and differs through its scale
normalization by $\sigma(\cdot;A)$.

Define the clean-sampling events
\begin{align*}
\cE_{0,n_1}^{\mathrm{Absolute}}
&:=\left\{
\|\mu(\cdot;\cD_0)-\mu(\cdot;P_0)\|_\infty
\le\eta_{n_1}(P_0)
\right\},\\
\cE_{0,n_1}^{\mathrm{CQR}}
&:=\left\{
\max_{g\in\{\ell,u\}}
\|g(\cdot;\cD_0)-g(\cdot;P_0)\|_\infty
\le\eta_{n_1}^{\mathrm{CQR}}(P_0)
\right\},\\
\cE_{0,n_1}^{\mathrm{Adaptive}}
&:=\bigcap_{g\in\{\mu,\sigma\}}
\left\{
\|g(\cdot;\cD_0)-g(\cdot;P_0)\|_\infty
\le\eta_{g,n_1}^{\mathrm{Adaptive}}(P_0)
\right\}.
\end{align*}
Assume
\begin{equation*}
\Pb(\cE_{0,n_1}^S)\ge1-\rho_{n_1}^S(P_0),
\qquad
S\in\{\mathrm{Absolute},\mathrm{CQR},\mathrm{Adaptive}\}.
\end{equation*}
For the absolute score, use the conventions
$\eta_{n_1}^{\mathrm{Absolute}}(P_0):=\eta_{n_1}(P_0)$ and
$\rho_{n_1}^{\mathrm{Absolute}}(P_0):=\rho_{n_1}(P_0)$.
Because the adaptive score uses the same location learner, also set
$\eta_{\mu,n_1}^{\mathrm{Adaptive}}(P_0):=\eta_{n_1}(P_0)$ and
$L_\mu^{\mathrm{Adaptive}}:=L$.
For the absolute score, suppose that $\mu$ is $L$-uniformly replacement
stable as in Definition~\ref{def:replacement-stability}. For the CQR and
adaptive scores, assume componentwise replacement stability: whenever
$\cD,\cD'$ differ in at most $k$ observations,
\begin{align}
\max_{g\in\{\ell,u\}}
\|g(\cdot;\cD)-g(\cdot;\cD')\|_\infty
&\le L_{\mathrm{CQR}}k/n_1,
\nonumber\\
\|g(\cdot;\cD)-g(\cdot;\cD')\|_\infty
&\le L_g^{\mathrm{Adaptive}}k/n_1,
\qquad g\in\{\mu,\sigma\}.
\label{eq:alternative-componentwise-replacement-stability}
\end{align}
Set $h_{n_1}:=\delta_1+\sqrt{\log(n_1)/n_1}$ and define
\begin{align*}
\varepsilon_{\mathrm{train},n_1}^{\mathrm{Absolute}}
&:=\eta_{n_1}(P_0)+Lh_{n_1},\\
\varepsilon_{\mathrm{train},n_1}^{\mathrm{CQR}}
&:=\eta_{n_1}^{\mathrm{CQR}}(P_0)+L_{\mathrm{CQR}}h_{n_1},\\
\varepsilon_{\mu,n_1}^{\mathrm{Adaptive}}
&:=\eta_{n_1}(P_0)+Lh_{n_1},\\
\varepsilon_{\sigma,n_1}^{\mathrm{Adaptive}}
&:=\eta_{\sigma,n_1}^{\mathrm{Adaptive}}(P_0)
+L_{\sigma}^{\mathrm{Adaptive}}h_{n_1}.
\end{align*}
Define the contaminated-training events
\begin{align*}
\cE_{n_1}^{\mathrm{Absolute}}
&:=\left\{
\|\mu(\cdot;\cD_1)-\mu(\cdot;P_0)\|_\infty
\le\varepsilon_{\mathrm{train},n_1}^{\mathrm{Absolute}}
\right\},\\
\cE_{n_1}^{\mathrm{CQR}}
&:=\left\{
\max_{g\in\{\ell,u\}}
\|g(\cdot;\cD_1)-g(\cdot;P_0)\|_\infty
\le\varepsilon_{\mathrm{train},n_1}^{\mathrm{CQR}}
\right\},\\
\cE_{n_1}^{\mathrm{Adaptive}}
&:=\bigcap_{g\in\{\mu,\sigma\}}
\left\{
\|g(\cdot;\cD_1)-g(\cdot;P_0)\|_\infty
\le\varepsilon_{g,n_1}^{\mathrm{Adaptive}}
\right\}.
\end{align*}
Under the Huber coupling used in
Proposition~\ref{prop:empirical-perturbation-replacement-stability}, the
coupled samples $\cD_1$ and $\cD_0$ differ in at most
$K\sim\operatorname{Binomial}(n_1,\delta_1)$ observations. On
$\{K/n_1\le h_{n_1}\}$, the componentwise bounds in
\eqref{eq:alternative-componentwise-replacement-stability}, combined with
the corresponding clean-sampling event $\cE_{0,n_1}^S$, imply
$\cE_{n_1}^S$. Hoeffding's inequality and a union bound therefore give,
uniformly over $Q_1$,
\begin{equation}
\Pb(\cE_{n_1}^S)
\ge1-\rho_{n_1}^S(P_0)-n_1^{-2}.
\label{eq:alternative-contaminated-training-events}
\end{equation}
Assume that all clean-sampling radii and failure probabilities above are
$o(1)$.

For the adaptive score, assume that, on
$\cE_{n_1}^{\mathrm{Adaptive}}$, the fitted and clean population scales
obey the global lower bound
\begin{equation*}
\inf_{\bbx\in\cX}
\min\{\sigma(\bbx;\cD_1),\sigma(\bbx;P_0)\}
\ge\underline\sigma>0.
\end{equation*}
This condition can, for example, be enforced by flooring the output of the
scale learner at $\underline\sigma$. Because it holds on all of $\cX$, it
does not restrict the support of $Q_2$.
For the uniform interval-length conclusion only, additionally define and
assume finite the global clean-scale upper bound
\begin{equation*}
\overline\sigma(P_0)
:=\sup_{\bbx\in\cX}\sigma(\bbx;P_0)<\infty,
\qquad
\overline\sigma:=\overline\sigma(P_0).
\end{equation*}

No bound on the clean or contaminating residuals is imposed. The adaptive
argument below controls the relevant order statistic directly rather than
the largest pointwise score perturbation.

For the local-deviation bounds below, suppose
$\alpha\in(0,1/2)$ and $0\le\delta_2<\alpha$, and set
\[
p_-:=\frac{1-\alpha-\delta_2}{1-\delta_2},
\qquad
p_+:=\frac{1-\alpha}{1-\delta_2}.
\]
For each score $S$, assume that $F_{P_0,P_0}^S$ is absolutely continuous
on an open interval containing
\[
\left[
(F_{P_0,P_0}^S)^{-1}(p_-),
(F_{P_0,P_0}^S)^{-1}(p_+)
\right]
\]
and that its density satisfies $f_{P_0,P_0}^S(t)\ge r_0^S>0$ throughout
this interval, with the convention
$r_0^{\mathrm{Absolute}}:=r_0$.

\subsubsection{Score-Quantile Deviation Bounds}

For the adaptive score, define
\begin{align*}
\kappa_{\sigma,n_1}^{\mathrm{Adaptive}}
&:=\frac{\varepsilon_{\sigma,n_1}^{\mathrm{Adaptive}}}
{\underline\sigma},
&
\beta_{\mu,n_1}^{\mathrm{Adaptive}}
&:=\frac{\varepsilon_{\mu,n_1}^{\mathrm{Adaptive}}}
{\underline\sigma},
\nonumber\\
\varepsilon_{\mathrm{train},n_1}^{\mathrm{Adaptive}}
&:=\beta_{\mu,n_1}^{\mathrm{Adaptive}}
+\kappa_{\sigma,n_1}^{\mathrm{Adaptive}}
q_{\mathrm{Adaptive}}(P_0;P_0).
\end{align*}
For the proofs below, define the score-specific calibration errors by
\begin{equation*}
\varepsilon_{\mathrm{cal},n_2}^S(\gamma)
:=\frac{\delta_2(1-\alpha)
+\sqrt{\log(2/\gamma)/(2n_2)}+2/n_2}
{(1-\delta_2)r_0^S},
\qquad
S\in\{\mathrm{Absolute},\mathrm{CQR},\mathrm{Adaptive}\}.
\end{equation*}

Define the total score-quantile errors by
\begin{equation}
\varepsilon_{q,n_1,n_2}^S(\gamma)
:=
\begin{cases}
\varepsilon_{\mathrm{train},n_1}^S
+\varepsilon_{\mathrm{cal},n_2}^S(\gamma),
&S\in\{\mathrm{Absolute},\mathrm{CQR}\},\\[0.35em]
\varepsilon_{\mathrm{train},n_1}^{\mathrm{Adaptive}}
+\bigl(1+\kappa_{\sigma,n_1}^{\mathrm{Adaptive}}\bigr)
\varepsilon_{\mathrm{cal},n_2}^{\mathrm{Adaptive}}(\gamma),
&S=\mathrm{Adaptive}.
\end{cases}
\label{eq:alternative-total-score-quantile-error}
\end{equation}
\begin{lemma}[Adaptive empirical-quantile perturbation]
\label{lem:adaptive-empirical-quantile-perturbation}
On $\cE_{n_1}^{\mathrm{Adaptive}}$, for every calibration sample
$\cD_2$,
\begin{equation}
\left|
\widehat q_{\mathrm{Adaptive}}(\cD_1;\cD_2)
-\widehat q_{\mathrm{Adaptive}}(P_0;\cD_2)
\right|
\le
\beta_{\mu,n_1}^{\mathrm{Adaptive}}
+\kappa_{\sigma,n_1}^{\mathrm{Adaptive}}
\widehat q_{\mathrm{Adaptive}}(P_0;\cD_2).
\label{eq:adaptive-empirical-quantile-perturbation}
\end{equation}
In particular, the bound depends on the relevant empirical quantile rather
than on the largest calibration residual.
\end{lemma}

\begin{proof}
Fix a calibration observation $(\bbx,y)$ and write
\begin{align*}
r_{\mathrm{fit}}&:=|y-\mu(\bbx;\cD_1)|,
&r_{\mathrm{clean}}&:=|y-\mu(\bbx;P_0)|,\\
\sigma_{\mathrm{fit}}&:=\sigma(\bbx;\cD_1),
&\sigma_{\mathrm{clean}}&:=\sigma(\bbx;P_0),\\
\varepsilon_\mu&:=\varepsilon_{\mu,n_1}^{\mathrm{Adaptive}},
&\varepsilon_\sigma&:=\varepsilon_{\sigma,n_1}^{\mathrm{Adaptive}}.
\end{align*}
On $\cE_{n_1}^{\mathrm{Adaptive}}$,
\[
|r_{\mathrm{fit}}-r_{\mathrm{clean}}|\le\varepsilon_\mu,
\quad
|\sigma_{\mathrm{fit}}-\sigma_{\mathrm{clean}}|
\le\varepsilon_\sigma,
\quad
\sigma_{\mathrm{fit}},\sigma_{\mathrm{clean}}\ge\underline\sigma.
\]
Set $\kappa:=\varepsilon_\sigma/\underline\sigma$ and
$\beta:=\varepsilon_\mu/\underline\sigma$. Since
$\sigma_{\mathrm{clean}}\le\sigma_{\mathrm{fit}}+\varepsilon_\sigma$,
\begin{align*}
R^{\mathrm{Adaptive}}(\cD_1;\bbx,y)
&=\frac{r_{\mathrm{fit}}}{\sigma_{\mathrm{fit}}}\\
&\le
\frac{r_{\mathrm{clean}}+\varepsilon_\mu}
{\sigma_{\mathrm{fit}}}\\
&\le
(1+\kappa)
R^{\mathrm{Adaptive}}(P_0;\bbx,y)+\beta.
\end{align*}
Interchanging the fitted and clean quantities gives the reverse comparison
\[
R^{\mathrm{Adaptive}}(P_0;\bbx,y)
\le
(1+\kappa)
R^{\mathrm{Adaptive}}(\cD_1;\bbx,y)+\beta.
\]
Both comparisons hold for every $(\bbx,y)$, without bounding the residual.
Monotonicity of order statistics therefore yields, with
$\widehat q_{\mathrm{fit}}
:=\widehat q_{\mathrm{Adaptive}}(\cD_1;\cD_2)$ and
$\widehat q_{\mathrm{clean}}
:=\widehat q_{\mathrm{Adaptive}}(P_0;\cD_2)$,
\[
\widehat q_{\mathrm{fit}}
\le(1+\kappa)\widehat q_{\mathrm{clean}}+\beta,
\qquad
\widehat q_{\mathrm{clean}}
\le(1+\kappa)\widehat q_{\mathrm{fit}}+\beta.
\]
If $\widehat q_{\mathrm{fit}}\ge\widehat q_{\mathrm{clean}}$, the first
inequality gives a difference of at most
$\kappa\widehat q_{\mathrm{clean}}+\beta$. Otherwise, the second gives
\[
\widehat q_{\mathrm{clean}}-\widehat q_{\mathrm{fit}}
\le
\frac{\kappa\widehat q_{\mathrm{clean}}+\beta}{1+\kappa}
\le
\kappa\widehat q_{\mathrm{clean}}+\beta.
\]
This proves \eqref{eq:adaptive-empirical-quantile-perturbation}.
\end{proof}

\begin{proposition}[Unified score-quantile deviation bounds]
\label{prop:alternative-score-quantile-bounds}
For every $\gamma\in(0,1)$ and all sufficiently large $n_2$, each
$S\in\{\mathrm{Absolute},\mathrm{CQR},\mathrm{Adaptive}\}$ satisfies,
with probability at least
$1-\rho_{n_1}^S(P_0)-n_1^{-2}-\gamma$,
\begin{align}
&\left|
\widehat q_S(\cD_1;\cD_2)-q_S(P_0;P_0)
\right|
\nonumber\\[-0.25em]
&\qquad\le
\varepsilon_{q,n_1,n_2}^S(\gamma).
\label{eq:alternative-score-quantile-bound}
\end{align}
The probability statement is uniform over the contamination distributions
$Q_1,Q_2$. Consequently, if $\rho_{n_1}^S(P_0)=o(1)$, then, for
$S\in\{\mathrm{Absolute},\mathrm{CQR}\}$,
\begin{equation*}
\left|
\widehat q_S(\cD_1;\cD_2)-q_S(P_0;P_0)
\right|
=
O_{\Pb}\left(
\varepsilon_{\mathrm{train},n_1}^S
+
\frac{n_2^{-1/2}+\delta_2(1-\alpha)}
{(1-\delta_2)r_0^S}
\right).
\end{equation*}
For the adaptive score, the corresponding statement is
\begin{align*}
&\left|
\widehat q_{\mathrm{Adaptive}}(\cD_1;\cD_2)
-q_{\mathrm{Adaptive}}(P_0;P_0)
\right|
\nonumber\\
&\quad=O_{\Pb}\left(
\beta_{\mu,n_1}^{\mathrm{Adaptive}}
+\kappa_{\sigma,n_1}^{\mathrm{Adaptive}}
q_{\mathrm{Adaptive}}(P_0;P_0)
\right.
\nonumber\\[-0.25em]
&\hspace{8em}\left.
+\bigl(1+\kappa_{\sigma,n_1}^{\mathrm{Adaptive}}\bigr)
\frac{n_2^{-1/2}+\delta_2(1-\alpha)}
{(1-\delta_2)r_0^{\mathrm{Adaptive}}}
\right).
\end{align*}
\end{proposition}

\begin{proof}
We first bound the training-induced score perturbation. For the absolute
score, the reverse triangle inequality gives, on
$\cE_{n_1}^{\mathrm{Absolute}}$,
\begin{align*}
\left|
R_i^{\mathrm{Absolute}}(\cD_1)
-R_i^{\mathrm{Absolute}}(P_0)
\right|
&\le
|\mu(\bX_i;\cD_1)-\mu(\bX_i;P_0)|\\
&\le \varepsilon_{\mathrm{train},n_1}^{\mathrm{Absolute}}.
\end{align*}
For CQR, the triangle inequality first gives, on
$\cE_{n_1}^{\mathrm{CQR}}$,
\[
\max_{g\in\{\ell,u\}}
\|g(\cdot;\cD_1)-g(\cdot;P_0)\|_\infty
\le \varepsilon_{\mathrm{train},n_1}^{\mathrm{CQR}}.
\]
The maximum of two coordinates is one-Lipschitz under the maximum norm, so
\begin{align*}
&\left|
R_i^{\mathrm{CQR}}(\cD_1)-R_i^{\mathrm{CQR}}(P_0)
\right|\\
&\quad\le
\max\left\{
|\ell(\bX_i;\cD_1)-\ell(\bX_i;P_0)|,
|u(\bX_i;\cD_1)-u(\bX_i;P_0)|
\right\}\\
&\quad\le \varepsilon_{\mathrm{train},n_1}^{\mathrm{CQR}}.
\end{align*}
Thus, for $S\in\{\mathrm{Absolute},\mathrm{CQR}\}$,
Lemma~\ref{lem:order-statistic-perturbation} implies
\begin{equation}
\left|
\widehat q_S(\cD_1;\cD_2)-\widehat q_S(P_0;\cD_2)
\right|
\le \varepsilon_{\mathrm{train},n_1}^S
\label{eq:alternative-score-training-part}
\end{equation}
on $\cE_{n_1}^S$.

For the calibration contribution, fixing the score at its clean population
value gives
\[
F_{P_0,P_2}^S(t)
=
(1-\delta_2)F_{P_0,P_0}^S(t)
+\delta_2F_{P_0,Q_2}^S(t).
\]
The quantile-sandwich argument of
Lemma~\ref{lem:calibration-contamination-quantile} yields
\begin{equation}
\left|q_S(P_0;P_2)-q_S(P_0;P_0)\right|
\le
\frac{\delta_2(1-\alpha)}{(1-\delta_2)r_0^S}.
\label{eq:alternative-score-calibration-population-part}
\end{equation}
The Dvoretzky--Kiefer--Wolfowitz inequality and the conformal rank
correction further give, with probability at least $1-\gamma$,
\begin{equation}
\left|
\widehat q_S(P_0;\cD_2)-q_S(P_0;P_2)
\right|
\le
\frac{
\sqrt{\log(2/\gamma)/(2n_2)}+2/n_2
}{(1-\delta_2)r_0^S}.
\label{eq:alternative-score-calibration-empirical-part}
\end{equation}
For $S\in\{\mathrm{Absolute},\mathrm{CQR}\}$, combining
\eqref{eq:alternative-score-training-part} with the two calibration bounds
by the triangle inequality gives the first branch of
\eqref{eq:alternative-total-score-quantile-error}.

For the adaptive score, abbreviate
\[
\widehat q_0
:=\widehat q_{\mathrm{Adaptive}}(P_0;\cD_2),
\qquad
q_0:=q_{\mathrm{Adaptive}}(P_0;P_0),
\]
and write $\kappa_\sigma
:=\kappa_{\sigma,n_1}^{\mathrm{Adaptive}}$ and
$\beta_\mu:=\beta_{\mu,n_1}^{\mathrm{Adaptive}}$.
On the calibration event given by
\eqref{eq:alternative-score-calibration-population-part} and
\eqref{eq:alternative-score-calibration-empirical-part},
\[
|\widehat q_0-q_0|
\le\varepsilon_{\mathrm{cal},n_2}^{\mathrm{Adaptive}}(\gamma),
\qquad
\widehat q_0
\le q_0+\varepsilon_{\mathrm{cal},n_2}^{\mathrm{Adaptive}}(\gamma).
\]
Lemma~\ref{lem:adaptive-empirical-quantile-perturbation} and the triangle
inequality now give
\begin{align*}
&\left|
\widehat q_{\mathrm{Adaptive}}(\cD_1;\cD_2)-q_0
\right|\\
&\quad\le
\beta_\mu+\kappa_\sigma\widehat q_0
+\varepsilon_{\mathrm{cal},n_2}^{\mathrm{Adaptive}}(\gamma)\\
&\quad\le
\beta_\mu+\kappa_\sigma q_0
+(1+\kappa_\sigma)
\varepsilon_{\mathrm{cal},n_2}^{\mathrm{Adaptive}}(\gamma),
\end{align*}
which is the adaptive branch of
\eqref{eq:alternative-total-score-quantile-error}. The training event is
uniform over $Q_1$ by \eqref{eq:alternative-contaminated-training-events},
while the mixture sandwich and DKW inequality are uniform over $Q_2$.
Taking a union bound proves \eqref{eq:alternative-score-quantile-bound}.
The $O_{\Pb}$ statements follow because
$n_2^{-1}=O(n_2^{-1/2})$.
\end{proof}

\subsubsection{Interval-Length Deviation Bounds}

The score-quantile bounds yield the following deviations of the fitted
interval lengths from their clean oracle counterparts.

\begin{proposition}[Alternative-score interval-length deviation bounds]
\label{prop:alternative-score-length-bounds}
Under the score-specific conditions of this subsection, for every
$\gamma\in(0,1)$ and all sufficiently large $n_1,n_2$, the following bounds
hold with probability at least
$1-\rho_{n_1}^S(P_0)-n_1^{-2}-\gamma$ for the corresponding score $S$:
\begin{align*}
\sup_{\bbx\in\cX}
\left|2\widehat q_{\mathrm{Absolute}}(\cD_1;\cD_2)
-2q_{\mathrm{Absolute}}(P_0;P_0)\right|
&\le 2\varepsilon_{q,n_1,n_2}^{\mathrm{Absolute}}(\gamma),\\
\sup_{\bbx\in\cX}
\Big|[u(\bbx;\cD_1)-\ell(\bbx;\cD_1)
+2\widehat q_{\mathrm{CQR}}(\cD_1;\cD_2)]_+
&\nonumber\\[-0.4em]
{}-[u(\bbx;P_0)-\ell(\bbx;P_0)
+2q_{\mathrm{CQR}}(P_0;P_0)]_+\Big|
&\le
4\varepsilon_{\mathrm{train},n_1}^{\mathrm{CQR}}
+2\varepsilon_{\mathrm{cal},n_2}^{\mathrm{CQR}}(\gamma),\\
\sup_{\bbx\in\cX}
\Bigl|2\sigma(\bbx;\cD_1)
\widehat q_{\mathrm{Adaptive}}(\cD_1;\cD_2)
&\nonumber\\[-0.4em]
{}-2\sigma(\bbx;P_0)q_{\mathrm{Adaptive}}(P_0;P_0)\Bigr|
&\le
2\bigl(\overline\sigma
+\varepsilon_{\sigma,n_1}^{\mathrm{Adaptive}}\bigr)
\varepsilon_{q,n_1,n_2}^{\mathrm{Adaptive}}(\gamma)
\nonumber\\[-0.4em]
&\quad
+2q_{\mathrm{Adaptive}}(P_0;P_0)
\varepsilon_{\sigma,n_1}^{\mathrm{Adaptive}}.
\end{align*}
The absolute-score bound recovers Theorem~\ref{thm:delta}. The CQR bound
also accounts for perturbations of the fitted lower and upper regressors,
while the adaptive bound includes both quantile and scale perturbations.
\end{proposition}

\begin{proof}[Proof of Proposition~\ref{prop:alternative-score-length-bounds}]
Fix a score $S$ and intersect its training event $\cE_{n_1}^S$ with the
calibration event in
Proposition~\ref{prop:alternative-score-quantile-bounds}. On this event, write
\[
E_q^S:=\varepsilon_{q,n_1,n_2}^S(\gamma),
\qquad
\left|\widehat q_S(\cD_1;\cD_2)-q_S(P_0;P_0)\right|\le E_q^S.
\]
For the absolute score,
\begin{align*}
&\sup_{\bbx\in\cX}
\left|2\widehat q_{\mathrm{Absolute}}(\cD_1;\cD_2)
-2q_{\mathrm{Absolute}}(P_0;P_0)\right|\\
&\qquad
\le2E_q^{\mathrm{Absolute}}.
\end{align*}

For CQR, set
$b_A(\bbx):=u(\bbx;A)-\ell(\bbx;A)$. The definition of
$\cE_{n_1}^{\mathrm{CQR}}$ and the triangle inequality give
\[
\sup_{\bbx\in\cX}|b_{\cD_1}(\bbx)-b_{P_0}(\bbx)|
\le2\varepsilon_{\mathrm{train},n_1}^{\mathrm{CQR}}.
\]
Because $z\mapsto[z]_+$ is one-Lipschitz,
\begin{align*}
&\sup_{\bbx\in\cX}
\Big|[
u(\bbx;\cD_1)-\ell(\bbx;\cD_1)
+2\widehat q_{\mathrm{CQR}}(\cD_1;\cD_2)
]_+\\
&\hspace{8em}
-[
u(\bbx;P_0)-\ell(\bbx;P_0)
+2q_{\mathrm{CQR}}(P_0;P_0)
]_+\Big|\\
&\le
2\varepsilon_{\mathrm{train},n_1}^{\mathrm{CQR}}
+2E_q^{\mathrm{CQR}}\\
&=
4\varepsilon_{\mathrm{train},n_1}^{\mathrm{CQR}}
+2\varepsilon_{\mathrm{cal},n_2}^{\mathrm{CQR}}(\gamma).
\end{align*}

For the adaptive score, abbreviate
$\varepsilon_\sigma:=\varepsilon_{\sigma,n_1}^{\mathrm{Adaptive}}$.
On $\cE_{n_1}^{\mathrm{Adaptive}}$,
\[
\sup_{\bbx\in\cX}
|\sigma(\bbx;\cD_1)-\sigma(\bbx;P_0)|
\le\varepsilon_\sigma,
\qquad
\sup_{\bbx\in\cX}\sigma(\bbx;\cD_1)
\le\overline\sigma+\varepsilon_\sigma.
\]
Since adaptive scores and hence $q_{\mathrm{Adaptive}}(P_0;P_0)$ are
nonnegative,
uniformly over $\bbx\in\cX$,
\begin{align*}
&\left|
2\sigma(\bbx;\cD_1)\widehat q_{\mathrm{Adaptive}}(\cD_1;\cD_2)
-2\sigma(\bbx;P_0)q_{\mathrm{Adaptive}}(P_0;P_0)
\right|\\
&\quad\le
2\sigma(\bbx;\cD_1)E_q^{\mathrm{Adaptive}}
+2q_{\mathrm{Adaptive}}(P_0;P_0)
|\sigma(\bbx;\cD_1)-\sigma(\bbx;P_0)|\\
&\quad\le
2(\overline\sigma+\varepsilon_\sigma)
E_q^{\mathrm{Adaptive}}
+2q_{\mathrm{Adaptive}}(P_0;P_0)\varepsilon_\sigma.
\end{align*}
Taking the supremum proves the adaptive row. The common event has probability
at least
$1-\rho_{n_1}^S(P_0)-n_1^{-2}-\gamma$ by
Proposition~\ref{prop:alternative-score-quantile-bounds}.
\end{proof}

\subsection{Derivation for Example~\ref{ex:rq-replacement-stability}}
\label{ssec:proof-rq-replacement-stability}

\RQReplacementStabilityExample*
\begin{proof}
Let $\cD_1'$ be a training dataset of size $n_1$ that differs from
$\cD_1$ in at most $k$ observations, and write
$\hat\omega_1:=\hat\omega(\cD_1)$ and
$\hat\omega_1':=\hat\omega(\cD_1')$. For each
$\cA\in\{\cD_1,\cD_1'\}$, write $J_{\cA}$ for the objective in
\eqref{eq:rq-objective} formed from $\cA$. This
objective is $\lambda$-strongly convex, so the minimizer is unique. For a
single observation $(\bX_i,Y_i)$ and any $\omega$, every subgradient of
$\rho_\tau(Y_i-\omega^\T\bX_i)$ with respect to $\omega$ has Euclidean
norm at most
\[
G:=B_x\max\{\tau,1-\tau\}.
\]
This bound does not require bounded responses.

Set
\[
I:=\{i:(\bX_i,Y_i)\ne(\bX_i',Y_i')\},
\qquad |I|\le k.
\]
Because $0\in\partial J_{\cD_1'}(\hat\omega_1')$, choose subgradients
$\xi_i'\in\partial_\omega
\rho_\tau(Y_i'-(\hat\omega_1')^\T\bX_i')$ such that
\[
\frac{1}{n_1}\sum_{i=1}^{n_1}\xi_i'
+\lambda\hat\omega_1'=0.
\]
For every $i\notin I$, the two observations are identical, and we choose
the corresponding subgradient for $J_{\cD_1}$ to be $\xi_i=\xi_i'$.
For $i\in I$, choose any
$\xi_i\in\partial_\omega
\rho_\tau(Y_i-(\hat\omega_1')^\T\bX_i)$. Then
\[
g:=
\frac{1}{n_1}\sum_{i=1}^{n_1}\xi_i
+\lambda\hat\omega_1'
\in\partial J_{\cD_1}(\hat\omega_1')
\]
satisfies
\[
\|g\|_2
=
\left\|
\frac{1}{n_1}\sum_{i\in I}(\xi_i-\xi_i')
\right\|_2
\le
\frac{2Gk}{n_1}.
\]
Since $0\in\partial J_{\cD_1}(\hat\omega_1)$ and $J_{\cD_1}$ is
$\lambda$-strongly convex, its subdifferential is
$\lambda$-strongly monotone:
\[
\lambda\|\hat\omega_1'-\hat\omega_1\|_2^2
\le
\langle g,\hat\omega_1'-\hat\omega_1\rangle
\le
\|g\|_2\,\|\hat\omega_1'-\hat\omega_1\|_2.
\]
If $\hat\omega_1'=\hat\omega_1$, there is nothing to prove. Otherwise,
division by $\|\hat\omega_1'-\hat\omega_1\|_2$ gives
\[
\|\hat\omega_1'-\hat\omega_1\|_2
\le
\frac{2G}{\lambda}\frac{k}{n_1}.
\]
Therefore,
\[
\sup_{\|\bX\|_2\le B_x}
\left|\bX^\T(\hat\omega_1'-\hat\omega_1)\right|
\le
B_x\|\hat\omega_1'-\hat\omega_1\|_2
\le
\frac{2B_xG}{\lambda}\frac{k}{n_1},
\]
which is the claimed bound.
\end{proof}

\subsection{Formal Joint Minimax Lower Bound and Proof}
\label{ssec:proof_joint_minimax_lower_bound}

\paragraph{Lower-bound model.}
The following assumptions are used only for the lower bound; they are not
consequences of sampling or replacement stability.
To distinguish the density scale of this least-favorable submodel from the
upper-bound constant $r_0$, we denote it by $r_\star$. The two quantities are
related only through the constant-factor comparison used in the informal rate
discussion in Section~\ref{ssec:joint-minimax-lower-bound}.

\begin{assumption}[Regular location submodel]
\label{ass:least-favorable-location}
There exist constants
$L_\star,h_0,s_0,r_\star,\kappa>0$ and
$C_{\mathrm{KL}},C_{\mathrm{TV}},C_f<\infty$ with the following
properties. Let $\Theta=[-h_0/2,h_0/2]$. The model contains the
intercept-only location family
\[
\bX=0,\qquad Y=\theta+\varepsilon,\qquad \theta\in\Theta,
\]
where $\varepsilon$ has a continuously differentiable density $g$ satisfying
the conditions below; denote the law of $(\bX,Y)$ by $P_\theta$. On this
submodel, the clean
population predictor is $\mu(0;P_\theta)=\theta$. 
Here $\theta$ is the population target defining the lower-bound submodel; the
proof does not require it to coincide with the conditional mean of $Y$ under
the asymmetric noise distribution used below.
Write
\[
F_h(t):=\Pb(|\varepsilon-h|\le t),\qquad
q_g(h):=F_h^{-1}(1-\alpha),\qquad q_\star:=q_g(0).
\]
Thus $q(P_\theta;P_\theta)=q_\star$ for every $\theta\in\Theta$. For $t>0$,
let
\[
f_h(t):=F_h'(t)=g(h+t)+g(h-t)
\]
denote the density of the absolute-residual score $|\varepsilon-h|$.
For every $\theta,\theta'\in\Theta$,
\begin{align}
\operatorname{KL}(P_\theta,P_{\theta'})
&\le
C_{\mathrm{KL}}\frac{(\theta-\theta')^2}{L_\star^2},
\label{eq:lower-local-kl}\\
\operatorname{TV}(P_\theta,P_{\theta'})
&\le
C_{\mathrm{TV}}\frac{|\theta-\theta'|}{L_\star}.
\label{eq:lower-local-tv}
\end{align}
For every $|h|\le h_0$, the score density $f_h$ satisfies
\begin{equation}
r_\star\le f_h(t)\le C_f r_\star
\quad\text{whenever}\quad
|t-q_g(h)|\le 2s_0,
\label{eq:lower-two-sided-score-density}
\end{equation}
and the population score quantile is locally identifiable:
\begin{equation}
|q_g(h)-q_\star|\ge\kappa|h|.
\label{eq:lower-quantile-transversality}
\end{equation}
Finally, there exist $z_-,z_+\in\R$ such that, for all $|h|\le h_0$,
\begin{equation}
|z_--h|\le q_g(h)-2s_0,
\qquad
|z_+-h|\ge q_g(h)+2s_0.
\label{eq:lower-extreme-score-points}
\end{equation}
\end{assumption}

The assumption is non-vacuous. For example, fix $\alpha$ and start from
a $\mathcal N(m,1)$ density with $m\ne0$. At the $(1-\alpha)$-quantile
$q_{\mathrm{base}}>0$ of its absolute value, the densities at $q_{\mathrm{base}}$ and
$-q_{\mathrm{base}}$ are unequal. Adding a sufficiently small Gaussian component
with a different mean and variance preserves this inequality at the
perturbed distribution's own absolute-value quantile and produces an
asymmetric two-Gaussian mixture $g_\star$. For $\sigma>0$, the scale family
$g_\sigma(x)=\sigma^{-1}g_\star(x/\sigma)$ supplies examples for the
assumption: taking $g=g_\sigma$, the local KL and TV bounds hold with
$L_\star$ comparable to $\sigma$ up to constant factors. On compact shift
intervals, for some finite
constant $C_g$,
\[
\operatorname{KL}(P_\theta,P_{\theta'})
\le C_g\frac{(\theta-\theta')^2}{\sigma^2},
\qquad
\operatorname{TV}(P_\theta,P_{\theta'})
\le\frac{\|g_\star'\|_1|\theta-\theta'|}{2\sigma}.
\]
Here $\|g_\star'\|_1:=\int_{\R}|g_\star'(x)|\,\mathrm dx$.
Moreover, implicit differentiation of $F_h(q_g(h))=1-\alpha$ gives
\begin{equation*}
q_g'(0)
=
\frac{g(-q_\star)-g(q_\star)}{g(-q_\star)+g(q_\star)}
\ne0.
\end{equation*}
Continuity yields \eqref{eq:lower-two-sided-score-density} and
\eqref{eq:lower-quantile-transversality} after choosing $h_0,s_0$ small
enough. Taking $h_0+2s_0<\inf_{|h|\le h_0}q_g(h)$ permits $z_-=0$, and any
sufficiently large $|z_+|$ supplies the high-score point.

\paragraph{Formal criterion and theorem.}
Let $\mathcal E_{n_1}(\Theta)$ be the class of measurable, possibly randomized
estimators $\widehat\theta:\cD_1\mapsto\Theta$, and write
$\widehat\theta_1:=\widehat\theta(\cD_1)$. For $u\in\Theta$, let
$\widehat q(u;\cD_2)$ denote the conformal calibration quantile based on the
scores $|Y_i-u|$. Fix $\delta_1,\delta_2\in[0,1)$. The contamination
distributions $Q_1,Q_2$ range over all probability measures on the observation
space, and the training and calibration samples are independent with laws
\[
\cD_1\sim\bigl((1-\delta_1)P_\theta+\delta_1Q_1\bigr)^{n_1},
\qquad
\cD_2\sim\bigl((1-\delta_2)P_\theta+\delta_2Q_2\bigr)^{n_2}.
\]
$\Pb_{\theta,Q_1,Q_2}$ denotes probability over this joint law and any
independent internal randomization of the estimator. Define the minimax
tail-probability profile
\begin{equation*}
\mathfrak p_{n_1,n_2}^{\mathrm{split}}(t)
:=
\inf_{\widehat\theta\in\mathcal E_{n_1}(\Theta)}
\sup_{\theta\in\Theta}\sup_{Q_1,Q_2}
\Pb_{\theta,Q_1,Q_2}
\left(
\left|2\widehat q(\widehat\theta_1;\cD_2)-2q_\star\right|
\ge t
\right).
\end{equation*}
Also define the training- and calibration-side difficulty scales
\begin{equation}
a_{n_1}
:=
\left\{
L_\star\bigl(n_1^{-1/2}+\delta_1\bigr)
\right\}\wedge h_0,
\qquad
b_{n_2}
:=
\left\{
\frac{n_2^{-1/2}+\delta_2}{r_\star}
\right\}\wedge s_0,
\label{eq:lower-rate-components}
\end{equation}
where $x\wedge y:=\min\{x,y\}$. The truncations restrict the construction to
the local regime of Assumption~\ref{ass:least-favorable-location}.

\begin{theorem}[Joint minimax lower bound for interval-length deviation]
\label{thm:joint-minimax-lower-bound}
Under Assumption~\ref{ass:least-favorable-location}, there exist constants
$c,c_0>0$ and $n_0<\infty$, depending only on $\alpha$ and the fixed
lower-bound submodel constants, such that, for all $n_1,n_2\ge n_0$,
\begin{equation}
\mathfrak p_{n_1,n_2}^{\mathrm{split}}
\left(c\bigl(a_{n_1}+b_{n_2}\bigr)\right)
\ge c_0.
\label{eq:joint-minimax-lower-bound}
\end{equation}
\end{theorem}

When the truncations are inactive, the theorem gives the lower-bound scale
\[
L_\star\bigl(n_1^{-1/2}+\delta_1\bigr)
+
\frac{n_2^{-1/2}+\delta_2}{r_\star}.
\]
The theorem's constants may depend on the fixed submodel, so the displayed
inverse dependence on $r_\star$ is not a uniform claim over a family in which
$r_\star$ varies.

\paragraph{Criterion and scope.}
When $\theta=0$, we write the corresponding clean law explicitly as
$P_{\theta=0}$. A probability
criterion is essential here. If $\delta_2>0$ and $Q_2$ places its mass at
an arbitrarily large score $M$, the event that all $n_2$ calibration draws
come from the contaminating component has probability $\delta_2^{n_2}>0$
and makes the conformal radius arbitrarily large. Hence the supremum over
$Q_2$ of the untruncated expected interval-length loss is infinite, whereas
the fixed-probability profile is non-degenerate and comparable to the
fixed-confidence upper bound.

The restriction $\widehat\theta_1\in\Theta$ is part of the estimator class in
the theorem. It guarantees that
$|\widehat\theta_1-\theta|\le h_0$, so the local transversality and score-density
conditions apply uniformly throughout the proof. Without an additional
global condition on $q_g$ or a result showing that projection onto $\Theta$
cannot increase the interval-length loss, the theorem does not extend to
arbitrary real-valued constant predictors.

The four terms in \eqref{eq:lower-rate-components} arise from clean
training-sample indistinguishability, overlap of training Huber
neighborhoods, fluctuation of the calibration order statistic, and
separation of low- and high-score calibration contaminations. The
supremum may select a different least-favorable distribution for each
component; the theorem does not claim one distribution simultaneously
realizes all four mechanisms. It is also restricted to the standard split
architecture. Finally, $L_\star$ is a local information scale, whereas
$L$ in Definition~\ref{def:replacement-stability} is an algorithmic
stability constant. Thus the theorem establishes the structural orders,
not universal optimality of the algorithm-dependent constant $L$. Likewise,
$r_\star$ is the local score-density scale of the lower-bound
submodel and need not equal the upper-bound constant $r_0$ outside a common
model class.

A rate comparison must use the same confidence level on both sides. Fix a failure probability $\beta<c_0$, where $c_0$ is the lower-bound probability in Theorem~\ref{thm:joint-minimax-lower-bound}. On the same local model, require
\[
\eta_{n_1}(P_\theta;\beta/3)=O_\beta(L_\star n_1^{-1/2})
\quad\text{uniformly over }\theta\in\Theta,
\qquad L\asymp L_\star,\quad r_0\asymp r_\star.
\]
Here $\eta_{n_1}(P_\theta;\beta/3)$ is the fixed-confidence training radius defined in Appendix~\ref{ssec:fixed-confidence-matching}; it is distinct from the radius paired with a vanishing failure probability in Assumption~\ref{ass:sampling-stability}. The density interval must have a uniform margin, and the local truncations in the lower bound must be inactive. The same appendix proves the corresponding upper bound and gives a learner meeting these conditions on the fixed location submodel. Without these conditions, the lower bound identifies four structural components rather than a universal match of the algorithm-dependent upper bound.

We first record the elementary overlap fact used for the two contamination
constructions.

\begin{lemma}[Intersection of two Huber neighborhoods]
\label{lem:huber-neighborhood-intersection}
Let $P$ and $P'$ be two probability measures and let
$\delta\in(0,1)$. If
\begin{equation}
\operatorname{TV}(P,P')
\le
\frac{\delta}{1-\delta},
\label{eq:huber-overlap-condition}
\end{equation}
then there exist probability measures $Q$ and $Q'$ such that
\begin{equation*}
(1-\delta)P+\delta Q
=
(1-\delta)P'+\delta Q'.
\end{equation*}
\end{lemma}

\begin{proof}
Let $p$ and $p'$ be densities of $P$ and $P'$ with respect to a common
dominating measure. Since
\[
\int\max\{p,p'\}
=
1+\operatorname{TV}(P,P'),
\]
condition~\eqref{eq:huber-overlap-condition} implies
\[
(1-\delta)\int\max\{p,p'\}\le1.
\]
Choose any probability density $r$, set
\[
\zeta
:=
1-(1-\delta)\int\max\{p,p'\}\ge0,
\]
and define
\[
m:=(1-\delta)\max\{p,p'\}+\zeta r.
\]
Then $m$ is a probability density and
$m\ge(1-\delta)p,(1-\delta)p'$ pointwise. Therefore
\[
q:=\frac{m-(1-\delta)p}{\delta},
\qquad
q':=\frac{m-(1-\delta)p'}{\delta}
\]
are probability densities and satisfy
\[
(1-\delta)p+\delta q
=m
=(1-\delta)p'+\delta q'.
\]
\end{proof}

\begin{lemma}[Uniform anti-concentration of the conformal rank]
\label{lem:beta-rank-anti-concentration}
Let $p=1-\alpha$ and $k_n=\lceil(n+1)p\rceil$. For every $n$ such that
$k_n\le n$, let
\[
B_n\sim\operatorname{Beta}(k_n,n+1-k_n).
\]
Denote the density of $B_n$ by $f_{B_n}$.
For fixed $\alpha\in(0,1)$, there are constants
$c_B,C_B>0$ and $n_B<\infty$ such that, for every $n\ge n_B$,
\begin{align}
\sup_{x\in[0,1]} f_{B_n}(x)
&\le C_B\sqrt n,
\label{eq:beta-density-upper-bound}\\
\inf_{u\in[0,1]}
\Pb\left(
|B_n-u|\ge\frac{c_B}{\sqrt n}
\right)
&\ge\frac12.
\label{eq:beta-uniform-anti-concentration}
\end{align}
Moreover, for every fixed $\epsilon>0$,
\begin{equation}
\Pb\left(|B_n-p|\le\epsilon\right)\ge\frac34
\label{eq:beta-central-concentration}
\end{equation}
for all sufficiently large $n$.
\end{lemma}

\begin{proof}
Increase $n_B$, if necessary, so that $k_n\le n$ for $n\ge n_B$.
Because $k_n/(n+1)\to p\in(0,1)$, both beta shape parameters are
bounded below by a fixed positive multiple of $n$. Evaluating the beta
density $f_{B_n}$ at its mode and applying Stirling's inequalities gives
\eqref{eq:beta-density-upper-bound}. Hence, uniformly in $u\in[0,1]$,
\[
\Pb\left(
|B_n-u|<\frac{c_B}{\sqrt n}
\right)
\le 2C_Bc_B.
\]
Taking $c_B\le(4C_B)^{-1}$ proves
\eqref{eq:beta-uniform-anti-concentration}. Finally,
\[
\E B_n=\frac{k_n}{n+1}=p+O(n^{-1}),
\qquad
\operatorname{Var}(B_n)\le\frac{1}{4(n+2)}.
\]
Equation~\eqref{eq:beta-central-concentration} follows from Chebyshev's
inequality after increasing $n_B$.
\end{proof}

\begin{proof}[Proof of Theorem~\ref{thm:joint-minimax-lower-bound}]
Fix an arbitrary
$\widehat\theta\in\mathcal E_{n_1}(\Theta)$, set
$\widehat\theta_1:=\widehat\theta(\cD_1)$, and write
\[
H:=\widehat\theta_1-\theta,
\qquad
\widehat q_H:=\widehat q(\widehat\theta_1;\cD_2).
\]
Throughout the proof, $c>0$ denotes a generic constant that may decrease
from one display to the next and depends only on $\alpha$ and the constants
in Assumption~\ref{ass:least-favorable-location}. Every constant carrying
a descriptive subscript is fixed when introduced.
For a randomized estimator, adjoin its independent random seed to both
experiments below; this leaves every KL and TV comparison unchanged.
Because both $\theta$ and $\widehat\theta_1$ belong to $\Theta$, we always
have $|H|\le h_0$. We prove training and calibration lower bounds
separately and then combine them. It is useful to write
\begin{equation*}
u_{n_2}:=\left(\frac{1}{r_\star\sqrt{n_2}}\right)\wedge s_0,
\qquad
v_{n_2}:=\left(\frac{\delta_2}{r_\star}\right)\wedge s_0.
\end{equation*}
Elementary truncation inequalities give
\begin{equation}
\max\{u_{n_2},v_{n_2}\}
\le b_{n_2}
\le u_{n_2}+v_{n_2}
\le2\max\{u_{n_2},v_{n_2}\}.
\label{eq:lower-calibration-component-comparison}
\end{equation}

\paragraph{Step 1: clean training-sample difficulty.}
Choose a sufficiently small constant $c_1>0$ and let
\[
d_1
:=
c_1\left(\frac{L_\star}{\sqrt{n_1}}\wedge h_0\right)
\]
and take $\theta_-=-d_1/2$ and $\theta_+=d_1/2$. For this part choose
$Q_{1,\pm}=Q_{2,\pm}=P_{\theta_\pm}$, so the observed training laws are the clean
laws $P_{\theta_\pm}^{n_1}$. By
\eqref{eq:lower-local-kl},
\[
\operatorname{KL}
\left(P_{\theta_-}^{n_1},P_{\theta_+}^{n_1}\right)
\le
n_1 C_{\mathrm{KL}}\frac{d_1^2}{L_\star^2}.
\]
Choosing $c_1$ small enough and applying Pinsker's inequality gives
\begin{equation*}
\operatorname{TV}
\left(P_{\theta_-}^{n_1},P_{\theta_+}^{n_1}\right)
\le\frac14.
\end{equation*}

For any $z\in\Theta$, transversality and the triangle inequality imply
\begin{align}
&|q_g(z-\theta_-)-q_\star|
+|q_g(z-\theta_+)-q_\star|
\nonumber\\
&\qquad\ge
\kappa\left(
|z-\theta_-|+|z-\theta_+|
\right)
\ge
\kappa d_1.
\label{eq:lower-two-point-radius-separation}
\end{align}
For every $z$, at least one of the two losses in
\eqref{eq:lower-two-point-radius-separation} is at least
$\kappa d_1/2$. Integrating this statement with respect to the common
part of the two product measures, whose mass is
$1-\operatorname{TV}(P_{\theta_-}^{n_1},P_{\theta_+}^{n_1})$, yields
\begin{equation}
\max_{\vartheta\in\{\theta_-,\theta_+\}}
\Pb_{\vartheta,P_\vartheta,P_\vartheta}\left(
\left|q_g(\widehat\theta_1-\vartheta)-q_\star\right|
\ge\frac{\kappa d_1}{2}
\right)
\ge
\frac38.
\label{eq:lower-clean-training-probability}
\end{equation}

\paragraph{Step 2: training-contamination difficulty.}
Choose a sufficiently small constant $c_2>0$ and let
\[
d_2
:=
c_2\left(L_\star\delta_1\wedge h_0\right)
\]
and again take $\theta_\pm=\pm d_2/2$. If $\delta_1=0$, skip this
construction. Suppose henceforth in this step that $\delta_1>0$. By
\eqref{eq:lower-local-tv}, after reducing $c_2$ if necessary,
\[
\operatorname{TV}(P_{\theta_-},P_{\theta_+})
\le
C_{\mathrm{TV}}\frac{d_2}{L_\star}
\le
\frac{\delta_1}{1-\delta_1}.
\]
Lemma~\ref{lem:huber-neighborhood-intersection} supplies
$Q_{1,-},Q_{1,+}$ for which the two observed training distributions are
identical:
\[
(1-\delta_1)P_{\theta_-}+\delta_1Q_{1,-}
=
(1-\delta_1)P_{\theta_+}+\delta_1Q_{1,+}.
\]
For $\vartheta\in\{\theta_-,\theta_+\}$, write
$Q_{1,\vartheta}:=Q_{1,-}$ when $\vartheta=\theta_-$ and
$Q_{1,\vartheta}:=Q_{1,+}$ when $\vartheta=\theta_+$.
Also set $Q_{2,\pm}=P_{\theta_\pm}$, so the corresponding calibration
laws are clean.
Consequently, $\widehat\theta_1$ has the same law under the two
hypotheses. The pointwise separation
\eqref{eq:lower-two-point-radius-separation}, now with $d_2$ in place
of $d_1$, gives
\begin{equation}
\max_{\vartheta\in\{\theta_-,\theta_+\}}
\Pb_{\vartheta,Q_{1,\vartheta},P_\vartheta}\left(
\left|q_g(\widehat\theta_1-\vartheta)-q_\star\right|
\ge\frac{\kappa d_2}{2}
\right)
\ge
\frac12.
\label{eq:lower-contaminated-training-probability}
\end{equation}
Since $\max\{d_1,d_2\}\ge c a_{n_1}$, use the clean construction alone
when $\delta_1=0$ and the harder of the two constructions otherwise.
Equations~\eqref{eq:lower-clean-training-probability}
and~\eqref{eq:lower-contaminated-training-probability} then imply, as
applicable, that there exist $\theta,Q_1$ and a constant
$c_{\mathrm{tr}}>0$ such that
\begin{equation}
\Pb_{\theta,Q_1,P_\theta}\left(
\left|q_g(H)-q_\star\right|
\ge c a_{n_1}
\right)
\ge
c_{\mathrm{tr}}
\label{eq:lower-combined-training-probability}
\end{equation}

We next pass from the population radius to the empirical conformal
radius. In the preceding construction $Q_2=P_\theta$, so the observed
calibration sample is clean. Under the regular conditional law given
$H=h$, let $\widehat q_h$
denote the $k_{n_2}$-th order statistic of $n_2$ independent scores with
CDF $F_h$, and put $B_{n_2,h}:=F_h(\widehat q_h)$. Continuity of $F_h$
gives
\[
B_{n_2,h}\sim
\operatorname{Beta}(k_{n_2},n_2+1-k_{n_2}),
\]
uniformly in $|h|\le h_0$. In particular,
\[
\E B_{n_2,h}=\frac{k_{n_2}}{n_2+1}=1-\alpha+O(n_2^{-1}),
\qquad
\operatorname{Var}(B_{n_2,h})\le\frac{1}{4(n_2+2)}.
\]
Chebyshev's inequality therefore supplies a constant $C_0<\infty$ such
that, for all sufficiently large $n_2$,
\[
\sup_{|h|\le h_0}
\Pb\left(
|B_{n_2,h}-(1-\alpha)|>\frac{C_0}{\sqrt{n_2}}
\ \middle|\ H=h
\right)
\le\frac{c_{\mathrm{tr}}}{2}.
\]
After increasing $n_2$ further so that
$C_0/(r_\star\sqrt{n_2})\le2s_0$, monotonicity of $F_h$ and the local lower
density bound in \eqref{eq:lower-two-sided-score-density} imply
\begin{equation}
\sup_{|h|\le h_0}
\Pb\left(
\left|\widehat q_h-q_g(h)\right|
>
\frac{C_0}{r_\star\sqrt{n_2}}
\ \middle|\ H=h
\right)
\le
\frac{c_{\mathrm{tr}}}{2}
\label{eq:lower-empirical-quantile-concentration}
\end{equation}
for all sufficiently large $n_2$. Indeed, if $\widehat q_h$ lies farther
than the displayed threshold from $q_g(h)$, the CDF difference is at least
$C_0/\sqrt{n_2}$, including when $\widehat q_h$ lies outside the local
band. Increase $n_0$ if necessary so
that $u_{n_2}=1/(r_\star\sqrt{n_2})$ throughout the remainder of the proof.
Choose a sufficiently large constant $C_{\mathrm{dom}}>0$. If
$a_{n_1}\ge C_{\mathrm{dom}}u_{n_2}$,
then \eqref{eq:lower-combined-training-probability} and
\eqref{eq:lower-empirical-quantile-concentration}, followed by the
reverse triangle inequality, give, for some fixed
$c_{\mathrm{train}}>0$,
\begin{equation}
\sup_{\theta,Q_1,Q_2}
\Pb_{\theta,Q_1,Q_2}
\left(
\left|\widehat q_H-q_\star\right|
\ge c a_{n_1}
\right)
\ge
c_{\mathrm{train}}.
\label{eq:lower-large-training-probability}
\end{equation}

\paragraph{Step 3: clean calibration-sample difficulty.}
Fix $\theta=0$ and choose $Q_1=P_{\theta=0}$ and $Q_2=P_{\theta=0}$. Conditional on
$H=h$, the calibration scores are i.i.d.\ with continuous CDF $F_h$.
The probability integral transform for the $k_{n_2}$-th order statistic
gives
\begin{equation}
F_h(\widehat q_h)
\sim
\operatorname{Beta}
\left(k_{n_2},n_2+1-k_{n_2}\right).
\label{eq:lower-beta-order-statistic}
\end{equation}
Under this conditional law, write
$B_{n_2}:=F_h(\widehat q_h)$ for the beta random variable in
\eqref{eq:lower-beta-order-statistic}.
We now derive a uniform probability lower bound, rather than an expected
loss. First suppose that $|q_\star-q_g(h)|\le s_0$ and put
$u_h:=F_h(q_\star)$. By
Lemma~\ref{lem:beta-rank-anti-concentration}, with conditional
probability at least $1/2$,
\[
|B_{n_2}-u_h|\ge\frac{c_B}{\sqrt{n_2}}.
\]
If $\widehat q_h$ lies outside
$[q_g(h)-2s_0,q_g(h)+2s_0]$, then
$|\widehat q_h-q_\star|\ge s_0$. Otherwise, absolute continuity and the
upper density bound in \eqref{eq:lower-two-sided-score-density} give
\[
|B_{n_2}-u_h|
=
|F_h(\widehat q_h)-F_h(q_\star)|
\le
C_f r_\star|\widehat q_h-q_\star|.
\]
Thus, in this first case,
\begin{equation*}
\Pb\left(
|\widehat q_h-q_\star|\ge c u_{n_2}
\ \middle|\ H=h
\right)
\ge\frac12.
\end{equation*}

If $|q_\star-q_g(h)|>s_0$, use
\eqref{eq:beta-central-concentration} with
$\epsilon=r_\star s_0/4$. The lower density bound then implies that, with
conditional probability at least $3/4$ for all sufficiently large
$n_2$, $|\widehat q_h-q_g(h)|\le s_0/2$. Hence
$|\widehat q_h-q_\star|\ge s_0/2$ on that event. Combining the two cases,
uniformly over $|h|\le h_0$, and then integrating over $H$, gives, for
some fixed $c_{\mathrm{cal},0}>0$,
\begin{equation}
\sup_{\theta,Q_1,Q_2}
\Pb_{\theta,Q_1,Q_2}
\left(
|\widehat q_H-q_\star|
\ge c u_{n_2}
\right)
\ge
c_{\mathrm{cal},0}.
\label{eq:lower-clean-calibration}
\end{equation}

\paragraph{Step 4: calibration-contamination difficulty.}
Choose $c_3>0$ small enough that $c_3r_\star s_0\le\alpha/2$ and that
\[
\frac{\bar\delta}{1-\bar\delta}\le \frac{r_\star s_0}{2}
\qquad\text{whenever}\qquad
0\le\bar\delta\le c_3r_\star s_0.
\]
Set
\[
\bar\delta_2
:=
\delta_2\wedge(c_3r_\star s_0).
\]
When $\delta_2>0$, define
contamination distributions so that the effective non-clean mass is
$\bar\delta_2$: let
\[
\widetilde Q_{2,-}:=\operatorname{Dirac}(0,z_-),
\qquad
\widetilde Q_{2,+}:=\operatorname{Dirac}(0,z_+),
\]
and set
\[
Q_{2,\pm}
:=
\frac{\bar\delta_2}{\delta_2}\widetilde Q_{2,\pm}
+
\left(1-\frac{\bar\delta_2}{\delta_2}\right)P_{\theta=0}.
\]
Then
\[
(1-\delta_2)P_{\theta=0}+\delta_2Q_{2,\pm}
=
(1-\bar\delta_2)P_{\theta=0}
+\bar\delta_2\widetilde Q_{2,\pm}.
\]
The case $\delta_2=0$ is already covered by
\eqref{eq:lower-clean-calibration}; hence assume $\delta_2>0$ for the
remainder of this step.

Conditional on $H=h$, set
\begin{align*}
q_{\mathrm{low}}(h)
&=
F_h^{-1}\left(
\frac{1-\alpha-\bar\delta_2}{1-\bar\delta_2}
\right),
\\
q_{\mathrm{high}}(h)
&=
F_h^{-1}\left(
\frac{1-\alpha}{1-\bar\delta_2}
\right).
\end{align*}
Both probability levels differ from $1-\alpha$ by at most
$\bar\delta_2/(1-\bar\delta_2)\le r_\star s_0/2$. The lower density bound in
\eqref{eq:lower-two-sided-score-density} therefore gives, uniformly over
$|h|\le h_0$,
\[
\max\bigl\{|q_{\mathrm{low}}(h)-q_g(h)|,
|q_{\mathrm{high}}(h)-q_g(h)|\bigr\}\le s_0/2.
\]
Condition~\eqref{eq:lower-extreme-score-points} now places the low atom
strictly below both quantiles and the high atom strictly above both; hence
$q_{\mathrm{low}}(h)$ and $q_{\mathrm{high}}(h)$ are precisely the two mixed
population quantiles.
Their probability levels differ by
$\bar\delta_2/(1-\bar\delta_2)$. The density upper bound therefore
implies, uniformly in $|h|\le h_0$,
\begin{equation*}
q_{\mathrm{high}}(h)-q_{\mathrm{low}}(h)
\ge
\frac{\bar\delta_2}
{(1-\bar\delta_2)C_f r_\star}
\ge
c\left(\frac{\delta_2}{r_\star}\wedge s_0\right).
\end{equation*}

For $|h|\le h_0$, define the two mixed-score CDFs by
\[
G_{\pm,h}(t)
:=
(1-\bar\delta_2)F_h(t)
+\bar\delta_2\bbone\{|z_\pm-h|\le t\}.
\]
The distribution of $H$ is the same under $Q_{2,-}$ and $Q_{2,+}$,
because $\widehat\theta_1$ is a function only of $\cD_1$. Let
$V_1,\ldots,V_{n_2}$ be i.i.d.\ $\operatorname{Unif}(0,1)$ variables,
independent of $H$, and let $V_{(k_{n_2})}$ be their $k_{n_2}$-th order
statistic. Set
\[
B_{n_2}^{\mathrm{cpl}}:=V_{(k_{n_2})}
\sim\operatorname{Beta}(k_{n_2},n_2+1-k_{n_2}).
\]
Conditional on $H=h$, couple the two score samples as
$G_{\pm,h}^{-1}(V_1),\ldots,G_{\pm,h}^{-1}(V_{n_2})$. Their conformal
order statistics are
\[
\widehat q_{h,\pm}
:=G_{\pm,h}^{-1}(V_{(k_{n_2})})
=G_{\pm,h}^{-1}(B_{n_2}^{\mathrm{cpl}}).
\]
Finally, define
$\widehat q_{H,\pm}:=G_{\pm,H}^{-1}(B_{n_2}^{\mathrm{cpl}})$.
For all sufficiently large $n_2$, choose $c_4>0$ small enough that
$c_4r_\star s_0\le\alpha/4$ and, whenever
$|b-(1-\alpha)|\le c_4r_\star s_0$, both transformed levels
\[
\frac{b-\bar\delta_2}{1-\bar\delta_2},
\qquad
\frac{b}{1-\bar\delta_2}
\]
differ from $1-\alpha$ by at most $r_\star s_0$. The
central-concentration part of
Lemma~\ref{lem:beta-rank-anti-concentration} shows that the event
\[
\mathcal E_B
:=
\left\{
\left|B_{n_2}^{\mathrm{cpl}}-(1-\alpha)\right|
\le c_4r_\star s_0
\right\}
\]
has probability at least $3/4$. On this event, the lower density bound
places both inverses in
$[q_g(h)-s_0,q_g(h)+s_0]$. By
\eqref{eq:lower-extreme-score-points}, the low atom lies strictly below this
interval and the high atom lies strictly above it. Thus the generalized
inverses satisfy
\begin{align*}
G_{-,h}^{-1}(B_{n_2}^{\mathrm{cpl}})
&=
F_h^{-1}\left(
\frac{B_{n_2}^{\mathrm{cpl}}-\bar\delta_2}{1-\bar\delta_2}
\right),\\
G_{+,h}^{-1}(B_{n_2}^{\mathrm{cpl}})
&=
F_h^{-1}\left(
\frac{B_{n_2}^{\mathrm{cpl}}}{1-\bar\delta_2}
\right).
\end{align*}
The two probability levels differ by
$\bar\delta_2/(1-\bar\delta_2)$. The upper density bound therefore gives,
on $\mathcal E_B$,
\[
\widehat q_{h,+}-\widehat q_{h,-}
\ge
c v_{n_2}.
\]
For every coupled realization in $\mathcal E_B$, at least one of the two
deviations from $q_\star$ is at least $cv_{n_2}/2$. Taking probabilities under
the coupling, whose marginals are the two contaminated experiments, gives
\[
\Pb_{0,P_{\theta=0},Q_{2,+}}\left(
|\widehat q_{H,+}-q_\star|\ge cv_{n_2}/2
\right)
+
\Pb_{0,P_{\theta=0},Q_{2,-}}\left(
|\widehat q_{H,-}-q_\star|\ge cv_{n_2}/2
\right)
\ge
\Pb(\mathcal E_B)
\ge\frac34.
\]
Consequently,
\begin{equation}
\sup_{\theta,Q_1,Q_2}
\Pb_{\theta,Q_1,Q_2}
\left(
|\widehat q_H-q_\star|\ge c v_{n_2}
\right)
\ge
\frac38.
\label{eq:lower-calibration-contamination}
\end{equation}
Combining \eqref{eq:lower-calibration-contamination} with the clean
calibration bound \eqref{eq:lower-clean-calibration} and
\eqref{eq:lower-calibration-component-comparison} gives, for some fixed
$c_{\mathrm{cal}}>0$,
\begin{equation}
\sup_{\theta,Q_1,Q_2}
\Pb_{\theta,Q_1,Q_2}
\left(
|\widehat q_H-q_\star|\ge c b_{n_2}
\right)
\ge
c_{\mathrm{cal}}.
\label{eq:lower-combined-calibration}
\end{equation}

\paragraph{Step 5: combining training and calibration.}
If $a_{n_1}\ge C_{\mathrm{dom}}u_{n_2}$, then
\eqref{eq:lower-large-training-probability} and
\eqref{eq:lower-combined-calibration} show that the larger of
$a_{n_1}$ and $b_{n_2}$ is incurred with probability bounded away from
zero under one of their least-favorable experiments. Since their maximum
is at least half their sum, the desired radius-deviation statement
follows. If $a_{n_1}<C_{\mathrm{dom}}u_{n_2}$, then
$u_{n_2}\le b_{n_2}$ by
\eqref{eq:lower-calibration-component-comparison}, so
$a_{n_1}\le C_{\mathrm{dom}}b_{n_2}$, and
\eqref{eq:lower-combined-calibration} alone controls the sum. Therefore,
for constants $c,c_0>0$ independent of the estimator,
\begin{equation*}
\sup_{\theta,Q_1,Q_2}
\Pb_{\theta,Q_1,Q_2}
\left(
|\widehat q_H-q_\star|
\ge c(a_{n_1}+b_{n_2})
\right)
\ge c_0.
\end{equation*}
Finally,
\[
\left|2\widehat q_H-2q_\star\right|
=
2\left|\widehat q_H-q_\star\right|.
\]
The preceding argument holds for every
$\widehat\theta\in\mathcal E_{n_1}(\Theta)$. Multiplying by
two and taking the infimum over this class proves
\eqref{eq:joint-minimax-lower-bound}.
\end{proof}

\subsection{Fixed-Confidence Upper Bound and Local Rate Matching}
\label{ssec:fixed-confidence-matching}
Sampling stability in Assumption~\ref{ass:sampling-stability} uses a failure probability tending to zero. A fixed-confidence comparison with the lower bound instead uses a training radius indexed by its failure probability. For a clean law $P$ and $\gamma\in(0,1)$, let $\eta_n(P;\gamma)$ be a deterministic radius satisfying
\[
\Pb_{\cD_n\sim P^n}\!\left\{\|\mu(\cdot;\cD_n)-\mu(\cdot;P)\|_\infty>\eta_n(P;\gamma)\right\}\le\gamma.
\]
This notation does not assert that the radius in Assumption~\ref{ass:sampling-stability} is $O(n^{-1/2})$. The distinction matters because a fixed confidence level allows a constant multiple of $n^{-1/2}$, whereas a failure probability tending to zero may require a larger radius.

\begin{proposition}[Fixed-confidence interval-length bound]
\label{prop:fixed-confidence-length-bound}
Suppose Assumption~\ref{ass:score-density} holds and $\mu$ is $L$-uniformly replacement stable. Fix $\beta\in(0,1)$ and contamination levels $\delta_1\in[0,1)$, $\delta_2\in[0,\alpha)$, with $\alpha\in(0,1/2)$. If the clean training radius $\eta_{n_1}(P_0;\beta/3)$ satisfies the preceding probability bound, then, for all sufficiently large $n_2$, with probability at least $1-\beta$, uniformly over $Q_1,Q_2$,
\begin{align}
2|\widehat q(\cD_1;\cD_2)-q(P_0;P_0)|
\le{}&2\eta_{n_1}(P_0;\beta/3)
+2L\left(\delta_1+\sqrt{\frac{\log(3/\beta)}{2n_1}}\right)\nonumber\\
&+\frac{2}{(1-\delta_2)r_0}
\left(\sqrt{\frac{\log(6/\beta)}{2n_2}}+\frac{2}{n_2}+\delta_2(1-\alpha)\right).
\label{eq:fixed-confidence-length-bound}
\end{align}
\end{proposition}

\begin{proof}
Under the Huber coupling, $K\sim\operatorname{Binomial}(n_1,\delta_1)$. Hoeffding's inequality gives
\[
\Pb\!\left\{\frac K{n_1}>\delta_1+\sqrt{\frac{\log(3/\beta)}{2n_1}}\right\}\le\beta/3.
\]
On this event's complement and the clean training event, replacement stability bounds the predictor error by $\eta_{n_1}(P_0;\beta/3)+L(\delta_1+\sqrt{\log(3/\beta)/(2n_1)})$. The absolute-score and order-statistic inequalities transfer this bound to the empirical quantile on the same calibration sample. Lemma~\ref{lem:empirical-quantile-error-under-p2} with $\gamma=\beta/3$ bounds calibration sampling error, and Lemma~\ref{lem:calibration-contamination-quantile} bounds population calibration bias. A union bound for the three failure events and multiplication by two prove \eqref{eq:fixed-confidence-length-bound}.
\end{proof}

\paragraph{An attaining learner on the location submodel.}
Consider the fixed location model in Assumption~\ref{ass:least-favorable-location}, with known noise law $g$ and a fixed scale $\sigma$ comparable to $L_\star$, as in its scale-family example. We first apply a bounded transformation so that even an arbitrarily large replacement has a bounded effect on the sample average. We then invert its population mean to recover the location parameter. Define
\begin{gather*}
\psi(y):=\tanh(y/\sigma),\qquad m(\theta):=\mathbb E_g\psi(\theta+\varepsilon),\\
c_\psi:=\inf_{\theta\in\Theta}\mathbb E_g\operatorname{sech}^2((\theta+\varepsilon)/\sigma),
\end{gather*}
where $\operatorname{sech}(z):=1/\cosh(z)$. Positivity, continuity, and compactness of $\Theta$ give $c_\psi>0$ and $m'(\theta)\ge c_\psi/\sigma$. Let $\operatorname{Proj}_{m(\Theta)}$ denote Euclidean projection onto the interval $m(\Theta)$, and set
\[
\widehat\theta(\cD_n)
:=m^{-1}\!\left(\operatorname{Proj}_{m(\Theta)}
\left[\frac1n\sum_{i=1}^n\psi(Y_i)\right]\right),
\qquad \mu(\bbx;\cD_n):=\widehat\theta(\cD_n).
\]
This estimator takes values in $\Theta$. Since projection is nonexpansive and $m^{-1}$ is $\sigma/c_\psi$-Lipschitz, replacing $k$ arbitrary observations gives
\[
|\widehat\theta(\cD_n)-\widehat\theta(\cD_n')|
\le\frac{2\sigma}{c_\psi}\frac{k}{n}.
\]
For clean data, $\mathbb E_\theta\psi(Y)=m(\theta)$ and $|\psi|\le1$. Hoeffding's inequality therefore yields, uniformly over $\theta\in\Theta$,
\[
\Pb_\theta\!\left\{|\widehat\theta-\theta|>
\frac{\sigma}{c_\psi}\sqrt{\frac{2\log(2/\gamma)}n}\right\}\le\gamma.
\]
Thus the clean population target is $\theta$, $L=2\sigma/c_\psi\asymp L_\star$, and $\eta_n(P_\theta;\gamma)=O_\gamma(L_\star n^{-1/2})$ uniformly over $\Theta$.

Restrict $\delta_2$ to a sufficiently small fixed neighborhood of zero so that the quantile range lies inside the common density interval in Assumption~\ref{ass:least-favorable-location}. The clean oracle score law is the same for every $\theta$, so $r_0$ can be chosen comparable to $r_\star$ with a uniform interval margin. Proposition~\ref{prop:fixed-confidence-length-bound} then gives the order $L_\star(n_1^{-1/2}+\delta_1)+(n_2^{-1/2}+\delta_2)/r_\star$. When the local truncations are inactive, Theorem~\ref{thm:joint-minimax-lower-bound} supplies the matching lower order for any fixed failure probability $\beta<c_0$. This establishes attainability within the stated estimator class on this fixed submodel. It does not establish the same training rate for every replacement-stable learner or for the experimental neural networks.

\section{Experimental details}
\label{app:experimental_details}

\subsection{Common protocol}
\label{app:common_protocol}

We study length deviations in prediction intervals under training and calibration contamination. The training and calibration  contamination probabilities are $\delta_1$ and $\delta_2$, respectively. Test observations remain uncontaminated. The target miscoverage level is $\alpha=0.1$. We use this level for calibration in the following text unless comparing the corrections in Figures~\ref{fig:all_correction}, \ref{fig:abalone_cadata_correction}, and~\ref{fig:cpusmall_msd_correction}.

\paragraph{Scores and prediction sets.}
We compare Absolute, Adaptive, and CQR scores:
\begin{align*}
 s_{\mathrm{abs}}(x,y)&=|y-\widehat\mu(x)|,\\
 s_{\mathrm{adp}}(x,y)&=|y-\widehat\mu(x)|/\widehat\sigma(x),\\
 s_{\mathrm{cqr}}(x,y)&=\max\{\widehat q^{\mathrm{lo}}(x)-y,y-\widehat q^{\mathrm{hi}}(x)\}.
\end{align*}
The location model estimates the conditional median. Adaptive score uses the same location model and a separate positive scale estimate. CQR estimates the $0.05$ and $0.95$ conditional quantiles. We sort the two fitted CQR endpoints at each input before computing scores. These base quantile levels remain fixed in the correction experiments.

For $n_{\mathrm{cal}}$ calibration scores, let
\[
 k=\left\lceil(n_{\mathrm{cal}}+1)(1-\alpha)\right\rceil,
 \qquad \widehat q=R_{(k)},\qquad R_{(n_{\mathrm{cal}}+1)}=+\infty.
\]
The three prediction sets are 
\begin{align*}
\text{Absolute:}\quad & [\widehat\mu(x) - \widehat q,\widehat\mu(x) + \widehat q] \\
\text{Adaptive:}\quad &[\widehat\mu(x) - \widehat q\cdot\widehat\sigma(x),\widehat\mu(x) + \widehat q\cdot\widehat\sigma(x)]\\
\text{CQR:}\quad & [\widehat q^{\mathrm{lo}}(x)-\widehat q,
\widehat q^{\mathrm{hi}}(x)+\widehat q]
\end{align*}
CQR scores and threshold $\hat q$ may be negative. If the final lower endpoint exceeds the upper endpoint, the prediction set is empty. Both its length and coverage contribution are zero.

\paragraph{Replacement sampling.}
Except for the diagnostic in Figure~\ref{fig:stability_constant} and the fixed-fraction experiments in Appendix~\ref{app:bernoulli-fixed-fraction}, each observation is replaced with the specified contamination probability $\delta_j$ at stage $j$. Within a repetition, we draw one independent uniform variable per observation and replace the observation if this draw is at most $\delta_j$ at stage $j$. Thus, the replacement sets are nested as $\delta_j$ increases, while ensuring that each set has the required Bernoulli marginal distribution. Training and calibration use independent random streams. For structured contamination, we use the same replacement masks and candidate pairs $(X_Q,Y_Q)$ across scores. Within each family, different severity levels also share the same masks and base noise draws.

\begin{algorithm}[H]
\caption{Split conformal prediction with contamination}
\label{alg:contaminated-split-cp}
\begin{algorithmic}[1]
\Require Training, calibration, and clean test sets; contamination rules $Q_1,Q_2$; probabilities $\delta_1,\delta_2$; level $\alpha$.
\State Replace training observations according to $Q_1$ and $\delta_1$. Fit all models needed for the chosen score.
\State If required by the design, construct $Q_2$ from the fitted models; if not, skip this step.
\State Replace calibration observations according to $Q_2$ and $\delta_2$.
\State Compute calibration scores. Set $k=\lceil(n_{\mathrm{cal}}+1)(1-\alpha)\rceil$ and $\widehat q=R_{(k)}$, with $R_{(n_{\mathrm{cal}}+1)}=+\infty$.
\State Form $C(x)=\{y:s(x,y)\leq\widehat q\}$ and evaluate it on the clean test set.
\end{algorithmic}
\end{algorithm}

\paragraph{Metrics and repetitions.}
Let $w_r^{\delta}(x)$ be the interval length in repetition $r$ under the specified contamination setting. For the score under consideration, define the reference length by
\[
w_r^{\mathrm{ref}}(x)=
\begin{cases}
w_{\mathrm{pop}}(x), & \text{for synthetic data},\\
w_r^0(x), & \text{for real data},
\end{cases}
\]
where $w_{\mathrm{pop}}(x)$ is the clean population-oracle length and $w_r^0(x)$ is the length of clean split conformal prediction fitted and calibrated within the same real-data repetition at $\alpha=0.1$. We record the signed mean difference
\[
 D_r=\frac{1}{n_{\mathrm{test}}}\sum_{i=1}^{n_{\mathrm{test}}}
       \bigl[w_r^{\delta}(X_i)-w_r^{\mathrm{ref}}(X_i)\bigr].
\]
Positive values indicate interval inflation; negative values indicate contraction. The same dataset-specific reference is used for the correction experiments. We also record the mean pointwise absolute difference,
\[
\frac{1}{n_{\mathrm{test}}}\sum_i|w_r^{\delta}(X_i)-w_r^{\mathrm{ref}}(X_i)|.
\]
We use the mean pointwise absolute difference for the contamination plots and the signed difference $D_r$ for the calibration-level correction plots (Figures~\ref{fig:all_correction},~\ref{fig:abalone_cadata_correction}, and~\ref{fig:cpusmall_msd_correction}).
Coverage is defined as the fraction of clean test responses contained in their prediction sets.

For all curves and surfaces, we report means over 20 repetitions. Where shaded bands appear, they indicate the sample standard deviation across repetitions, using denominator $19$. The score histograms are based on repetition $r=0$, which was selected in advance.

\subsection{Synthetic setup}
\label{app:exp_setup}

The five coordinates of $X$ are independent and uniformly distributed on $[-1,1]$. We generate the response from
\[
 Y=f(X)+\sigma\varepsilon,\qquad
 f(x)=\sin(\pi x_1)+x_2^2-\tfrac12x_3,
 \qquad \varepsilon\sim N(0,1).
\]
In each repetition, the training, calibration, and test samples are independent and have sizes $5{,}000$, $5{,}000$, and $10{,}000$, respectively. We set $\sigma=1$, except in Figure~\ref{fig:all_correction}, where $\sigma=3$.
Each regression model uses 36 fixed features: an intercept; the linear, quadratic, and cubic terms of each coordinate; $\sin(\pi x_j)$ and $\cos(\pi x_j)$ for each coordinate; and the ten pairwise interactions $x_i x_j$, $i<j$. For a quantile level $\tau$, we minimize
\[
 \frac{1}{n_{\mathrm{train}}}\sum_i
 \rho_\tau\bigl(Y_i-\theta^\top\phi(X_i)\bigr)
 +\frac{0.1}{2}\|\theta\|_2^2,
 \qquad \rho_\tau(u)=u\bigl(\tau-\mathbf1\{u<0\}\bigr).
\]
After fitting, all predictions are clipped to $[-20,20]$. The optimization uses accelerated projected gradient updates on the dual, with working-set refinement, a relative primal--dual gap tolerance of $10^{-8}$, and at most $200{,}000$ iterations. The gap is checked every ten iterations.

The location model is fitted at $\tau=0.5$. The Adaptive scale model also uses $\tau=0.5$, using absolute training residuals as responses. Its predictions are clipped to $[0.25,20]$. When the training data are contaminated, both models are fitted on the corresponding contaminated training sample. CQR fits separate models at $\tau=0.05$ and $0.95$.

For the synthetic experiments, the reference is the clean population oracle obtained using the same regularized learner and population score calibration; Adaptive and CQR use pointwise oracle widths evaluated at the corresponding test covariates.

\subsection{Contamination designs}
\label{app:contamination_designs}

\paragraph{Joint training and calibration contamination.}
For Figures~\ref{fig:joint_length} and~\ref{fig:all_surface}, we set
\[
 Q_1=Q_2=\operatorname{Dirac}(x^\star,-80),
 \qquad x^\star=(-1,-1,-1,-1,+1).
\]
The same covariate and response values are used at both stages and for all three scores. The models are refitted for each training contamination level. Consequently, the score assigned to the fixed calibration point may vary with $\delta_1$. Figure~\ref{fig:stability_constant} uses the same training point mass, without calibration contamination.

\paragraph{Structured calibration contamination.}
The training data remain clean in these experiments. Table~\ref{tab:synthetic_contamination} defines the four contamination families. In the first three families, we set \(X_Q=X\). The Gaussian and Student-\(t\) draws are independent of the clean data. In the shift family, we retain the observed response and then add the shift.

\begin{table}[H]
\centering
\caption{Structured contamination for the synthetic experiments ($\sigma=1$).
Here $\varepsilon_Q\sim N(0,1)$ and $T_3$ has a Student-$t$ distribution
with three degrees of freedom.}
\label{tab:synthetic_contamination}
\renewcommand{\arraystretch}{1.25}
\setlength{\tabcolsep}{4pt}
\begin{tabularx}{\linewidth}{@{}%
>{\hsize=0.9\hsize\raggedright\arraybackslash}X%
>{\hsize=0.6\hsize\raggedright\arraybackslash}X%
>{\hsize=1.3\hsize\raggedright\arraybackslash}X%
>{\hsize=1.2\hsize\raggedright\arraybackslash}X@{}}
\toprule
Family & \(X_Q\) & \(Y_Q\) & Parameter values \\
\midrule
Mean shift & \(X\) & \(Y+a\) & \(a\in\{3,5,10,15\}\) \\
Variance inflation & \(X\) & \(f(X)+s\sigma\varepsilon_Q\)
 & \(s\in\{3,5,10,15\}\) \\
Heavy tail & \(X\) & \(f(X)+c\sigma T_3/\sqrt{3}\)
 & \(c\in\{1,3,5,10\}\) \\
Covariate shift + heteroscedastic noise
 & \(X+b\mathbf{1}_5\)
 & {\(f(X_Q)+\sigma(0.5\)\(+1.5\|X_Q\|_\infty)\varepsilon_Q\)}
 & \(b\in\{1.5,3,5,8\}\) \\
\bottomrule
\end{tabularx}
\end{table}

The multiplier \(s\) multiplies the noise standard deviation; the corresponding variance multiplier is therefore \(s^2\). The random variable \(T_3/\sqrt{3}\) has unit variance. The final family changes both the covariate distribution and the noise scale. For example, setting \(b=8\) shifts the support of the synthetic covariates to \([7,9]^5\). Each family is evaluated over its own severity grid.

\paragraph{Finite high-score references.}
Figures~\ref{fig:various_contami}, \ref{fig:all_trend_nonlinear}, and~\ref{fig:real-width-trend} use calibration point masses with score \(M=200\) as empirical references. The training data remain clean. For the synthetic data, the covariate is \(x^\star\), and we choose the response for each score as
\[
 y_Q=\begin{cases}
 \widehat{\mu}_0(x^\star)-M, & \text{Absolute},\\
 \widehat{\mu}_0(x^\star)-M\widehat{\sigma}_0(x^\star), & \text{Adaptive},\\
 \widehat{q}^{\mathrm{lo}}_0(x^\star)-M, & \text{CQR}.
 \end{cases}
\]
The subscript \(0\) denotes a fit on clean training data. For \(\delta_2<\alpha\), the Dirac reference curves show the mean length deviation under these point masses, plus one sample standard deviation. These are descriptive summaries for a specified finite contamination distribution, not theoretical upper bounds. Randomness in the replacement counts can cause a sharp increase before the contamination probability reaches \(0.1\).

\paragraph{Low-score contamination and calibration-level correction.}
For Absolute and Adaptive, we retain the selected calibration covariates and replace their responses with \(\widehat{\mu}_0(X)\). Their scores are then zero. For CQR, let
\[
 B_0(x)=\widehat{q}^{\mathrm{hi}}_0(x)-\widehat{q}^{\mathrm{lo}}_0(x).
\]
We numerically search for a point \(x^\star_{\mathrm{CQR}}\) with large \(B_0(x)\) and set
\[
 X_Q=x^\star_{\mathrm{CQR}},\qquad
 Y_Q=\frac{\widehat{q}^{\mathrm{lo}}_0(x^\star_{\mathrm{CQR}})
             +\widehat{q}^{\mathrm{hi}}_0(x^\star_{\mathrm{CQR}})}{2}.
\]
The resulting score is \(-B_0(x^\star_{\mathrm{CQR}})/2\). For synthetic data, the search domain is \([-1,1]^5\). We evaluate a scrambled Sobol sequence of \(2^{16}\) points and all \(32\) corners. The eight best candidates are used to initialize Powell searches, each with at most \(300\) iterations, a position tolerance of \(10^{-8}\), and an objective tolerance of \(10^{-10}\). We select the point with the largest fitted band length among all candidate and locally searched points.

We compare the uncorrected level with two corrections:
\[
 \alpha'_{\mathrm{unc}}=\alpha,\qquad
 \alpha'_{\mathrm{tight}}=\alpha(1-\delta_2),\qquad
 \alpha'_{\mathrm{naive}}=\alpha-\delta_2.
\]
Each rule replaces \(\alpha\) only in the calibration rank. The additive rule is used only for \(\delta_2\le\alpha\). At \(\delta_2=\alpha\), it gives \(\alpha'_{\mathrm{naive}}=0\), rank \(n_{\mathrm{cal}}+1\), an infinite interval length, and coverage equal to one. For any common \(\delta_2\) for which both rules are defined, the additive rule selects a threshold that is at least as large as the one selected by the multiplicative rule.

\subsection{Grids, diagnostics, and reproducibility}
\label{app:grids_diagnostics}

For the length-trend and joint-contamination experiments, \(\delta_2\) ranges from zero to \(0.11\) in increments of \(0.005\). We also include \(0.096\), \(0.097\), \(0.098\), and \(0.099\), for a total of 27 points. The joint training grid ranges from zero to \(0.30\) in increments of \(0.005\), for a total of 61 points. The correction experiments range from zero to \(0.95\) with the same step size, for a total of 191 points. Figures~\ref{fig:all_trend_nonlinear} and~\ref{fig:real-width-trend} display structured means at
\[
 \delta_2\in\{0,0.015,0.03,0.045,0.06,0.075,0.09,0.10,0.11\}.
\]
The results for the full grid are retained.

\paragraph{Predictor and density diagnostics.}
Figure~\ref{fig:stability_constant} uses exact replacement counts along one fixed random permutation of the training sample. The first 11 settings replace zero through ten of the \(5{,}000\) training observations. The remaining contamination levels range from \(0.005\) to \(0.095\) in increments of \(0.005\). In each repetition, we compute
\[
 \eta_r(\delta_1)=\max_{1\le i\le n_{\mathrm{test}}}
 \bigl|\widehat{\mu}_{r,\delta_1}(X_i)-\widehat{\mu}_{r,0}(X_i)\bigr|
\]
and fit a line through the origin to these values. The mean fitted slope across repetitions is \(56.37\), with a standard deviation of \(5.28\). This slope summarizes the sensitivity of the predictor to this point-mass contamination along the specified replacement path, as measured on the clean test sample.

All prediction-interval thresholds remain empirical order statistics, as specified in Appendix~\ref{app:common_protocol}. For the density diagnostic in Figure~\ref{fig:lower_bound} and the reference curve in Figure~\ref{fig:train_calibration}, we condition on each saved clean predictor $\widehat\mu_{r,0}$ and compute the distribution of $|f(X)-\widehat\mu_{r,0}(X)+\varepsilon|$ under the known synthetic law. Here $X\sim\operatorname{Unif}([-1,1]^5)$ and $\varepsilon\sim N(0,1)$ are independent. We integrate over the Gaussian noise analytically and over the covariates numerically using $2^{16}$ scrambled Sobol points (seed 20260915). This yields the score CDF and density without kernel smoothing. For \(\delta_2<\alpha\), the two endpoint probabilities of the interval are
\[
 p_-(\delta_2)=\frac{1-\alpha-\delta_2}{1-\delta_2},\qquad p_+(\delta_2)=\frac{1-\alpha}{1-\delta_2}.
\]
We numerically invert this CDF at the two endpoint probabilities to locate the endpoints of the diagnostic interval. For each evaluated model, we verify that the density is decreasing throughout this interval, so its minimum is attained at the upper endpoint, and we compute that minimum. These diagnostic quantiles are not used to construct prediction intervals. Let \(\widehat{r}_{0,r}(\delta_2)\) denote the computed density minimum. The computed minima change by less than \(1.72\times10^{-5}\) relative to the primary computation when using \(2^{14}\) points or an independent scramble.
Figure~\ref{fig:lower_bound} reports the mean of this quantity across repetitions. The reference curve in Figure~\ref{fig:train_calibration} is computed as
\[
 \overline{B}(\delta_2)=\frac{1}{20}\sum_{r=0}^{19}\frac{2\delta_2(1-\alpha)}{(1-\delta_2)\widehat{r}_{0,r}(\delta_2)}.
\]
This diagnostic curve evaluates the population-level contamination expression with fitted-model scores and then averages over repetitions. It uses the clean score distribution of each fitted predictor instead of that of the population learner, and it omits the training and calibration sampling terms in Theorem~\ref{thm:delta}. It is distinct from the clean population-oracle interval reference used in the synthetic length metric and is not a certified finite-sample upper bound. Comparisons near \(\delta_2=\alpha\) are therefore interpreted descriptively and remain subject to finite-sample variability.
Figure~\ref{fig:train_calibration} uses logarithmic axes and shows only nonzero contamination levels. Its calibration-only branch is restricted to \(\delta_2<\alpha\), consistent with the density diagnostic.

\paragraph{Score histograms.}
Figures~\ref{fig:score_distribution}, \ref{fig:abalone_cadata_distribution}, and~\ref{fig:cpusmall_msd_distribution} use repetition 0 and \(\delta_2=0.1\). The family parameters are \(a=5\), \(s=5\), \(c=10\), and \(b=8\). Scores and source labels are taken from the same contaminated calibration samples used for the length trends. Each panel uses 40 equal-width bins. Each source component has height \(N_{\mathrm{source},\mathrm{bin}}/(n_{\mathrm{cal}}\times\text{bin width})\), so the stacked components sum to the mixture histogram. The vertical lines use the conformal order statistics for the clean and contaminated calibration samples.

Most panels cover the full range of the pooled scores from the clean, contamination-only, and mixture settings. For heavy tails, the initial range is based on pooled quantiles \(0.002\) and \(0.985\), extended to include both thresholds, and given a \(4\%\) margin. If the lower endpoint is nonnegative before adding the margin, the final lower limit is clipped at zero. Figure~\ref{fig:abalone_score} overlays the clean and contaminated Absolute-score histograms for Abalone at \(a=5\), using the same fixed repetition, 40 bins, and range \([0,8]\). Density heights are always normalized using the full sample size, including observations outside the displayed range. Across the displayed histograms, at most \(0.46\%\) of a mixture sample lies outside its window.

\paragraph{Random seeds and saved outputs.}
Repetitions are indexed by \(r=0,\ldots,19\), with a data-split seed \(42+10{,}000r\). Auxiliary streams are seeded with \(42+10{,}000r+100j+c\), where \(j\) and \(c\) are fixed stream identifiers recorded for each experiment. We save the split indices, fitted models, replacement masks, contaminated observations, scores, calibration ranks, and per-repetition metrics. The implementation uses Python 3.12.14, NumPy 2.5.3, SciPy 1.16.2, scikit-learn 1.7.2, and PyTorch 2.8.0 (CPU). Numerical computations are run using one process and one thread per library.

\subsection{Further synthetic results}
\label{app:further_synthetic}

Figure~\ref{fig:all_surface} extends the joint contamination experiment to all three scores.
Across the evaluated grid, the magnitude of the mean length deviation depends on the score and the two contamination levels.
The fixed point mass produces finite widths across the evaluated grid, including for \(\delta_2\ge\alpha\).
The theoretical breakdown result is stated in terms of the supremum over calibration contamination distributions.

\begin{figure}[tbp]
\centering
\includegraphics[width=\textwidth]{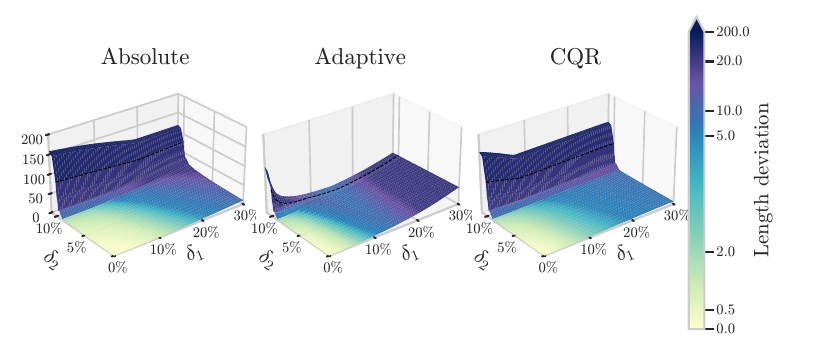}
\caption{Joint training and calibration contamination for Absolute,
Adaptive, and CQR. Both stages use
$Q_1=Q_2=\operatorname{Dirac}((-1,-1,-1,-1,+1),-80)$.
The two horizontal axes are $\delta_1$ and $\delta_2$; height and color
show the mean length deviation over 20 repetitions at $\alpha=0.1$.
Height uses a linear $0$--$200$ scale, and color follows the labeled
nonlinear scale.}
\label{fig:all_surface}
\end{figure}

Figure~\ref{fig:all_correction} shows the effect of contamination from low-score calibration points. Without correction, intervals contract and coverage on the clean test set falls. Across all three scores and 191 contamination levels, the mean coverage over 20 repetitions under the multiplicative correction ranges from $0.897$ to $0.904$. The additive rule produces wider intervals where defined and an infinite interval at \(\delta_2=0.1\).

\begin{figure}[tbp]
\centering
\includegraphics[width=\textwidth]{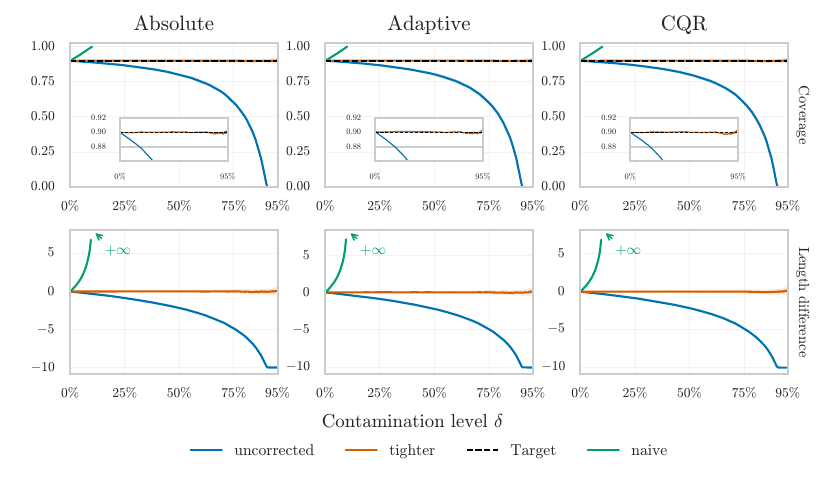}
\caption{Calibration-level correction under low-score contamination in
the synthetic model with $\sigma=3$. Columns correspond to Absolute,
Adaptive, and CQR. The top row shows clean-test coverage and the bottom
row the signed length difference. Curves compare $\alpha$,
$\alpha(1-\delta_2)$, and $\alpha-\delta_2$, with target coverage $0.9$.
Results are means over 20 repetitions. The additive rule has infinite
length at $\delta_2=\alpha=0.1$ and is undefined beyond this point.}
\label{fig:all_correction}
\end{figure}

Figure~\ref{fig:all_trend_nonlinear} shows that interval inflation depends on both the score and the contamination family. Under the combined covariate and noise change, Absolute and CQR exhibit much larger maximum mean changes in interval length than Adaptive over the tested grid. Under mean shift, the largest changes are similar across scores.

\begin{figure}[tbp]
\centering
\includegraphics[width=\textwidth]{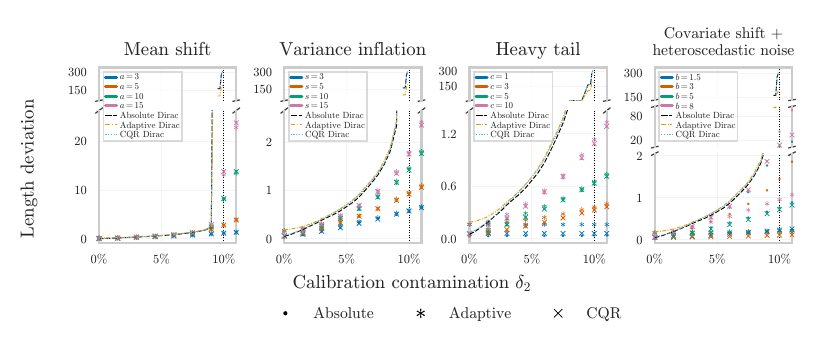}
\caption{Synthetic calibration contamination at $\alpha=0.1$. Columns show mean shift, variance inflation, heavy tails, and covariate shift plus heteroscedastic noise. Colors indicate severity and marker shapes distinguish scores. Points show mean length deviation over 20 repetitions. The score-$200$ Dirac references show mean plus one standard deviation for $\delta_2<\alpha$. Broken vertical axes separate these large reference values from the structured results; the final column uses three segments. Tick ranges differ between panels.}
\label{fig:all_trend_nonlinear}
\end{figure}

Figure~\ref{fig:score_distribution} displays the score mixtures underlying these changes. For the fixed repetition and severity levels, contamination raises the conformal threshold in every panel. The shape and magnitude of the shift vary across scores. These panels correspond to a single calibration sample from repetition 0.

\begin{figure}[tbp]
\centering
\includegraphics[width=\textwidth]{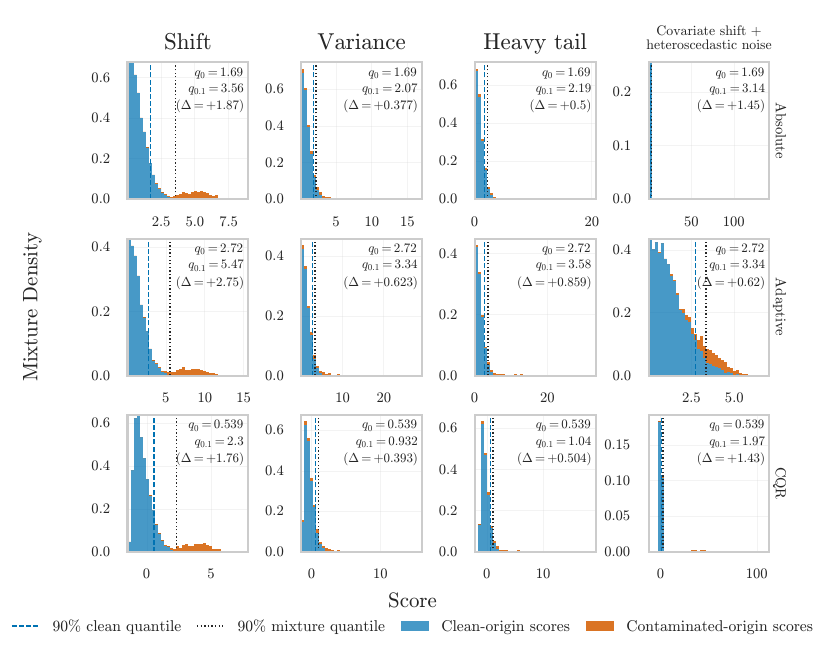}
\caption{Synthetic score distributions at $\delta_2=0.1$, using the fixed
repetition zero. Rows correspond to Absolute, Adaptive, and CQR; columns
use $a=5$, $s=5$, $c=10$, and $b=8$. The respective replacement counts
are $537$, $463$, $452$, and $491$ out of $5{,}000$, shared across scores.
Stacked bars distinguish retained clean observations from replacements.
Vertical lines show the clean and contaminated conformal thresholds at
$\alpha=0.1$. Each panel reports their difference.}
\label{fig:score_distribution}
\end{figure}

\clearpage
\section{Experiments on real data}
\label{app:real_data}

\subsection{Data and model fitting}
\label{app:real_setup}

We use the LIBSVM versions of \texttt{Abalone}, \texttt{California Housing (cadata)}, \texttt{CPU Small}, and \texttt{YearPredictionMSD (MSD)}. Table~\ref{tab:real_data_splits} reports the dimensions and sample sizes. For the first three datasets, a fixed \texttt{RandomState(42)} permutation assigns the first \(\lfloor 0.8N \rfloor\) observations to a training pool and the remaining observations to a clean test set. Each repetition permutes the pool and selects disjoint training and calibration samples. For both Abalone and CPU Small, one pool observation is unused in each split. MSD provides predefined training and test files with \(463{,}715\) and \(51{,}630\) observations, respectively. Using seed 42, we first select a fixed training pool of \(50{,}000\) observations and a fixed test subset of \(10{,}000\) observations. Each repetition then draws \(5{,}000\) training and \(5{,}000\) calibration observations from this pool. Test sets remain fixed across repetitions. Variation across repetitions reflects the train--calibration split, model fitting, and contamination sampling, conditional on the fixed dataset and test set.

\begin{table}[H]
\centering
\caption{Real-data dimensions and sample sizes. The training pool is
fixed before the repeated train--calibration splits.}
\label{tab:real_data_splits}
\setlength{\tabcolsep}{5pt}
\begin{tabular}{@{}lrrrrr@{}}
\toprule
Dataset & $d$ & Pool & Train & Calibration & Test \\
\midrule
Abalone & 8 & 3,341 & 1,670 & 1,670 & 836 \\
California Housing & 8 & 16,512 & 8,256 & 8,256 & 4,128 \\
CPU Small & 12 & 6,553 & 3,276 & 3,276 & 1,639 \\
YearPredictionMSD & 90 & 50,000 & 5,000 & 5,000 & 10,000 \\
\bottomrule
\end{tabular}
\end{table}

Features are standardized coordinate-wise using only the clean training split.
Responses are centered and scaled using the mean and population standard deviation of the clean training split.
These transformations are then applied to the calibration and test observations.
Interval lengths are converted back to the original response units.
Score histograms use the internal standardized response scale; for Adaptive, they additionally use the fitted scale division.

Each quantile model is a neural network trained with the pinball loss defined in Section~\ref{app:exp_setup}.
Table~\ref{tab:real_nn_configuration} reports its configuration.
The outer training split contains an internal \(10\%\) validation subset for early stopping.
Absolute and Adaptive share the clean median network \(\widehat{\mu}_0\).
Adaptive fits a separate median network to the absolute residuals computed on the outer training split.
CQR uses separate networks at levels \(0.05\) and \(0.95\).

\begin{table}[H]
\centering
\caption{Neural-network configuration for all real-data experiments.}
\label{tab:real_nn_configuration}
\renewcommand{\arraystretch}{1.15}
\begin{tabularx}{\linewidth}{@{}lX@{}}
\toprule
Component & Setting \\
\midrule
Architecture & $d$--64--32--1, with ReLU activations \\
Optimizer & AdamW; learning rate $10^{-3}$; weight decay $10^{-4}$ \\
Training & Batch size 256; at most 100 epochs; gradient norm clipped at 5 \\
Early stopping & Internal validation fraction $0.1$; patience 10;
minimum improvement $10^{-4}$; restore the best validation checkpoint \\
\bottomrule
\end{tabularx}
\end{table}

\subsection{Training residual scale and contamination}
\label{app:real_contamination}

All response-side contamination parameters are defined in standardized response units. After fitting \(\widehat{\mu}_0\), we compute
\begin{align*}
 e_i&=|Y_i-\widehat{\mu}_0(X_i)|,\qquad \overline{e}=\frac{1}{n_{\mathrm{train}}}\sum_i e_i,\\
 \sigma_{\mathrm{res}}&=\max\left\{\sqrt{\frac{1}{n_{\mathrm{train}}}\sum_i(e_i-\overline{e})^2},\,10^{-12}\right\}.
\end{align*}
This is the standard deviation of the absolute residuals on the clean outer training split, computed with denominator $n_{\mathrm{train}}$. We fix it before applying calibration contamination. The Adaptive-scale predictions are clipped to \([0.25\sigma_{\mathrm{res}},20\sigma_{\mathrm{res}}]\).

Training remains clean throughout the real-data experiments. Structured calibration contamination uses the same parameter grids as those in Table~\ref{tab:synthetic_contamination}. The observed response is replaced by \(Y+a\sigma_{\mathrm{res}}\) for mean shift, by \(\widehat{\mu}_0(X)+s\sigma_{\mathrm{res}}\varepsilon_Q\) for variance inflation, and by \(\widehat{\mu}_0(X)+c\sigma_{\mathrm{res}}T_3/\sqrt{3}\) for heavy tails. For the composite family, we use
\[
 X_Q=X+b\mathbf{1}_d,\qquad Y_Q=\widehat{\mu}_0(X_Q)+\sigma_{\mathrm{res}}(0.5+1.5\|X_Q\|_\infty)\varepsilon_Q.
\]
The feature shift is applied after training-based standardization. The same clean median model defines these contamination centers for all three scores.

For the finite score-\(200\) references, the anchor is the coordinate-wise signed extreme of the standardized training design:
\[
 x^\star_j=X_{i_j^\star,j},\qquad i_j^\star\in\operatorname*{arg\,max}_i|X_{ij}|.
\]
We set \(Y_Q\) to \(\widehat{\mu}_0(x^\star)+200\) for Absolute, \(\widehat{\mu}_0(x^\star)+200\widehat{\sigma}_0(x^\star)\) for Adaptive, and \(\widehat{q}^{\mathrm{hi}}_0(x^\star)+200\) for CQR. These choices all give score \(200\) in the corresponding fitted score function.

Low-score contamination follows Section~\ref{app:contamination_designs}. For the CQR search, the domain is the coordinate-wise minimum--maximum box of the standardized clean training features. Candidates comprise the training observations, \(2^{14}\) scrambled Sobol points, 256 random corners, the corner with the largest absolute value in each coordinate, and its opposite corner. The best 16 candidates are refined by 300 Adam steps with learning rate \(0.05\), constrained to the box. We retain the best point found across the candidates and refinements.

\subsection{Calibration-level correction}

Figures~\ref{fig:abalone_cadata_correction} and~\ref{fig:cpusmall_msd_correction} compare the three calibration-level rules under low-score contamination.
The multiplicative correction maintains coverage near the target level \(0.9\), although differences remain across datasets and scores.
Table~\ref{tab:real_correction_ranges} reports the range of mean coverage across the full grid and all three scores.
Mean coverage on Abalone is above \(0.9\) throughout the grid.
Some settings for the other datasets fall below \(0.9\).
The additive correction is more conservative and yields infinite length at \(\delta_2=0.1\).

\begin{table}[H]
\centering
\caption{Range of mean clean-test coverage with
$\alpha'=\alpha(1-\delta_2)$. Each range spans all three scores and
all 191 contamination levels; each mean uses 20 repetitions.}
\label{tab:real_correction_ranges}
\begin{tabular}{@{}lc@{}}
\toprule
Dataset & Minimum--maximum mean coverage \\
\midrule
Abalone & $0.916$--$0.934$ \\
California Housing & $0.888$--$0.899$ \\
CPU Small & $0.886$--$0.908$ \\
YearPredictionMSD & $0.887$--$0.897$ \\
\bottomrule
\end{tabular}
\end{table}

\begin{figure}[tbp]
\centering
\includegraphics[width=\textwidth]{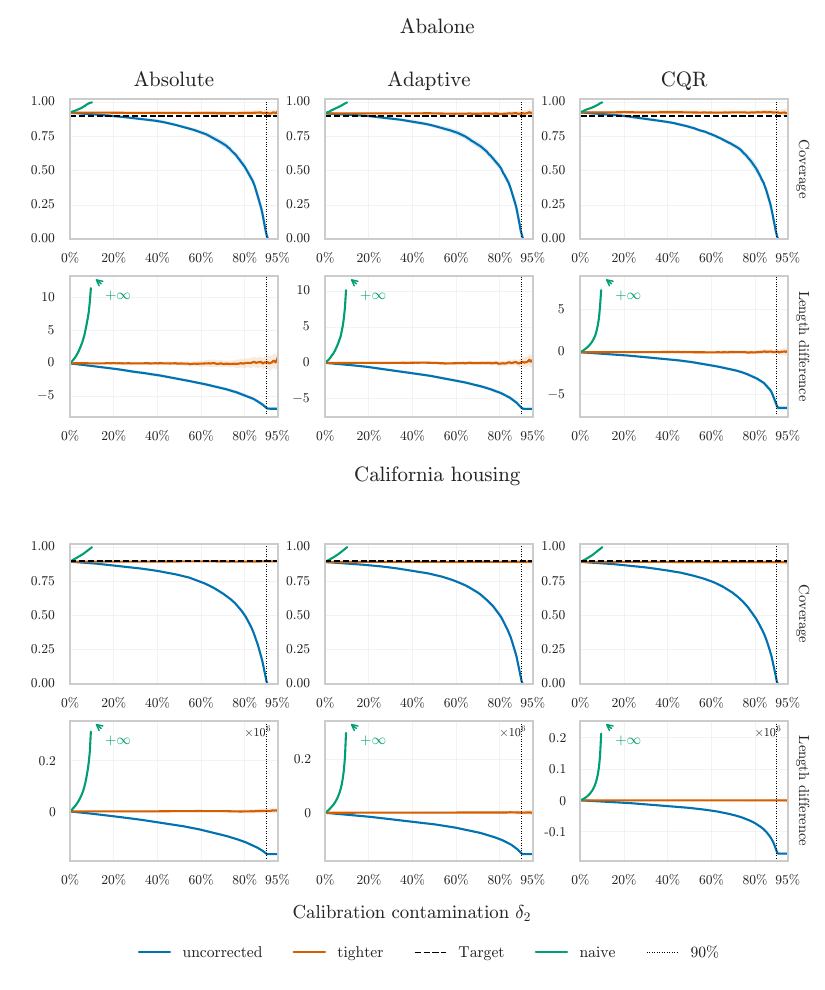}
\caption{Calibration-level correction on Abalone (top two rows) and
California Housing (bottom two rows). Columns correspond to Absolute,
Adaptive, and CQR. Each dataset has a coverage row and a signed length
difference row. Curves compare $\alpha$, $\alpha(1-\delta_2)$, and
$\alpha-\delta_2$ under low-score contamination and average 20 repetitions.
The dashed horizontal line marks coverage $0.9$, and the dotted vertical
line marks $\delta_2=0.9$. The additive rule has infinite length at
$\delta_2=0.1$. Length differences use the original response units.}
\label{fig:abalone_cadata_correction}
\end{figure}

\begin{figure}[tbp]
\centering
\includegraphics[width=\textwidth]{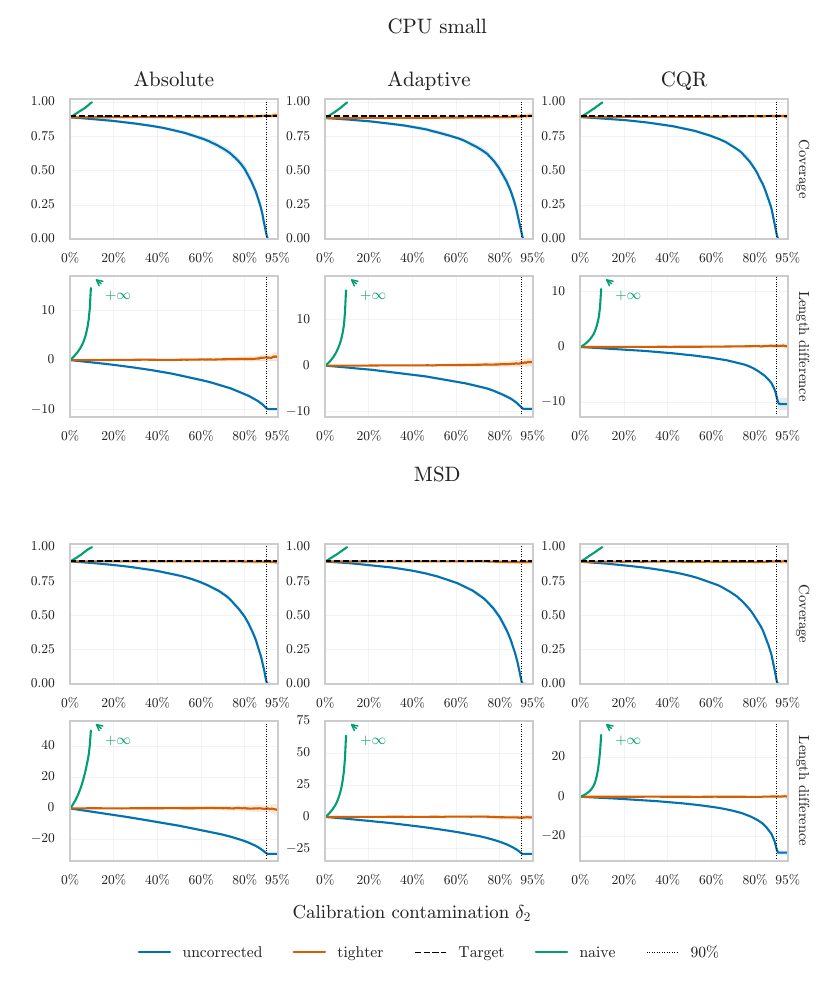}
\caption{Calibration-level correction on CPU Small (top two rows) and
YearPredictionMSD (bottom two rows), using the same conventions as
Figure~\ref{fig:abalone_cadata_correction}. Columns correspond to Absolute,
Adaptive, and CQR. Results average 20 repetitions, and length differences
are in the original response units. The additive rule has infinite length
at $\delta_2=0.1$.}
\label{fig:cpusmall_msd_correction}
\end{figure}

\subsection{Calibration contamination and interval length}
\label{app:real_upper_bound}

Figure~\ref{fig:real-width-trend} compares the four structured families at \(\alpha=0.1\).
For each dataset and score, the largest mean length difference across the tested severity and contamination grids occurs under mean shift.
The ordering of the other families depends on the dataset and score.
The Dirac curves show the mean plus one sample standard deviation, as defined in Section~\ref{app:contamination_designs}.

We use broken vertical axes to display both the structured points and the much larger finite Dirac references.
Each panel has its own set of axis ranges.
California Housing uses an axis multiplier of \(10^6\), which is applied to both segments in each panel.

\begin{figure}[tbp]
\centering
\includegraphics[width=\textwidth]{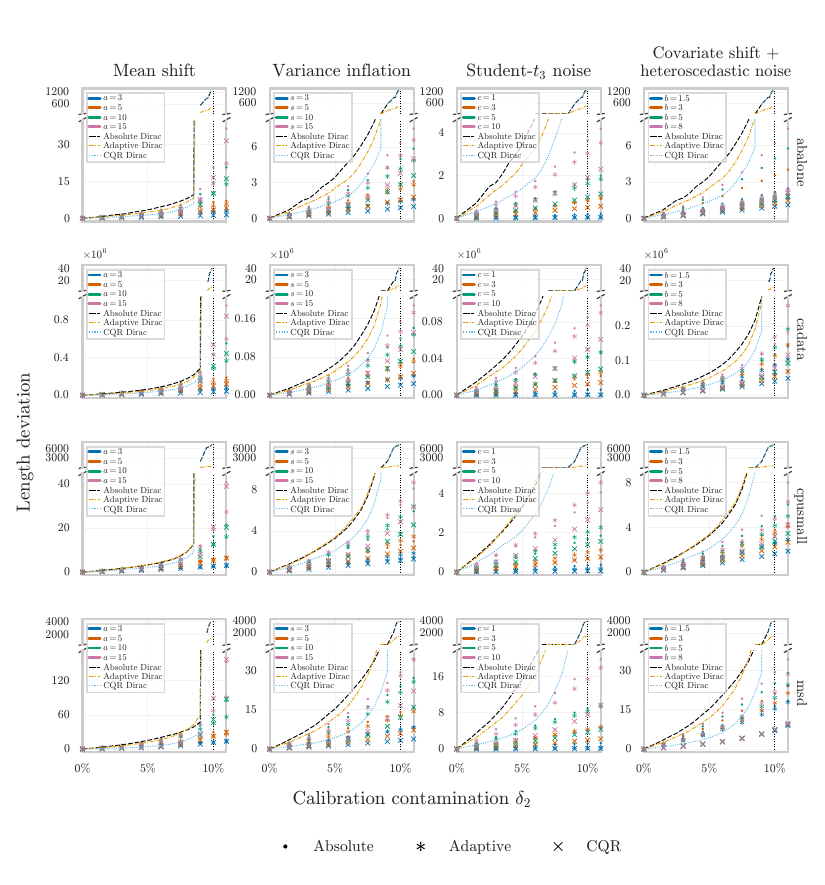}
\caption{Real-data calibration contamination at $\alpha=0.1$. Rows correspond to Abalone, California Housing, CPU Small, and YearPredictionMSD. Columns show mean shift, variance inflation, heavy tails, and covariate shift plus heteroscedastic noise. Colored markers show severity levels, and marker shapes distinguish scores. Points show mean length deviation over 20 repetitions, in original response units. The score-$200$ Dirac curves show mean plus one standard deviation only for $\delta_2<\alpha$. Each panel uses two independently labeled vertical-axis segments; both California Housing segments share the displayed $10^6$ multiplier.}
\label{fig:real-width-trend}
\end{figure}

\clearpage
\subsection{Score distributions}
\label{app:real_score_distribution}

Figures~\ref{fig:abalone_cadata_distribution} and~\ref{fig:cpusmall_msd_distribution} use the fixed repetition \(0\), \(\delta_2=0.1\), and the same four severity values as in the synthetic histograms.
In each displayed panel, the contaminated conformal threshold is above the clean threshold.
Their separation and the mixture shape vary across datasets, scores, and contamination families.
The realized replacement fractions also differ because the sampling is Bernoulli.
Table~\ref{tab:real_replacement_counts} reports the realized counts, which are shared across scores within each dataset and family.
The effect on length depends on the score: for Adaptive, the threshold is multiplied by the fitted scale, while CQR adjusts the endpoints of its fitted base interval.

\begin{table}[H]
\centering
\caption{Realized replacement counts for the real-data score histograms
at $\delta_2=0.1$ in repetition zero. Columns use $a=5$, $s=5$, $c=10$,
and $b=8$, respectively. ``Composite'' denotes covariate shift plus
heteroscedastic noise.}
\label{tab:real_replacement_counts}
\setlength{\tabcolsep}{4pt}
\begin{tabular}{@{}lrrrrr@{}}
\toprule
Dataset & $n_{\mathrm{cal}}$ & Shift & Variance & Heavy tail & Composite \\
\midrule
Abalone & 1,670 & 183 & 161 & 144 & 161 \\
California Housing & 8,256 & 870 & 791 & 769 & 812 \\
CPU Small & 3,276 & 349 & 303 & 304 & 338 \\
YearPredictionMSD & 5,000 & 537 & 463 & 452 & 491 \\
\bottomrule
\end{tabular}
\end{table}

\begin{figure}[tbp]
\centering
\includegraphics[width=\textwidth]{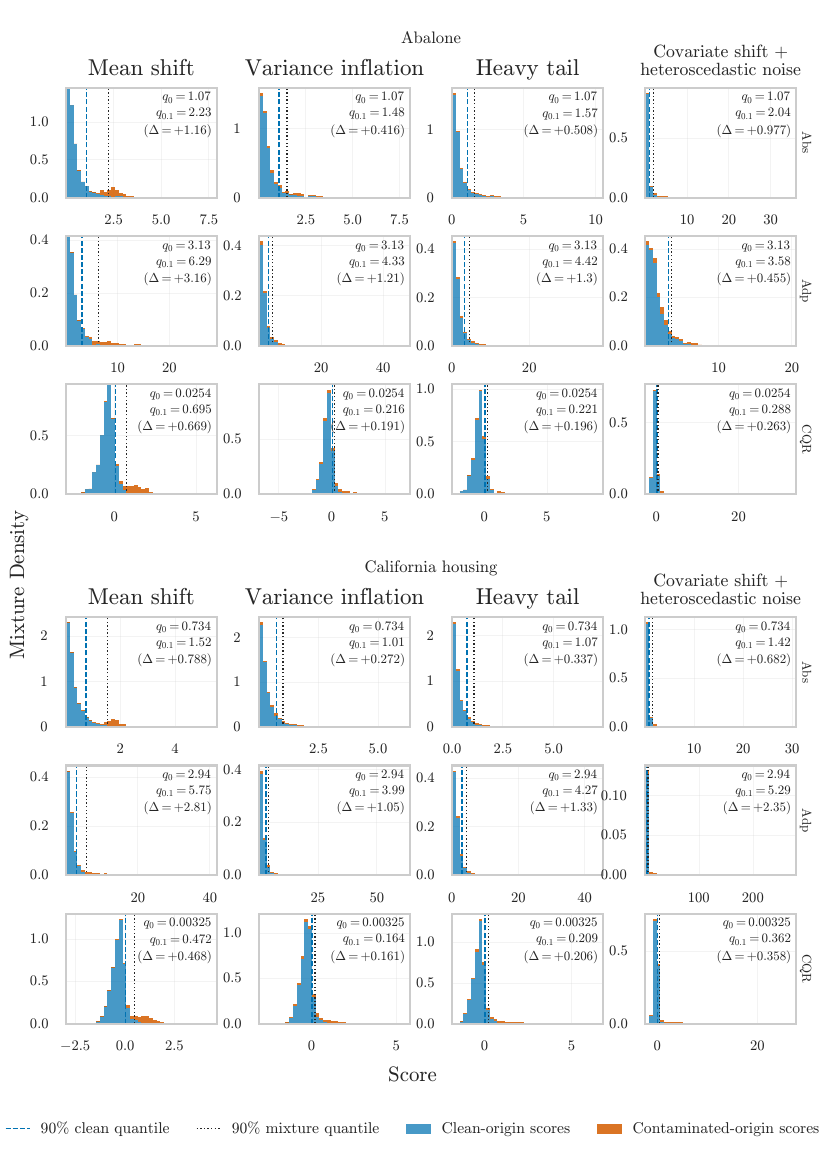}
\caption{Score mixtures on Abalone (top) and California Housing (bottom),
using repetition zero and $\delta_2=0.1$.
Within each dataset, rows show Absolute, Adaptive, and CQR. Columns use
mean shift $a=5$, variance inflation $s=5$, heavy tails $c=10$, and
covariate shift plus heteroscedastic noise $b=8$.
Stacked bars show the source components; vertical lines are conformal
thresholds at $\alpha=0.1$. Panel annotations report their differences.}
\label{fig:abalone_cadata_distribution}
\end{figure}

\begin{figure}[tbp]
\centering
\includegraphics[width=\textwidth]{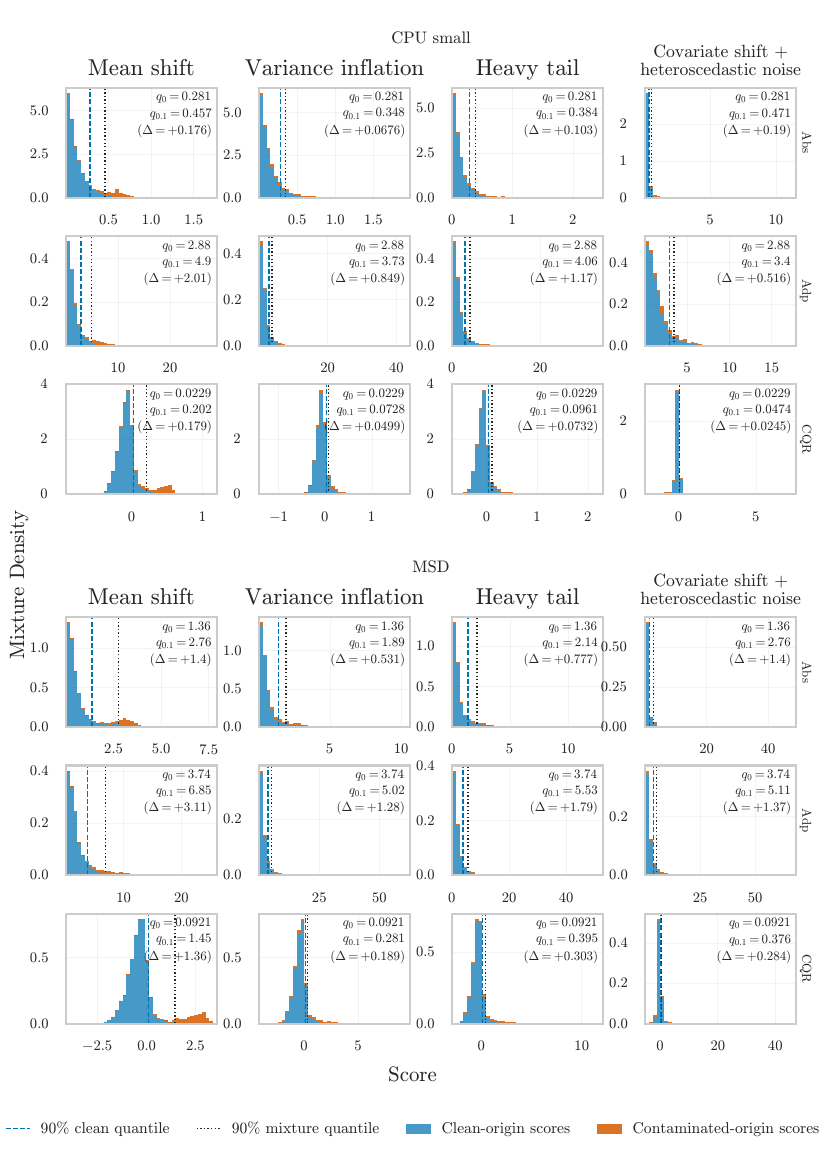}
\caption{Score mixtures on CPU Small (top) and
YearPredictionMSD (bottom), with the same ordering and
conventions as Figure~\ref{fig:abalone_cadata_distribution}.
Results use repetition zero at $\delta_2=0.1$.}
\label{fig:cpusmall_msd_distribution}
\end{figure}
\clearpage

\section{Bernoulli and Fixed-Fraction Replacement}
\label{app:bernoulli-fixed-fraction}

\paragraph{Conditional relationship.}
At each stage $j\in\{1,2\}$, write $(n,\delta,Q)=(n_j,\delta_j,Q_j)$, where $n$, $\delta$, and $Q$ denote the sample size, replacement probability, and contamination distribution. Under Bernoulli replacement, independent indicators $U_i\sim\operatorname{Bernoulli}(\delta)$ mark the replaced observations, and $K=\sum_{i=1}^nU_i\sim\operatorname{Binomial}(n,\delta)$ counts the replacements. For $0<\delta<1$ and any $u\in\{0,1\}^n$ with $\sum_i u_i=k$,
\[
\Pb\bigl((U_1,\ldots,U_n)=u\mid K=k\bigr)=\binom nk^{-1}.
\]
Thus, choosing exactly $k=\lfloor n\delta\rfloor$ locations uniformly and independently of the data and replacing those observations with draws from $Q$ is precisely the Bernoulli replacement scheme conditional on $K=k$. The two stages use independent copies of this construction. At $\delta=0$, both schemes leave the sample unchanged. Moreover, $K/n\to\delta$ in probability and $k/n\to\delta$, so the two schemes have the same limiting replacement fraction.

\paragraph{Interval-length bound.}
Let $k_j=\lfloor n_j\delta_j\rfloor$ denote the fixed replacement counts. Under the conditions of Theorem~\ref{thm:delta}, for every $\gamma\in(0,1)$ and all sufficiently large $n_1,n_2$, the fixed-fraction procedure satisfies, with probability at least $1-\rho_{n_1}(P_0)-\gamma$,
\[
\begin{aligned}
2\left|\widehat q(\cD_1;\cD_2)-q(P_0;P_0)\right|
&\le 2\eta_{n_1}(P_0)+\frac{2Lk_1}{n_1}\\
&\quad+\frac{2}{(1-\delta_2)r_0}
\left(\sqrt{\frac{\log(2/\gamma)}{2n_2}}+\frac{2}{n_2}+\delta_2(1-\alpha)\right).
\end{aligned}
\]
Indeed, replacement stability bounds the training perturbation by $Lk_1/n_1$ deterministically. Conditional on the selected calibration locations, the oracle scores are independent: $n_2-k_2$ are clean draws, and $k_2$ are draws from the contaminating score distribution. Applying Hoeffding's inequality at the two thresholds bracketing the target quantile gives the sampling bound in Lemma~\ref{lem:empirical-quantile-error-under-p2} with replacement fraction $k_2/n_2$. Since $k_2/n_2\le\delta_2$, the original density lower bound and population-bias bound remain valid. The clean-training sampling term is unchanged, which completes the same decomposition as in Theorem~\ref{thm:delta}. The upper bound in Theorem~\ref{thm:delta} therefore remains valid, with the training count-fluctuation term removed.

\paragraph{Experimental comparison.}
We keep the original splits, contamination distributions, learners, and 20 repetitions. At each stage, the $k_j$ smallest saved uniform draws determine the replacement locations; training-contaminated models are refitted. Figure~\ref{fig:fixed-fraction-placement} compares the two schemes in the setting of Figure~\ref{fig:joint_length}(a). Using fixed counts removes the early rise caused by replacement-count fluctuations near $\delta_2=\alpha$. Figure~\ref{fig:fixed-fraction-real} complements this comparison with the four structured contamination settings in the Abalone dataset. 
The reference curves in these figures are empirical references, not certified finite-sample upper bounds.

\begin{figure}[!htbp]
    \centering
    \begin{subfigure}[b]{0.48\linewidth}
        \includegraphics[width=\linewidth]{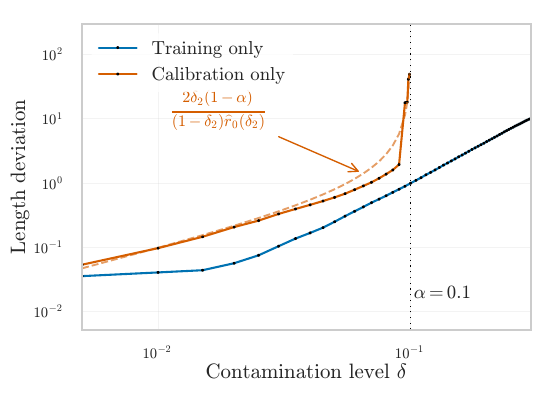}
        \caption{Bernoulli replacement.}
    \end{subfigure}\hfill
    \begin{subfigure}[b]{0.48\linewidth}
        \includegraphics[width=\linewidth]{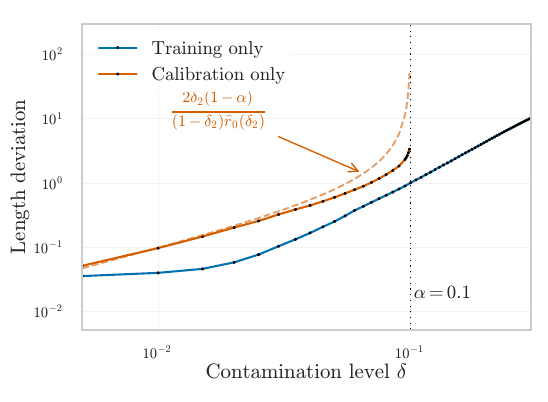}
        \caption{Fixed-fraction replacement.}
    \end{subfigure}
    \caption{Training-only and calibration-only mean length deviation relative to the clean population oracle, averaged over 20 repetitions. Both panels use the setting of Figure~\ref{fig:joint_length}(a) and the same numerical reference for the population calibration term.}
    \label{fig:fixed-fraction-placement}
    \par\medskip
    \centering
    \includegraphics[width=\linewidth]{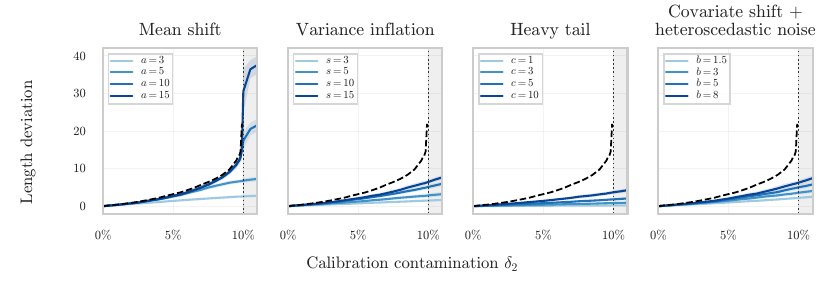}
    \caption{Fixed-fraction counterpart of Figure~\ref{fig:various_contami}: calibration-only mean length deviation on Abalone with absolute-residual scores. Curves and shading show means and one standard deviation over 20 repetitions. The dashed curve is the finite score-$200$ Dirac mean plus one standard deviation for $\delta_2<\alpha$.}
    \label{fig:fixed-fraction-real}
\end{figure}

\end{document}